\documentclass[11pt]{article}

\usepackage[utf8]{inputenc}
\usepackage[T1]{fontenc}
\usepackage{amsmath,amssymb,amsthm,mathtools,bm}
\usepackage{graphicx}
\usepackage{float}
\usepackage{wrapfig}
\usepackage{algorithm}
\usepackage{algpseudocode}
\usepackage{enumitem}
\usepackage{xspace}
\usepackage{array}
\usepackage{booktabs}
\usepackage{multirow}
\usepackage{makecell}
\usepackage{colortbl}
\usepackage{tikz}
\usepackage{caption}
\usepackage{subcaption}
\usepackage{microtype}
\usepackage[dvipsnames]{xcolor}
\usepackage{twemojis}

\def\E{\mathbb{E}}

\def\sg{\operatorname{sg}}
\def\epsilonv{{\bm{\epsilon}}}
\def\vV{{\bm{V}}}
\def\vc{{\bm{c}}}
\def\vu{{\bm{u}}}
\def\vv{{\bm{v}}}
\def\vx{{\bm{x}}}
\def\rmI{{\mathbf{I}}}
\def\gN{{\mathcal{N}}}

\definecolor{hunyuanblue}{HTML}{1E4A8F}
\definecolor{colGray}{gray}{0.5}
\definecolor{colSky}{HTML}{A7CAEA}
\definecolor{mygreen}{HTML}{2E8B57}

\newcommand{\nmfmt}[1]{\texttt{#1}}
\newcommand{\sdmodel}{\nmfmt{SD3.5-M}\xspace}
\newcommand{\wanmodel}{\nmfmt{Wan2.1}\xspace}
\newcommand{\clipscore}{\nmfmt{CLIPScore}\xspace}
\newcommand{\pickscore}{\nmfmt{PickScore}\xspace}
\newcommand{\hpstwo}{\nmfmt{HPSv2}\xspace}
\newcommand{\hpsthree}{\nmfmt{HPSv3}\xspace}
\newcommand{\imagereward}{\nmfmt{ImageReward}\xspace}
\newcommand{\aesthetic}{\nmfmt{Aesthetic Score}\xspace}

\newcommand{\genevaltwo}{\nmfmt{GenEval2}\xspace}
\newcommand{\ocrscore}{\nmfmt{OCR}\xspace}
\newcommand{\videoalign}{\nmfmt{VideoAlign}\xspace}
\newcommand{\vbench}{\nmfmt{VBench}\xspace}
\newcommand{\longcat}{\nmfmt{LongCat-Video}\xspace}
\newcommand{\drawbench}{\nmfmt{DrawBench}\xspace}

\newcommand{\best}[1]{\textbf{#1}}
\newcommand{\second}[1]{\underline{#1}}
\newcommand{\defstep}[1]{{\itshape\textcolor{black!60}{\,(#1\,steps)}}}

\usepackage{hunyuan}

\usepackage[capitalize,noabbrev]{cleveref}
\usepackage[most]{tcolorbox}
\usepackage{fancyhdr}
\usepackage{xurl}
\usepackage{titletoc}

\crefname{appendix}{Appendix}{Appendices}
\Crefname{appendix}{Appendix}{Appendices}
\numberwithin{theorem}{section}

\makeatletter
\long\def\@makecaption#1#2{%
  \vskip 10pt
  \setbox\@tempboxa\hbox{#1: #2}%
  \ifdim \wd\@tempboxa >\hsize
    \noindent #1: #2\par
  \else
    \hbox to\hsize{\hfil\box\@tempboxa\hfil}%
  \fi}
\makeatother

\hypersetup{
  colorlinks=true,
  linkcolor=hunyuanblue,
  citecolor=hunyuanblue,
  urlcolor=hunyuanblue,
}

\makeatletter
\def\section{\@startsiction{section}{1}{\z@}{-0.24in}{0.10in}
             {\large\bf\raggedright\color{hunyuanblue}}}
\def\subsection{\@startsection{subsection}{2}{\z@}{-0.20in}{0.08in}
                {\normalsize\bf\raggedright\color{hunyuanblue}}}
\makeatother

\fancypagestyle{plain}{%
  \fancyhf{}%
  \fancyfoot[C]{\thepage}%
}
\fancypagestyle{firststyle}{%
  \fancyhf{}%
  \fancyhead[L]{%
    \IfFileExists{figs/hunyuan-logo.png}
      {\raisebox{-18.5mm}[0pt][0pt]{\includegraphics[height=12mm]{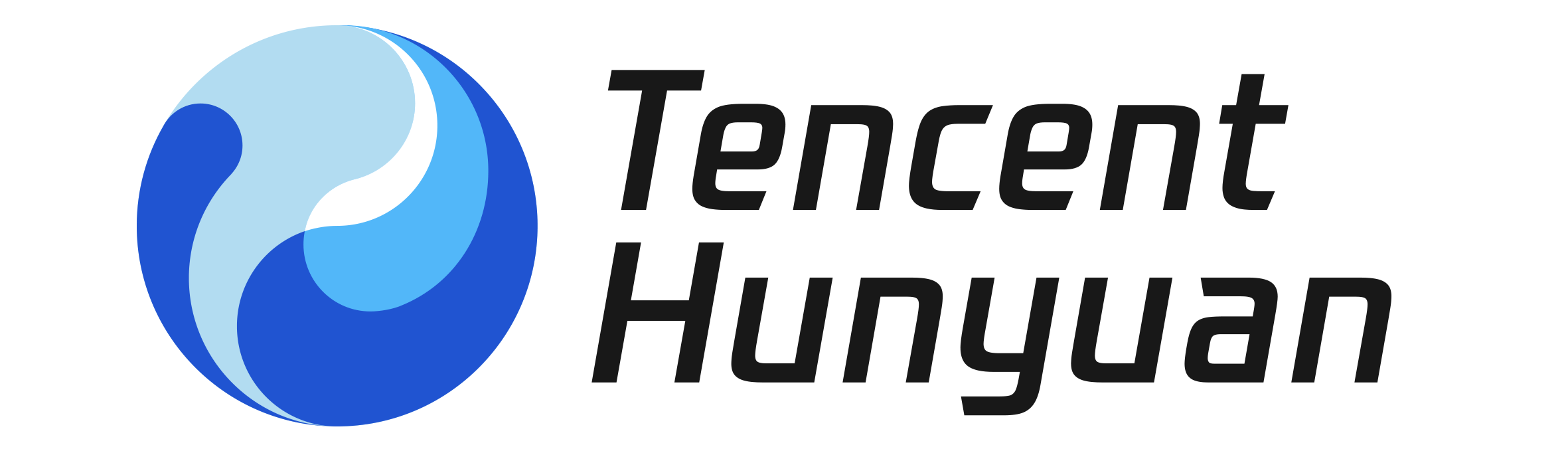}}}
      {\raisebox{-14mm}[0pt][0pt]{\textcolor{hunyuanblue}{\sffamily\bfseries HUNYUAN}}}}%
  \fancyfoot[C]{\thepage}%
}

\makeatletter
\newcommand{\appendixtitle}{%
  \clearpage
  \thispagestyle{plain}%
  \vbox{%
    \hsize\textwidth
    \linewidth\hsize
    \vskip 0.1in
    {\color{hunyuanblue}\hrule height 0.6pt}%
    \vskip 4mm
    \centering
    {\LARGE\bfseries Appendix\par}%
    \vskip 4mm
    {\color{hunyuanblue}\hrule height 0.6pt}%
    \vskip 0.3in\@minus0.1in
  }%
}
\makeatother

\definecolor{abstractbg}{HTML}{F0F7FC}
\tcbset{
  greenhighlight/.style={
    colback=mygreen!5,
    colframe=mygreen!40,
    boxrule=0.8pt,
    arc=1.5mm,
    left=1mm,
    right=1mm,
    top=0.3mm,
    bottom=0.3mm,
    enhanced,
    rounded corners,
  }
}

\begin{document}

\thispagestyle{firststyle}
\vspace*{0.25cm}
{\color{hunyuanblue}\hrule height 0.6pt}
\vskip 6mm
\begin{center}
{\LARGE\bfseries MeanFlowNFT: Bringing Forward-Process RL to
Average-Velocity Generators\par}
\end{center}
\vskip 3mm
{\color{hunyuanblue}\hrule height 0.6pt}
\vskip 6mm

\begin{center}
\textbf{Yushi Huang}$^{1,2,*}$ \quad
\textbf{Xiangxin Zhou}$^{1,*,\ddagger}$ \quad
\textbf{Jun Zhang}$^{2}$ \quad
\textbf{Liefeng Bo}$^{1}$ \quad
\textbf{Tianyu Pang}$^{1,\ddagger}$
\\[8pt]
$^1$Tencent Hunyuan \quad
$^2$The Hong Kong University of Science and Technology
\\[6pt]
{\small
$^*$Equal contribution \quad
$^\ddagger$Corresponding authors}
\\[8pt]
{\small
\href{https://harahan.github.io/meanflownft-project-page/}
  {\raisebox{-0.12em}{\twemoji[height=0.95em]{globe with meridians}}\ Project Page}
\quad
\href{https://github.com/Harahan/MeanFlowNFT}
  {\raisebox{-0.12em}{\twemoji[height=0.95em]{laptop}}\ GitHub}
\quad
\href{https://huggingface.co/Harahan/MeanFlowNFT}
  {\raisebox{-0.12em}{\twemoji[height=0.95em]{smiling face with open hands}}\ Hugging Face}}
\end{center}
\vskip 6mm

\begin{tcolorbox}[
  colframe=abstractbg,
  colback=abstractbg,
  boxrule=0pt,
  arc=2mm,
  enhanced,
  top=12pt,
  bottom=12pt,
  left=15pt,
  right=15pt,
  width=\textwidth,
]
\textbf{Abstract.}\quad
MeanFlow generators achieve fast few-step sampling by predicting average
velocities over time intervals, making them attractive for efficient generation.
Reinforcement learning (RL) has become a powerful way to align diffusion and
flow models with human preferences and task-specific objectives.
In particular, DiffusionNFT offers an efficient forward-process RL framework
that does not require reverse-process trajectories or likelihood estimation.
However, applying such RL methods to MeanFlow remains underexplored.
DiffusionNFT optimizes instantaneous velocities, whereas MeanFlow samples with
average velocities.
To bridge this gap, we introduce \textbf{MeanFlowNFT}.
Inspired by the MeanFlow identity, which bridges average and instantaneous
velocities, we construct an induced instantaneous-velocity predictor.
We apply the DiffusionNFT objective to this predictor, making reward
optimization well-defined for MeanFlow.
Sampling remains based on the average velocity, preserving MeanFlow's fast
few-step generation.
We further prove that MeanFlowNFT inherits DiffusionNFT's strict
policy-improvement guarantee.
Experiments on image and video generation show that MeanFlowNFT consistently
improves baselines.
Moreover, it outperforms prior state-of-the-art RL-tuned few-step generators on
most metrics ($6$ of $8$ on \sdmodel), and can even surpass multi-step RL-tuned
diffusion while using only a few sampling steps.
For instance, on \wanmodel, $4$-step MeanFlowNFT reaches a VBench score of
$84.33$, surpassing $50$-step \longcat RL ($82.57$).

\vskip 8pt
\textbf{Date:} July 2026
\end{tcolorbox}

\begin{figure}[!ht]
  \centering
  \setlength{\fboxsep}{0pt}
  \def\bimg#1{{\color{black!35}\fbox{\includegraphics[width=\dimexpr\linewidth-2\fboxrule\relax]{#1}}}}%
  \captionsetup[subfigure]{skip=1pt,belowskip=0.5pt}%
  \begin{minipage}[t]{0.510\textwidth}
  \centering
  {\fontsize{6}{6.4}\selectfont\itshape
  ``A sleek car keeps driving down a neon-lit cyberpunk street, its glowing red
  taillights trailing as the city lights streak past.''\par}\vspace{-1pt}
  \begin{subfigure}{\linewidth}\centering
    \bimg{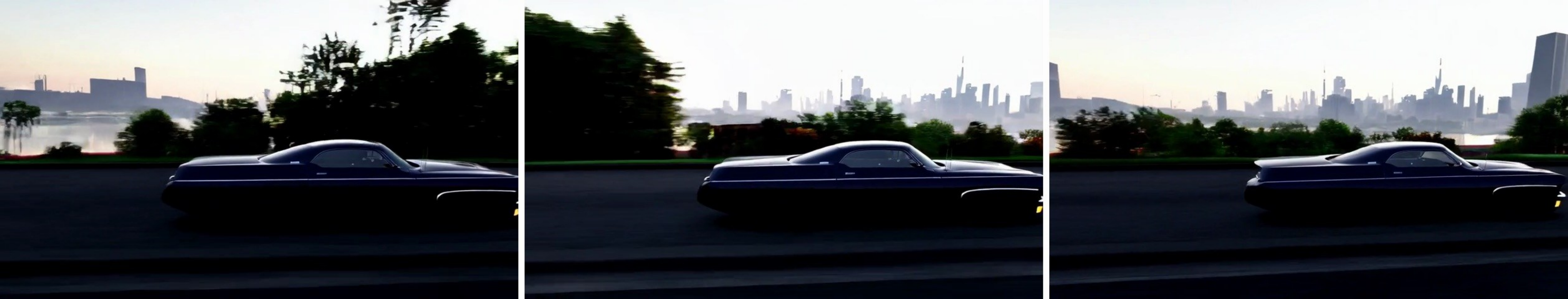}
    \caption{\wanmodel~1.3B (50 steps)}
  \end{subfigure}
  \begin{subfigure}{\linewidth}\centering
    \bimg{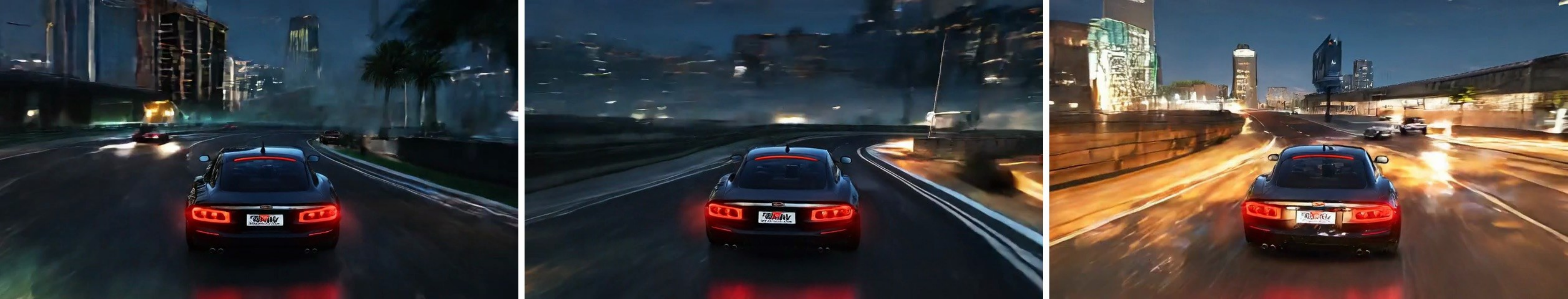}
    \caption{\longcat RL (50 steps)}
  \end{subfigure}
  \begin{subfigure}{\linewidth}\centering
    \bimg{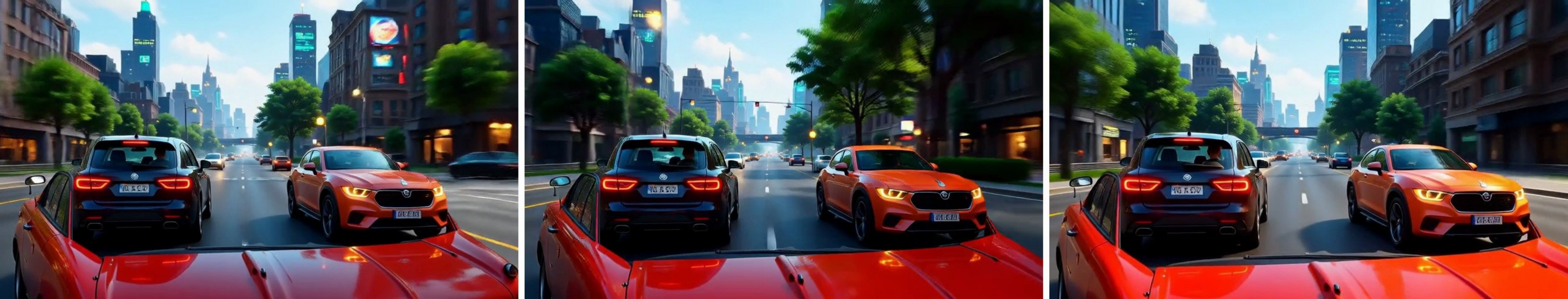}
    \caption{AnyFlow (4 steps)}
  \end{subfigure}
  \begin{subfigure}{\linewidth}\centering
    \bimg{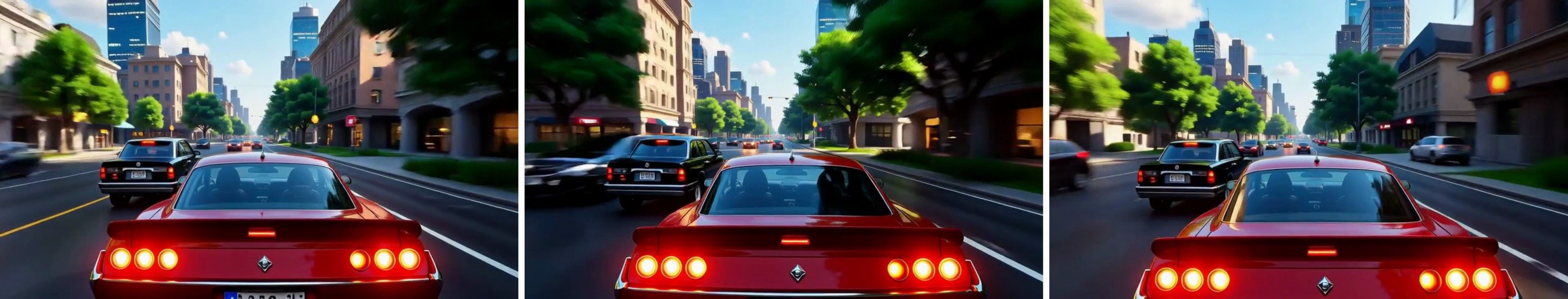}
    \caption{\textbf{MeanFlowNFT (4 steps)}}
  \end{subfigure}
  \vspace{-0.22in}
  \caption{Qualitative comparisons on \wanmodel~1.3B. Each row shows $3$ frames
  sampled uniformly over time.}
  \label{fig:teaser_video}
  \end{minipage}%
  \hspace{0.008\textwidth}%
  \begin{minipage}[t]{0.452\textwidth}
  \centering
  \captionsetup[subfigure]{font=tiny}%
  {\fontsize{6}{6.4}\selectfont\itshape
  ``a frosted donut with a bite out of it''\par}\vspace{-3.5pt}
  {\setlength{\parskip}{0pt}%
  \noindent
  \begin{subfigure}[t]{0.49\linewidth}\centering
    \bimg{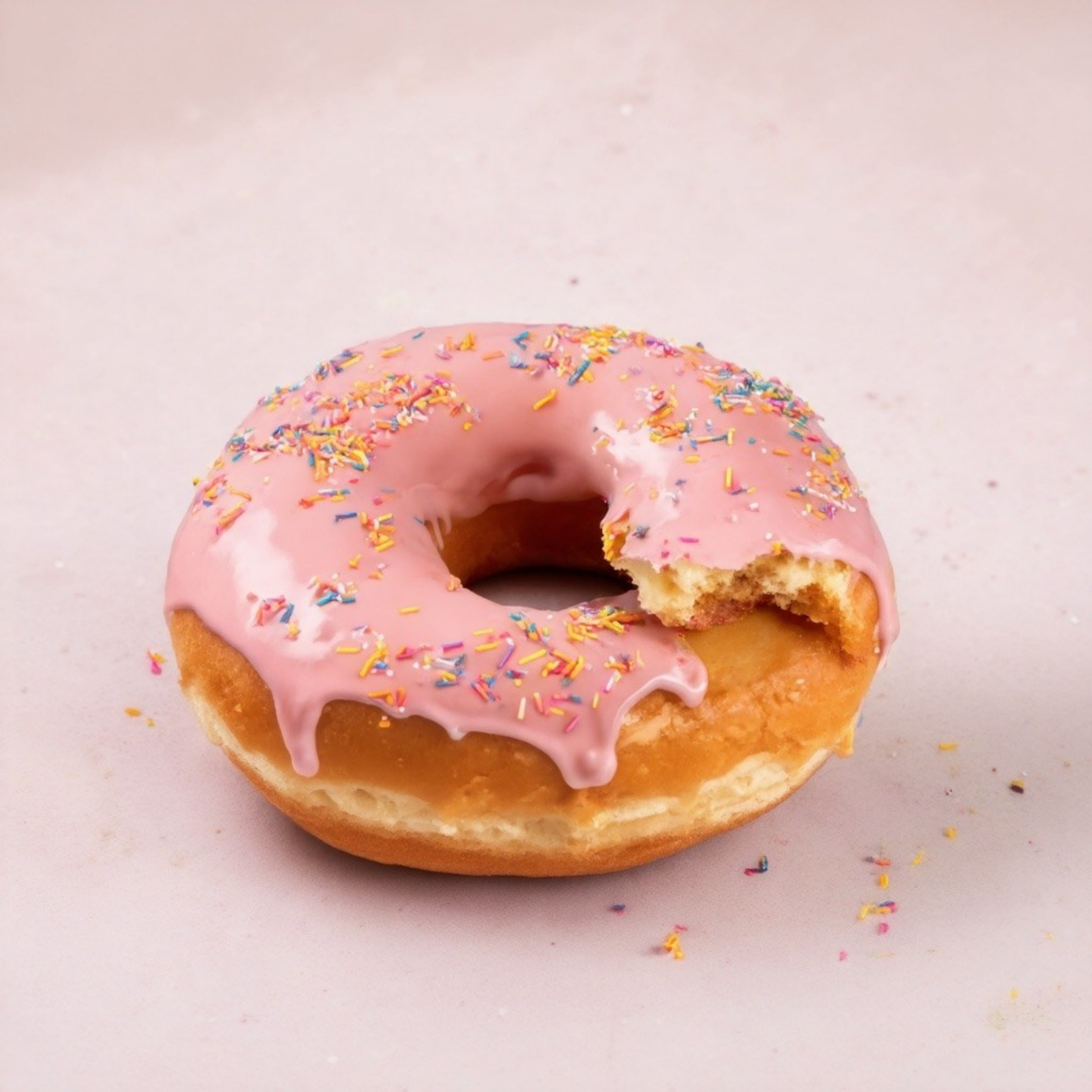}
    \caption{Flow-GRPO (40 steps)}
  \end{subfigure}\hfill
  \begin{subfigure}[t]{0.49\linewidth}\centering
    \bimg{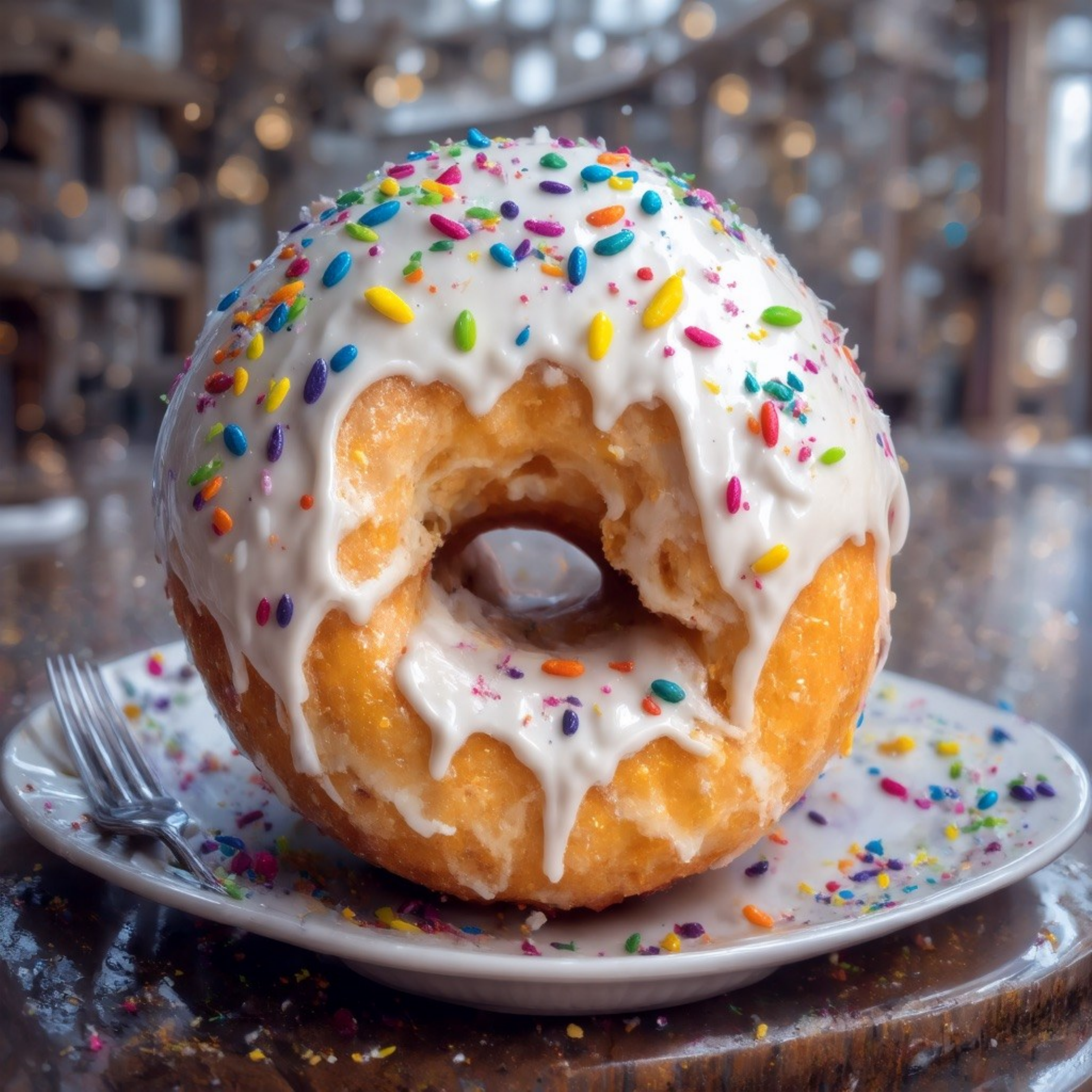}
    \caption{DiffusionNFT (40 steps)}
  \end{subfigure}
  \par\vspace{1pt}
  \noindent
  \begin{subfigure}[t]{0.49\linewidth}\centering
    \bimg{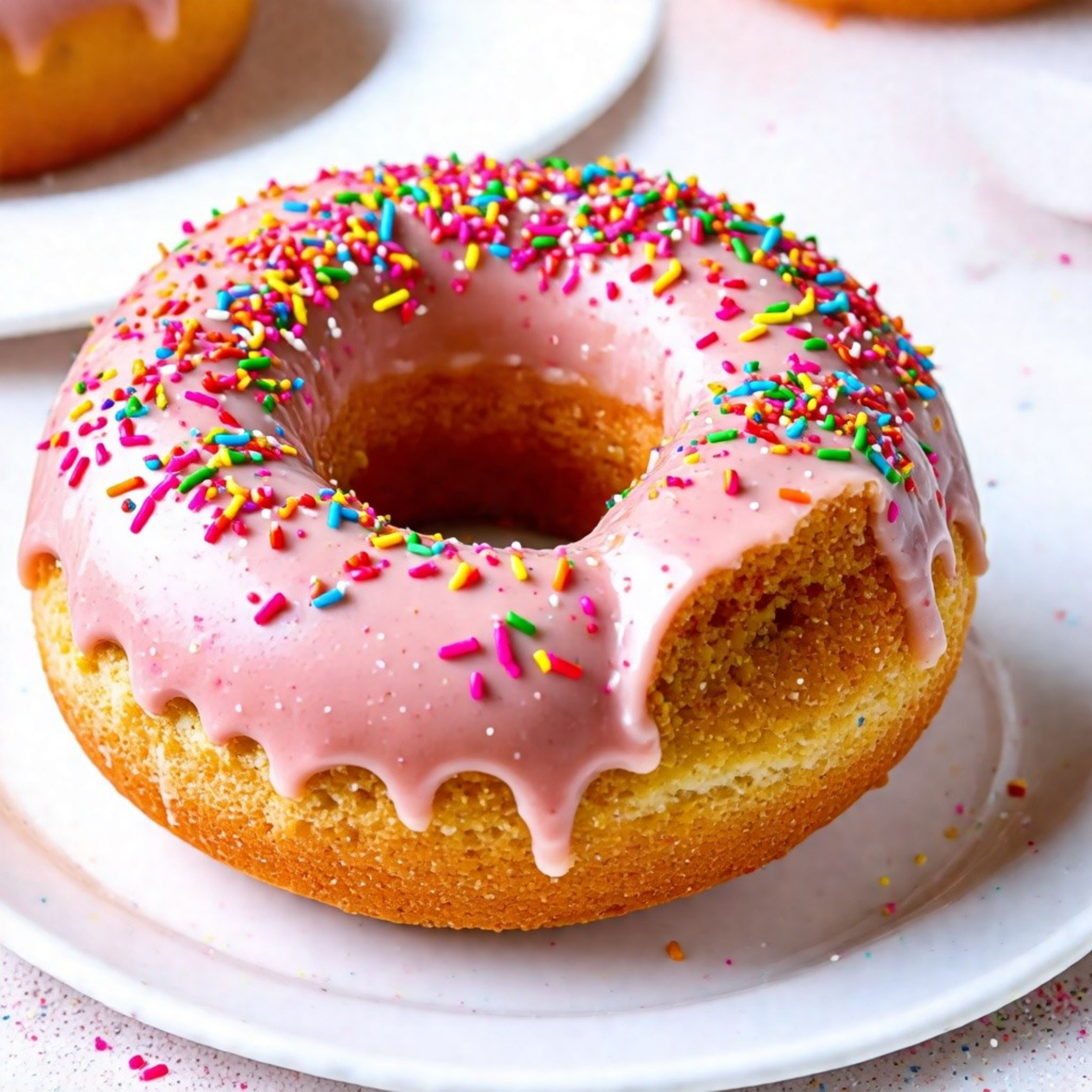}
    \caption{AnyFlow (4 steps)}
  \end{subfigure}\hfill
  \begin{subfigure}[t]{0.49\linewidth}\centering
    \bimg{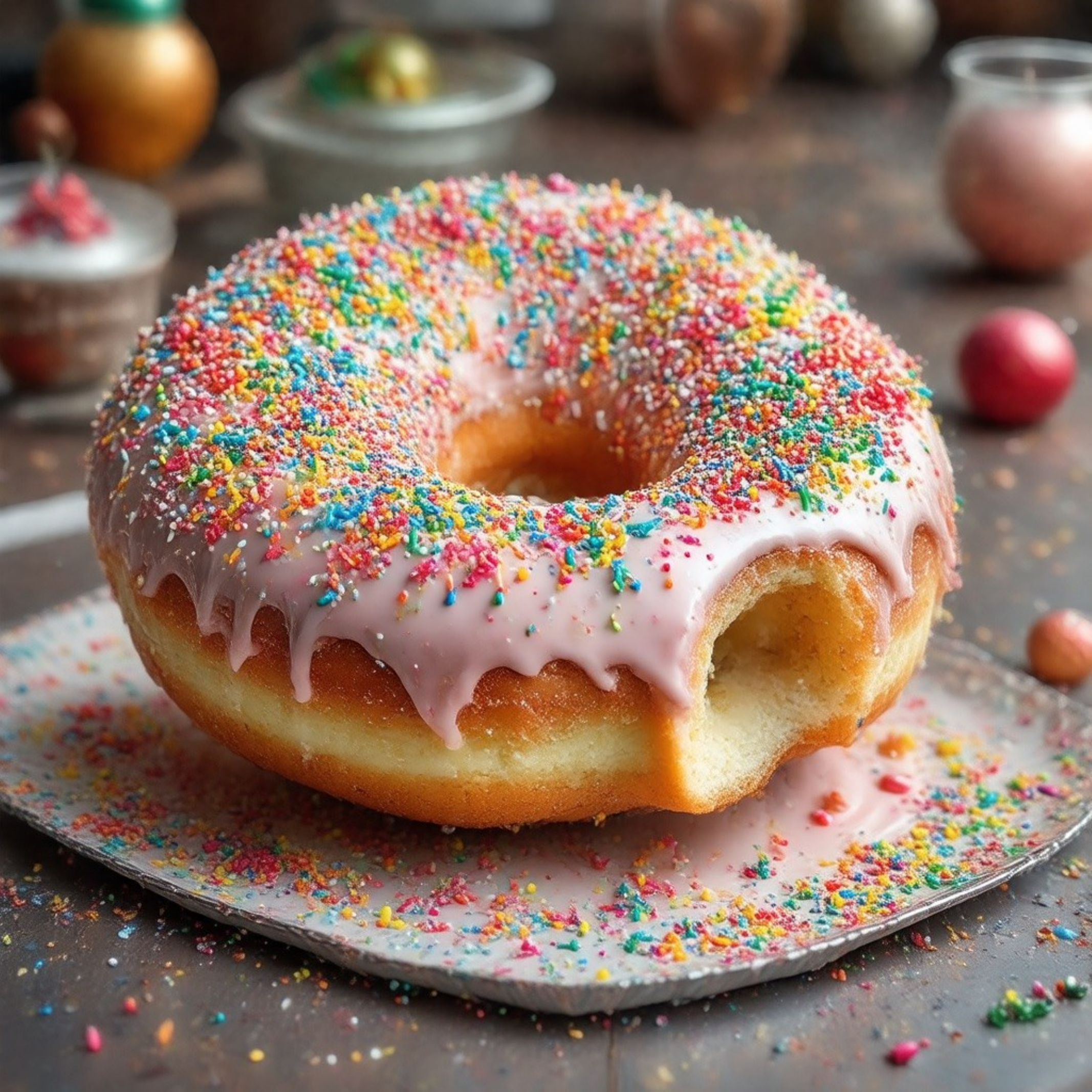}
    \caption{\textbf{MeanFlowNFT (4 steps)}}
  \end{subfigure}
  }
  \vspace{-0.008in}
  \caption{Qualitative comparisons on \sdmodel. More visual results are provided
  in \cref{app:qualitative_results}.}
  \label{fig:teaser_image}
  \end{minipage}
\end{figure}

\section{Introduction}
\label{sec:intro}
Diffusion~\citep{ddpm,ddim,score-base-diff} and flow~\citep{lipman2022flow,liu2022flow} models are the
dominant paradigm for high-quality image and video generation. However, they synthesize a sample by
integrating an instantaneous velocity over many sequential steps~\citep{sdxl,seedream4,lumina2}, which
makes generation slow. MeanFlow~\citep{meanflow} removes this bottleneck by predicting the
\emph{average velocity} over a time interval rather than the instantaneous velocity, so that one or a
few steps suffice. This few-step efficiency makes MeanFlow an increasingly practical deployment target.

Reinforcement learning (RL)~\citep{ddpo,fan2023dpok,diffusiondpo} from a scalar reward has become a
standard tool for aligning such generative models with downstream objectives. For diffusion and flow
models, the dominant recipe~\citep{flowgrpo,xue2025dancegrpo} discretizes the reverse generative
process and applies GRPO-style policy gradients~\citep{guo2025deepseek}, which require a stochastic
policy and per-step likelihood estimation. DiffusionNFT~\citep{diffusionnft} instead performs online
RL \emph{directly on the forward process}: it contrasts ``positive'' and ``negative'' generations
(split by rewards) to form an implicit policy-improvement direction and folds rewards into a
flow-matching objective. This makes training likelihood-free, solver-agnostic, and far more efficient
than GRPO-based methods. However, RL for MeanFlow has received little attention. A key obstacle is that
forward-process RL such as DiffusionNFT optimizes the \emph{instantaneous} velocity, whereas MeanFlow
samples with the \emph{average} velocity. This raises a natural question:

\begin{tcolorbox}[greenhighlight, before skip=8pt, after skip=8pt, top=0.3mm, bottom=0.3mm, colback=mygreen!5, colframe=mygreen!40]
\centering
\emph{Can we finetune a pretrained MeanFlow model with efficient forward-process RL to enable superior few-step generation?}
\end{tcolorbox}

In this paper, we answer this question affirmatively with \textbf{MeanFlowNFT}, the first
forward-process RL framework for MeanFlow generators. Our starting point is the MeanFlow identity
(\cref{eq:identity}), an intrinsic link between the average and instantaneous velocity. Through it, a
MeanFlow network yields an \textit{induced instantaneous-velocity predictor}, to which we apply the
DiffusionNFT-style objective during RL training. This brings two benefits. As in DiffusionNFT, training is
likelihood-free and stays on the forward noising process. It also decouples optimization from
sampling, so inference still uses MeanFlow's efficient few-step sampler. The objective never acts on
the average velocity explicitly. Still, we prove in an idealized setting that the optimal induced
predictor recovers DiffusionNFT's improved policy and carries this improvement over to the
average-velocity network. On the practical side, we approximate the total-derivative
terms in the induced predictors by finite differences.
The trainable and reference predictors in our algorithm use the same finite-difference estimate, computed along the
forward-process conditional velocity. As a result,
MeanFlowNFT delivers strong generation quality (\cref{fig:teaser_image,fig:teaser_video}) while
preserving MeanFlow's few-step efficiency.

Our contributions are summarized as follows:
\begin{itemize}[leftmargin=*,nosep]
\item We propose \textbf{MeanFlowNFT}, the first forward-process RL framework for
MeanFlow generators: it applies the DiffusionNFT-style objective to an induced
instantaneous-velocity predictor derived from the average-velocity network. This keeps training
likelihood-free and leaves the efficient few-step sampler unchanged.
\item We provide theoretical guarantees: in an idealized setting, the optimum of the induced
predictor matches the DiffusionNFT improved-policy target, and this policy improvement provably
carries over to MeanFlow's average-velocity generator.
\item Comprehensive experiments across text-to-image and text-to-video generation show that MeanFlowNFT
consistently improves MeanFlow baselines and outperforms prior few-step methods. It even surpasses
multi-step RL with far fewer sampling steps while scaling gracefully at test time.
\end{itemize}

\section{Preliminaries}
\label{sec:prelims}

Throughout, $\vx_t$ denotes a noised sample at time $t$ and $\vc$ a conditioning prompt. We write $\vu$ for an average velocity~\citep{meanflow} and $\vv$ for an instantaneous velocity field.

\noindent\textbf{Notation.} \textit{Unless stated otherwise, all velocity fields and predictors are conditioned on $\vc$, but we omit $\vc$ from their arguments for notational brevity.}

\subsection{Flow Matching}
\label{sec:fm}

Flow Matching~\citep{lipman2022flow,liu2022flow} learns a probability-flow ODE
that transports Gaussian noise to the data distribution $\pi(\vx_0\mid\vc)$.
Given a schedule $(\alpha_t,\sigma_t)$ and writing
$\dot f_t\coloneq \mathrm{d}f_t/\mathrm{d}t$, the forward process is $\vx_t = \alpha_t\vx_0 + \sigma_t\epsilonv$, with
$\vx_0\sim\pi(\cdot\mid\vc)$ and $\epsilonv\sim\gN(\bm{0},\rmI)$.
Differentiating along a fixed pair $(\vx_0,\epsilonv)$ gives the
\emph{conditional velocity} $\vv_t \triangleq \dot\alpha_t\vx_0 + \dot\sigma_t\epsilonv$.
For rectified flow~\citep{liu2022flow}, $\alpha_t=1-t$ and $\sigma_t=t$, so
$\vv_t=\epsilonv-\vx_0$. Flow Matching trains a velocity predictor
$\vv_\theta(\vx_t,t)$ by minimizing
$\mathcal{L}_{\mathrm{FM}}(\theta) = \E_{t,\,\vc,\,\vx_0\sim\pi(\cdot\mid\vc),\,\epsilonv}\!\left[ w(t)\,\|\vv_\theta(\vx_t,t)-\vv_t\|_2^2 \right]$.
Here $w(t)$ is a time-dependent loss weighting.

Although $\vv_t$ is random given $\vx_t$, the optimal predictor under this
squared loss is the deterministic \emph{marginal (instantaneous) velocity}
\begin{equation}
\label{eq:marg_vel}
\vv(\vx_t,t) \,\triangleq\, \E[\vv_t\mid \vx_t,\vc,t],
\end{equation}
the conditional velocity averaged over the posterior of $(\vx_0,\epsilonv)$ given
$\vx_t$~\citep{lipman2022flow}. At inference one integrates
$\mathrm{d}\vx_t/\mathrm{d}t=\vv(\vx_t,t)$ from noise to data using
$\vv_\theta\approx\vv$, which typically requires many network evaluations.

\subsection{MeanFlow}
\label{sec:mf-prelim}

Flow Matching sampling is costly because integrating the generative ODE requires
many instantaneous-velocity evaluations. MeanFlow~\citep{meanflow} instead learns
a finite-interval flow map: predicting the average velocity from time $t$ to an
earlier time $s$ lets sampling take large jumps rather than many small steps. It considers the \emph{average velocity} over an
interval $[s,t]$ along the ODE induced by $\vv$ (\cref{eq:marg_vel}),
\begin{equation}
\label{eq:avg_vel}
\vu(\vx_t,s,t)
\,\triangleq\,
\frac{1}{t-s}\int_s^t \vv(\vx_\tau,\tau)\,\mathrm{d}\tau,
\qquad
\frac{\mathrm{d}\vx_\tau}{\mathrm{d}\tau}=\vv(\vx_\tau,\tau),
\end{equation}
where the integral runs along the ODE trajectory from $\vx_t$ to time $s$, so
$\vu$ depends on both endpoints ($t,s$) and recovers the instantaneous velocity as
$s\to t$, $\lim_{s\to t}\vu(\vx_t,s,t)=\vv(\vx_t,t)$.

\noindent\textbf{MeanFlow identity and training.}
Directly regressing $\vu$ is impractical because \cref{eq:avg_vel}
contains a path integral. Differentiating the displacement form
$(t-s)\,\vu=\int_s^t \vv\,\mathrm{d}\tau$ with respect to $t$ yields
the exact \emph{MeanFlow identity}
\begin{equation}
\label{eq:identity}
\vv(\vx_t,t)
\,=\,
\vu(\vx_t,s,t)
\,+\,
(t-s)\,\frac{\mathrm{d}}{\mathrm{d}t}\vu(\vx_t,s,t),
\end{equation}
with total derivative
$\tfrac{\mathrm{d}}{\mathrm{d}t}\vu=\partial_t\vu+(\partial_{\vx}\vu)\,\vv$.
Because the marginal velocity is the posterior mean of $\vv_t$ (\cref{eq:marg_vel}),
replacing this intractable target with the conditional velocity $\vv_t$ gives a
trainable regression target. Specifically, a MeanFlow
network $\vu_\theta(\vx_t,s,t)$ is trained by minimizing
$\mathcal{L}_{\mathrm{MF}}(\theta) = \E_{s,t,\,\vc,\,\vx_0,\,\epsilonv}\!\left[w(t)\,\|\vu_\theta-\sg(\vu_{\mathrm{tgt}})\|_2^2\right]$, with target
$\vu_{\mathrm{tgt}} = \vv_t-(t-s)\,(\partial_t\vu_\theta + (\partial_{\vx}\vu_\theta)\,\vv_t)$.
The stop-gradient $\sg(\cdot)$ is applied
to the whole target, so no gradient backpropagates through the total derivative term.
When $s=t$, the derivative correction term vanishes and $\mathcal{L}_{\mathrm{MF}}$ reduces to $\mathcal{L}_{\mathrm{FM}}$.

\noindent\textbf{Few-step sampling.}
Since $\vu$ is the exact average velocity over $[s,t]$, it obeys the exact
displacement identity $\vx_s=\vx_t-(t-s)\,\vu(\vx_t,s,t)$. The MeanFlow
sampler therefore updates
\begin{equation}
\label{eq:meanflow_sampling}
\vx_{t_{i-1}}
\,=\,
\vx_{t_i} - (t_i-t_{i-1})\,\vu_\theta(\vx_{t_i},t_{i-1},t_i),
\qquad i=N,\ldots,1.
\end{equation}
No instantaneous velocity is evaluated at inference, and a single step suffices
in principle when $\vu_\theta$ matches the true average velocity.

\subsection{DiffusionNFT}
\label{sec:diffnft}

DiffusionNFT~\citep{diffusionnft} is an online diffusion RL method built on the
forward-process Flow Matching objective rather than on policy gradients through a
discretized reverse sampler.

\noindent\textbf{Optimality partition.}
Let $\pi^{\mathrm{old}}$ be a frozen data-collection policy with marginal velocity
$\vv^{\mathrm{old}}(\vx_t,t)$. For each prompt $\vc$ one samples $K$ images
$\vx_0^{1:K}\sim\pi^{\mathrm{old}}(\cdot\mid\vc)$ and scores each with a reward
$r(\vx_0,\vc)\in[0,1]$, read as the optimality probability
$r(\vx_0,\vc)\triangleq p(\mathbf{o}=1\mid\vx_0,\vc)$ in the RL-as-inference
view~\citep{levine2018reinforcement}. This induces positive and negative posteriors of $\pi^{\mathrm{old}}$,
$\pi^{+}\propto r\,\pi^{\mathrm{old}}$ and $\pi^{-}\propto(1-r)\,\pi^{\mathrm{old}}$,
normalized by $Z(\vc)\triangleq\E_{\pi^{\mathrm{old}}}[r]$ and $1-Z(\vc)$, respectively.
Under the objective $J(\pi)=\E_{\pi(\cdot\mid\vc)}[r]$, with $\pi\succ\pi'$
denoting $J(\pi)>J(\pi')$, one has $\pi^{+}\succ\pi^{\mathrm{old}}\succ\pi^{-}$ for
any non-degenerate (non-constant) reward, so $\pi^{+}$ is a valid improved policy.

\noindent\textbf{Reinforcement guidance.}
Rather than treating $\pi^{+}$ as an optimization \emph{point} (rejection
finetuning, which discards negatives), DiffusionNFT extracts an optimization
\emph{direction} from the triplet $(\pi^{+},\pi^{\mathrm{old}},\pi^{-})$. With
\begin{equation}
\label{eq:alpha}
\alpha(\vx_t,\vc)
\,\triangleq\,
\frac{\pi_t^{+}(\vx_t\mid\vc)}{\pi_t^{\mathrm{old}}(\vx_t\mid\vc)}\,Z(\vc)
\,=\, \E_{\pi^{\mathrm{old}}}[r\mid\vx_t,\vc]
\,\in [0,1],
\end{equation}
because $\vv(\vx_t,t)$ is a posterior mean (\cref{eq:marg_vel}), the marginal-velocity
decomposition $\vv^{\mathrm{old}}(\vx_t,t)=\alpha\,\vv^{+}(\vx_t,t)+(1-\alpha)\,\vv^{-}(\vx_t,t)$ holds
\citep[Thm.~3.1]{diffusionnft}, giving
the shared \emph{reinforcement guidance}
\begin{equation}
\label{eq:delta}
\Delta(\vx_t,\vc,t)
\,\triangleq\,
\alpha\bigl(\vv^{+}(\vx_t,t)-\vv^{\mathrm{old}}(\vx_t,t)\bigr)
\,=\,
(1-\alpha)\bigl(\vv^{\mathrm{old}}(\vx_t,t)-\vv^{-}(\vx_t,t)\bigr).
\end{equation}
Guiding the reference model along $\Delta$ with strength $1/\beta$ defines the
target $\vv^{*}(\vx_t,t)\triangleq\vv^{\mathrm{old}}(\vx_t,t)+\tfrac{1}{\beta}\Delta$, which mirrors
classifier-free guidance~\citep{ho2021classifierfree} and recovers $\vv^{*}(\vx_t,t)=\vv^{+}(\vx_t,t)$ at $\beta=\alpha$.

\noindent\textbf{Implicit parameterization.}
Instead of learning separate positive and negative models, DiffusionNFT uses the
implicit parameterization $\vv_\theta^{+}(\vx_t,t) := (1-\beta)\,\vv^{\mathrm{old}}(\vx_t,t)+\beta\,\vv_\theta(\vx_t,t)$
and $\vv_\theta^{-}(\vx_t,t) := (1+\beta)\,\vv^{\mathrm{old}}(\vx_t,t)-\beta\,\vv_\theta(\vx_t,t)$, and optimizes the reward-weighted objective
\begin{equation}
\label{eq:diffnft_loss}
\mathcal{L}_{\mathrm{DNFT}}(\theta)
\,=\,
\underset{\substack{\vc,\;\vx_0\sim\pi^{\mathrm{old}}(\cdot\mid\vc),\\
\epsilonv\sim\gN(\bm{0},\rmI),\;t}}{\E}
\left[
r\,\|\vv_\theta^{+}(\vx_t,t)-\vv_t\|_2^2
\,+\,
(1-r)\,\|\vv_\theta^{-}(\vx_t,t)-\vv_t\|_2^2
\right].
\end{equation}
At the exact optimum, \citet[Thm.~3.2]{diffusionnft} show
$\vv_{\theta^*}(\vx_t,t)=\vv^{\mathrm{old}}(\vx_t,t)+\tfrac{2}{\beta}\,\Delta$. In particular, when
$\beta=2\alpha$ the optimum coincides with $\vv^{+}(\vx_t,t)$, so the trained model itself
realizes the positive policy improvement.

\section{MeanFlow Reinforcement via Forward-Process RL}
\label{sec:method}

Here, we first build the proposed MeanFlowNFT (\cref{sec:theory}), which finetunes a MeanFlow generator with DiffusionNFT-style RL. Then, we analyze its closed-form induced
optimum and policy-improvement guarantee (\cref{sec:mf-theory}). Finally, we present the practical implementation
(\cref{sec:algorithm}).

\subsection{MeanFlowNFT}
\label{sec:theory}
DiffusionNFT acts on the \emph{instantaneous} velocity, while a MeanFlow network predicts the \emph{average} velocity. 
Our key idea is to \emph{keep the network in average-velocity space, but carry out optimization in instantaneous-velocity space.}

To achieve this, for any interval with $s\le t$, we substitute $\vu_\theta$ into the MeanFlow identity (\cref{eq:identity}) to build an \emph{induced instantaneous-velocity predictor} $\vV_\theta$\footnote{For clarity,
lowercase $\vv$ denotes marginal instantaneous velocities of policies (\emph{e.g.}, $\vv^{+}$ of $\pi^{+}$), while
uppercase $\vV$ denotes instantaneous predictors induced from MeanFlow
average-velocity networks (\emph{e.g.}, $\vV_\theta$ induced by $\vu_\theta$).},
\begin{equation}
\label{eq:V_theta}
\vV_\theta(\vx_t,s,t)
\;\triangleq\;
\vu_\theta(\vx_t,s,t)
\;+\;
(t-s)\left[
\partial_t\vu_\theta(\vx_t,s,t)
+
(\partial_{\vx}\vu_\theta)(\vx_t,s,t)\,
\widehat{\vv}_\theta(\vx_t,t)
\right].
\end{equation}
Here $\widehat{\vv}_\theta(\vx_t,t)\triangleq\vu_\theta(\vx_t,t,t)$ is the network's instantaneous velocity at time $t$, 
obtained by letting $s\to t$ as in \cref{eq:avg_vel}. Inside \cref{eq:V_theta}, $\widehat{\vv}_\theta(\vx_t,t)$ serves only as the direction that $\partial_{\vx}\vu_\theta$ acts on, whereas $\vV_\theta(\vx_t,s,t)$ is the quantity we actually optimize.
In the idealized setting, \cref{eq:identity} gives
$\vV_\theta(\vx_t,s,t)=\widehat{\vv}_\theta(\vx_t,t)$ for all $s\le t$, so the
induced velocity is identical across interval starts for fixed $(\vx_t,\vc,t)$.
Yet $\vV_\theta$ is constructed from the full-interval prediction
$\vu_\theta(\vx_t,s,t)$, so optimizing it still reinforces the average velocity.
In contrast, optimizing the single-time $\widehat{\vv}_\theta$, 
\emph{i.e.}, using only the $s=t$ case of \cref{eq:V_theta}, 
would reduce to plain Flow Matching~\citep{liu2022flow} and lose MeanFlow's average-velocity parameterization and few-step sampling. 
To keep the update on $\vu_\theta(\vx_t,s,t)$ and reduce costs, 
we wrap the total-derivative term of \cref{eq:V_theta} in a stop-gradient during optimization (\emph{i.e.}, $\mathrm{sg}\big(\partial_t\vu_\theta+(\partial_{\vx}\vu_\theta)\,\widehat{\vv}_\theta\big)$).

Now, we apply a DiffusionNFT-style objective to $\vV_\theta$ to reinforce the underlying average velocity
$\vu_\theta(\vx_t, s, t)$. Concretely, let $\vV^{\mathrm{old}}$ be the same construction applied to the frozen reference $\vu^{\mathrm{old}}$~\footnote{Following DiffusionNFT~\citep{diffusionnft}, $\vu^{\mathrm{old}}$ is an exponential moving average (EMA) of $\vu_\theta$.}. We then define implicit ``positive'' and ``negative'' predictors in instantaneous-velocity space, $\vV_\theta^{+}(\vx_t,s,t):=(1-\beta)\vV^{\mathrm{old}}(\vx_t,s,t)+\beta\vV_\theta(\vx_t,s,t)$ and $\vV_\theta^{-}(\vx_t,s,t):=(1+\beta)\vV^{\mathrm{old}}(\vx_t,s,t)-\beta\vV_\theta(\vx_t,s,t)$. Mirroring \cref{eq:diffnft_loss}, 
we optimize the following objective
\begin{equation}
\label{eq:mfnft_loss}
\mathcal{L}_{\mathrm{MFNFT}}(\theta)
\;=\;
\underset{\substack{\vc,\;\vx_0\sim\pi^{\mathrm{old}}(\cdot\mid\vc),\\ \epsilonv\sim\gN(\bm{0},\rmI),\;s\le t}}{\E}
\left[
r\,\|\vV_\theta^{+}(\vx_t,s,t)-\vv_t\|_2^2
\;+\;
(1-r)\,\|\vV_\theta^{-}(\vx_t,s,t)-\vv_t\|_2^2
\right].
\end{equation}
Here $\pi^{\mathrm{old}}$ is the policy whose average velocity is $\vu^{\mathrm{old}}$, and $\vx_0\sim\pi^{\mathrm{old}}$ means $\vx_0$ 
is obtained by running the few-step MeanFlow sampler (\cref{eq:meanflow_sampling}) with $\vu^{\mathrm{old}}$. 
The reward $r$ and the other symbols follow the same definitions as in DiffusionNFT (\cref{sec:diffnft}).

\begin{figure}[!ht]
\centering
\includegraphics[width=0.83\textwidth]{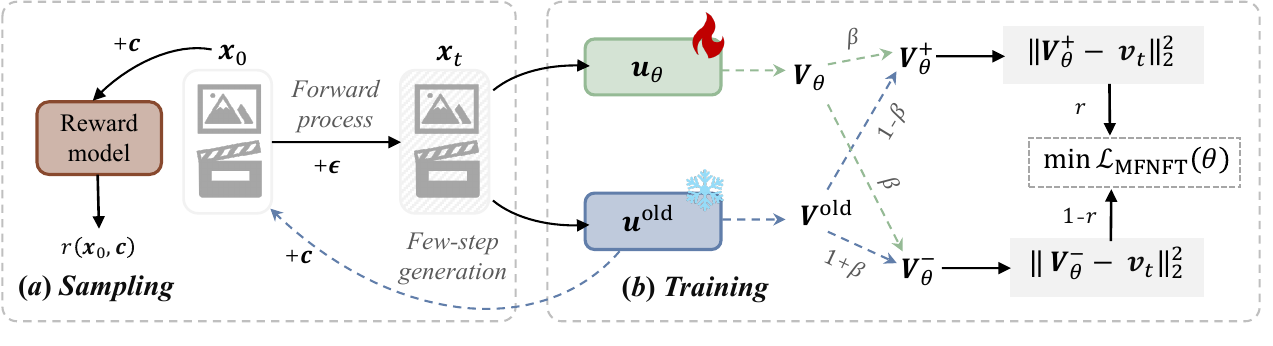}
\vspace{-0.05in}
\caption{MeanFlowNFT keeps MeanFlow's \textit{average-velocity}
parameterization (\cref{eq:avg_vel}) and native few-step sampler (\cref{eq:meanflow_sampling}), while constructing an \textit{induced
instantaneous-velocity predictor} $\vV_\theta$ for forward-process negative-aware finetuning optimization.}
\label{fig:overview}
\vspace{-0.05in}
\end{figure}
In effect, MeanFlowNFT directly optimizes $\vV_\theta$, and this improvement transfers to the average velocity $\vu_\theta$. The overall pipeline is depicted in \cref{fig:overview}, and it gives two key benefits.
\textbf{(i)} Training depends only on the forward noising process: following DiffusionNFT, it is likelihood-free and never unrolls the reverse denoising process, in contrast to GRPO-style policy gradients~\citep{flowgrpo,mixgrpo} that discretize the reverse sampler and estimate per-step likelihoods.
\textbf{(ii)} It decouples optimization from sampling: while optimization acts in \emph{instantaneous-velocity} space, inference and RL sampling still deploy the \emph{average-velocity} $\vu_\theta$ through its efficient few-step MeanFlow sampler.

\subsection{Optimum and Policy Improvement}
\label{sec:mf-theory}

In this subsection, we provide theoretical guarantees for MeanFlowNFT in three steps: the idealized
pointwise optimum in induced-predictor space
(\cref{prop:mfnft_opt}), its recovery of the improved policy's
marginal velocity (\cref{cor:V_improvement}), and the transfer of this
induced-velocity guarantee to the average-velocity network
$\vu_\theta$ (\cref{thm:u_improvement}). All proofs are provided in \cref{sec:appendix-proofs}.

The objective \cref{eq:mfnft_loss} matches the predictors
$\vV_\theta^{\pm}$ to the conditional velocity $\vv_t$. Treating the induced
prediction at each $(\vx_t,\vc,s,t)$ as an unconstrained value gives the following
closed-form optimum.
\begin{proposition}[Idealized pointwise optimum]
\label[proposition]{prop:mfnft_opt}
Conditioned on $(\vx_t,\vc,s,t)$, the idealized pointwise minimizer is
\begin{equation}
\label{eq:opt_form}
\begin{aligned}
\vV_{\theta^*}(\vx_t,s,t)
&\;=\;
\vV^{\mathrm{old}}(\vx_t,s,t)
\;+\;
\frac{2}{\beta}\,\widehat{\Delta}(\vx_t,\vc,s,t),
\\[2pt]
\widehat{\Delta}(\vx_t,\vc,s,t)
&\;\triangleq\;
\frac{1}{2}
\E_{\pi^{\mathrm{old}}}\!\left[
(2r-1)\bigl(\vv_t-\vV^{\mathrm{old}}(\vx_t,s,t)\bigr)
\,\middle|\, \vx_t,\vc,s,t
\right].
\end{aligned}
\end{equation}
\end{proposition}

The optimum mirrors the DiffusionNFT fixed point
$\vv_{\theta^*}(\vx_t,t)=\vv^{\mathrm{old}}(\vx_t,t)+\tfrac{2}{\beta}\Delta$ (\cref{eq:delta}), with the induced
reference $\vV^{\mathrm{old}}$ and guidance $\widehat{\Delta}$ in place of
$\vv^{\mathrm{old}}$ and $\Delta$.

\begin{corollary}[$\vV_{\theta^*}$ recovers the improved marginal velocity]
\label[corollary]{cor:V_improvement}
Setting the guidance strength to $\beta=2\alpha(\vx_t,\vc)$, \cref{eq:opt_form} collapses to
\begin{equation}
\label{eq:V_is_vplus}
\vV_{\theta^*}(\vx_t,s,t)=\vv^{+}(\vx_t,t)
\end{equation}
for all $s\le t$, the marginal instantaneous velocity of the improved policy $\pi^{+}$.
\end{corollary}
This gives the DiffusionNFT improvement target in $\vV$-space. To transfer it to the
deployed average velocity, we use the following consequence of the MeanFlow identity.
\begin{lemma}[MeanFlow consistency]
\label[lemma]{lem:mf_consistency}
Let $\vv$ be an instantaneous velocity field. If $\vu$ satisfies
\begin{equation}
\label{eq:mf_consistency}
\vu(\vx_t,s,t)
+
(t-s)
\left[
\partial_t\vu(\vx_t,s,t)
+
(\partial_{\vx}\vu)(\vx_t,s,t)\vv(\vx_t,t)
\right]
=
\vv(\vx_t,t)
\end{equation}
for all $s\le t$, then $\vu$ is the exact average velocity of the ODE
$\dot\vx_\tau=\vv(\vx_\tau,\tau)$ over $[s,t]$.
\end{lemma}

\begin{theorem}[MeanFlow policy improves]
\label{thm:u_improvement}
In the setting of \cref{cor:V_improvement}, if the induced optimum is attained
for all intervals $s\le t$, the optimal
average velocity $\vu_{\theta^*}$ is the exact average velocity of the
ODE induced by $\vv^{+}$. Therefore the MeanFlow policy induced by $\vu_{\theta^*}$
coincides with $\pi^{+}$, and consequently
\begin{equation*}
J(\pi_{\theta^*}) = J(\pi^{+}) > J(\pi^{\mathrm{old}}).
\end{equation*}
\end{theorem}

Indeed, if $\vV_{\theta^*}(\vx_t,s,t)=\vv^{+}(\vx_t,t)$ for all
$s\le t$, then setting $s=t$ in \cref{eq:V_theta} removes the $(t-s)$ term and gives
$\widehat{\vv}_{\theta^*}(\vx_t,t)=\vu_{\theta^*}(\vx_t,t,t)=\vv^{+}(\vx_t,t)$.
Substituting this identity back into \cref{eq:V_theta} shows that $\vu_{\theta^*}$ 
satisfies \cref{eq:mf_consistency} with $\vv=\vv^{+}$. By \cref{lem:mf_consistency}, 
$\vu_{\theta^*}$ is then the exact average velocity of the ODE induced by $\vv^{+}$.

\subsection{Practical implementation}
\label{sec:algorithm}

\cref{eq:V_theta} provides the theoretical construction analyzed in
\cref{sec:mf-theory}. Guided by this construction, the practical implementation
introduces the following design choices to reduce computational cost and improve training
stability. \cref{alg:mfnft} summarizes the resulting MeanFlowNFT update.

\algrenewcommand\algorithmiccomment[1]{\hfill\textcolor{colBlue}{\footnotesize$\triangleright$~\textit{#1}}}
\begin{algorithm}[!ht]
\caption{MeanFlowNFT (one update step)}
\label{alg:mfnft}
\begin{algorithmic}[1]
\Require pretrained MeanFlow $\vu_\theta$, frozen reference $\vu^{\mathrm{old}}$ (EMA of $\vu_\theta$),
reward $r$, guidance $\beta$
\Statex \textcolor{colGray}{\texttt{// Sampling}}
\State sample prompt $\vc$; roll out $\vx_0$ with $\vu^{\mathrm{old}}$; evaluate $r(\vx_0,\vc)\in[0,1]$
\State sample an interval $s\le t$, and $\epsilonv\sim\gN(\bm{0},\rmI)$
\State $\vx_t \gets \alpha_t\vx_0+\sigma_t\epsilonv$; \quad
$\vv_t \gets \dot\alpha_t\vx_0+\dot\sigma_t\epsilonv$
\Statex \textcolor{colGray}{\texttt{// Training}}
\State $\vx_{t\pm\Delta t} \gets \vx_t \pm \Delta t\,\vv_t$
\State $\bm{d} \gets \big[\vu^{\mathrm{old}}(\vx_{t+\Delta t},s,t{+}\Delta t)-\vu^{\mathrm{old}}(\vx_{t-\Delta t},s,t{-}\Delta t)\big]/(2\Delta t)$
\State $\vV_\theta \gets \vu_\theta(\vx_t,s,t) + (t-s)\,\bm{d}$
\State $\vV^{\mathrm{old}} \gets \vu^{\mathrm{old}}(\vx_t,s,t) + (t-s)\,\bm{d}$
\State $\vV_\theta^{+}\gets(1-\beta)\vV^{\mathrm{old}}+\beta\vV_\theta$; \quad
$\vV_\theta^{-}\gets(1+\beta)\vV^{\mathrm{old}}-\beta\vV_\theta$
\State $\mathcal{L}_{\mathrm{MFNFT}}\gets r\,\|\vV_\theta^{+}-\vv_t\|_2^2
+(1-r)\,\|\vV_\theta^{-}-\vv_t\|_2^2$
\State update $\theta$ by $\nabla_\theta\mathcal{L}_{\mathrm{MFNFT}}$
\end{algorithmic}
\end{algorithm}

\noindent\textbf{Finite-difference derivative.} The total-derivative term $\tfrac{\mathrm{d}}{\mathrm{d}t}\vu_\theta=\partial_t\vu_\theta+(\partial_{\vx}\vu_\theta)\,\widehat{\vv}_\theta$ in \cref{eq:V_theta} can be seen as a Jacobian-vector product (JVP)\footnote{The Jacobian $[\partial_{\vx}\vu_\theta,\,\partial_{s}\vu_\theta,\,\partial_{t}\vu_\theta]$ stacks the first-order partial derivatives of $\vu_\theta$, and its product with the direction $[\widehat{\vv}_\theta,\,0,\,1]$ recovers the total derivative $\partial_t\vu_\theta+(\partial_{\vx}\vu_\theta)\,\widehat{\vv}_\theta$.}.
Computing this JVP with forward-mode automatic differentiation (\texttt{torch.func.jvp}) is expensive and not fully compatible with Fully Sharded Data Parallel (FSDP)~\citep{fsdp} training.
Following prior works~\citep{Nie_2026_CVPR,gu2026anyflowanystepvideodiffusion}, we instead approximate it by a central finite difference in $t$,
\begin{equation}
\tfrac{\mathrm{d}}{\mathrm{d}t}\vu_\theta(\vx_t,s,t)\;\approx\;
\frac{\vu_\theta(\vx_{t+\Delta t},s,t+\Delta t)-\vu_\theta(\vx_{t-\Delta t},s,t-\Delta t)}{2\,\Delta t},
\label{eq:finite_diff}
\end{equation}
where $\vx_{t\pm\Delta t}=\vx_t\pm\Delta t\,\vv_t$ displace $\vx_t$ along the direction $\vv_t$, instead of $\widehat{\vv}_\theta$ (justified below).

\noindent\textbf{Shared total derivative.} Following DiffusionNFT, we take
$\vu^{\mathrm{old}}$ to be an EMA of $\vu_\theta$, keeping the reference close
to the trainable model during online training. Our objective, however, compares
the induced predictors. If each predictor uses its own total derivative, their
difference is
\begin{equation}
\vV_\theta(\vx_t,s,t)-\vV^{\mathrm{old}}(\vx_t,s,t)
={}\vu_\theta(\vx_t,s,t)-\vu^{\mathrm{old}}(\vx_t,s,t)+(t-s)\left(
\tfrac{\mathrm{d}}{\mathrm{d}t}\vu_\theta
-\tfrac{\mathrm{d}}{\mathrm{d}t}\vu^{\mathrm{old}}
\right).
\end{equation}
The EMA keeps the average-velocity difference
$\vu_\theta-\vu^{\mathrm{old}}$ small, but it does not control the difference
between their derivatives. This additional term destabilizes
$\vV_\theta-\vV^{\mathrm{old}}$ and causes training to collapse
(\cref{fig:shared}). We therefore use the same derivative
$\tfrac{\mathrm{d}}{\mathrm{d}t}\vu^{\mathrm{old}}$ for both predictors, so the
derivative-difference term cancels exactly and leaves $\vV_\theta(\vx_t,s,t)-\vV^{\mathrm{old}}(\vx_t,s,t)=\vu_\theta(\vx_t,s,t)-\vu^{\mathrm{old}}(\vx_t,s,t).$ This choice also saves computation because the finite-difference derivative is
evaluated once for $\vu^{\mathrm{old}}$ rather than separately for both
predictors.

\noindent\textbf{Conditional-velocity direction.} The idealized direction
$\widehat{\vv}_\theta(\vx_t,t)=\vu_\theta(\vx_t,t,t)$ in \cref{eq:V_theta} 
requires an extra network evaluation. We instead reuse the forward-process conditional velocity
$\vv_t=\dot\alpha_t\vx_0+\dot\sigma_t\epsilonv$, which is already available at no extra cost. In practice, it is used to compute the displaced points $\vx_{t\pm\Delta t}$ in \cref{eq:finite_diff}. 
This choice is decisive: tying the direction to the shifting $\widehat{\vv}_\theta$ collapses training,
whereas $\vv_t$ stays stable (\cref{fig:training_curves}).

\section{Experiments}
\label{sec:experiment}

\subsection{Implementation Details}
\label{sec:impl}

\noindent\textbf{Models.}
We validate MeanFlowNFT on both image and video generation. For image generation, we use
Stable Diffusion~3.5-Medium (\sdmodel)~\citep{esser2024scaling} at
$512{\times}512$ for training and $1024{\times}1024$ for evaluation following DiffusionNFT~\citep{diffusionnft}. For video generation, we use \wanmodel~1.3B~\citep{wan2025wanopenadvancedlargescale} at
$480$p with $81$ frames. In each case, the MeanFlow policy is obtained by
distilling the base model with AnyFlow~\citep{gu2026anyflowanystepvideodiffusion}. For video generation, we directly
adopt the publicly released AnyFlow checkpoint. Following
DiffusionNFT~\citep{diffusionnft}, the entire pipeline is \emph{CFG-free}
\citep{ho2021classifierfree}. At inference, the model is deployed with its native
MeanFlow few-step sampler with different sampling steps.

\noindent\textbf{Rewards.}
For image RL, we train with \clipscore~\citep{clipscore}, \pickscore~\citep{pickscore},
and \hpstwo~\citep{hpsv2} on \pickscore prompt set. For video RL, we train
with \nmfmt{HPSv3-general} and \nmfmt{HPSv3-percentile}~\citep{hpsv3} together with the
motion-quality (\nmfmt{MQ}) and text-alignment (\nmfmt{TA}) scores of
\videoalign~\citep{liu2025improvingvideogenerationhuman}, following the reward setup of
\longcat~\citep{meituanlongcatteam2025longcatvideotechnicalreport}. Here, the training prompts are those used by
DanceGRPO~\citep{xue2025dancegrpo}. More details are deferred to the \cref{app:impl}.

\noindent\textbf{Training.}
We finetune with LoRA~\citep{lora} (rank $32$, scaling factor $64$) applied to all linear
layers in the attention blocks. We set the
guidance strength to $\beta{=}0.1$ and update $\vu^{\mathrm{old}}$ with the
same EMA schedule in DiffusionNFT~\citep{diffusionnft}, $\eta_i{=}\min(0.001\,i,\,0.5)$ ($i$ denotes the $i^{\text{th}}$ update step), and weight the Kullback--Leibler (KL)
regularization term by $10^{-4}$.
Optimization uses AdamW~\citep{loshchilov2018decoupled} with a constant learning rate of $3{\times}10^{-6}$. For
each prompt, we adopt the generator with $4$ steps to collect rollouts. For \sdmodel we use a group size of
$24$ with $48$ groups per update on $8$ NVIDIA~H20 GPUs. For \wanmodel we use a group
size of $16$ with $8$ groups per epoch on $32$ NVIDIA~H20 GPUs.

\noindent\textbf{Evaluation.}
For images, we report
\imagereward~\citep{xu2023imagereward}, \clipscore, \aesthetic~\citep{schuhmann2022aesthetics}, \pickscore, \hpstwo, \hpsthree~\citep{hpsv3},
\genevaltwo~\citep{kamath2025geneval2addressingbenchmark}, and
\ocrscore~\citep{flowgrpo}. We evaluate \hpsthree and \ocrscore following
\citet{flowgrpo,xue2025dancegrpo}, adopt the official setting for \genevaltwo, and
compute the remaining metrics on \drawbench~\citep{saharia2022photorealistictexttoimagediffusionmodels}. For video, we report
\vbench~\citep{huang2024vbench} together with the four video training metrics on a
held-out set of $256$ prompts that are excluded from training.

\subsection{Main Results}
\label{sec:main_results}

\begin{table}[t]
\centering
\setlength{\tabcolsep}{4.5pt}
\renewcommand{\arraystretch}{1.2}
\caption{Main results on \sdmodel for image generation with $1024{\times}1024$ resolution. 
Among few-step models, \best{bold} and \second{underline} denote the best and second-best results.}
\label{tab:sd35_main}
\resizebox{\textwidth}{!}{%
\begin{tabular}{l c c c c c c c c}
\toprule
\textbf{Method}
& \nmfmt{ImageReward}$\uparrow$ & \nmfmt{CLIPScore}$\uparrow$ & \nmfmt{Aesthetic}$\uparrow$
& \nmfmt{PickScore}$\uparrow$ & \nmfmt{HPSv2}$\uparrow$ & \nmfmt{HPSv3}$\uparrow$
& \nmfmt{GenEval2}$\uparrow$ & \nmfmt{OCR}$\uparrow$ \\
\midrule
\rowcolor{colGray!16}\multicolumn{9}{c}{\textit{Multi-step models}\defstep{40}} \\
\sdmodel (w/o CFG)        & -0.4770 & 0.2391 & 5.1514 & 20.6587 & 0.2109 & 3.4601 & 0.1171 & 0.1439 \\
\quad + DiffusionNFT      & 1.4066 & 0.2889 & 5.9647 & 23.6440 & 0.3236 & 13.5959 & 0.2613 & 0.9098 \\
\sdmodel (w/ CFG)         & 0.9253 & 0.2880 & 5.3811 & 22.4638 & 0.2822 & 11.2489 & 0.2038 & 0.5449 \\
\quad + Flow-GRPO         & 1.0193 & 0.2912 & 5.2843 & 22.4783 & 0.2657 & 9.8059 & 0.2638 & 0.6282 \\
\midrule
\rowcolor{colGray!16}\multicolumn{9}{c}{\textit{Few-step models}\defstep{4}} \\
DMD               & 0.9241 & 0.2841 & 5.5055 & 22.2807 & 0.2874 & 11.7156 & 0.2042 & 0.3996 \\ %
CDM               & 1.0307 & 0.2819 & 5.5721 & 22.4160 & 0.2976 & 12.5190 & 0.2020 & 0.3225 \\ %
AnyFlow           & 1.1125 & 0.2886 & 5.4203 & 22.4789 & 0.2969 & 12.0691 & 0.1896 & 0.4520 \\ %
\midrule
RTDMD              & 1.2315 & 0.2775 & \best{6.1290} & \second{23.2825} & \second{0.3265} & \best{13.9253} & 0.2042 & 0.2965 \\
$R_\text{dm}$      & 0.7236 & 0.2720 & 5.7537 & 22.0748 & 0.2759 & 11.2115 & 0.1619 & 0.3759 \\
DMD + DiffusionNFT & 0.7158 & 0.2843 & 5.3784 & 21.9571 & 0.2708 & 9.6985 & 0.2249 & 0.4865 \\ %
CDM + DiffusionNFT & 0.1455 & 0.2745 & 4.8335 & 21.3849 & 0.2143 & -3.2834 & 0.2103 & 0.3303 \\ %
AnyFlow + DiffusionNFT & \second{1.2394} & \second{0.2915} & \second{5.9489} & 23.0876 & 0.2919 & 12.1378 & \second{0.2335} & \second{0.5948} \\ %
\rowcolor{colSky!22}\textbf{MeanFlowNFT (Ours)} & \best{1.4504} & \best{0.2967} & 5.9275 & \best{23.5019} & \best{0.3269} & \second{13.8826} & \best{0.2375} & \best{0.6534} \\ %
\bottomrule
\end{tabular}%
}
\vspace{-0.1in}
\end{table}

\noindent\textbf{Image generation.}
\cref{tab:sd35_main} reports the quantitative comparison on \sdmodel. Among
all few-step models, MeanFlowNFT attains the best results on $6$ of the
$8$ metrics. It clearly outperforms the few-step distillation baselines DMD~\citep{dmd},
CDM~\citep{cdm}, and AnyFlow~\citep{gu2026anyflowanystepvideodiffusion}, and also beats the
recent few-step RL methods $R_{\mathrm{dm}}$~\citep{fan2026rtextdmreconceptualizingdistributionmatching} and
RTDMD~\citep{huang2026reinforcingfewstepgeneratorsrewardtilted} on most
metrics (\emph{e.g.}, OCR $0.65$ \emph{vs.}\ $0.30$ against RTDMD). Remarkably, with only $4$ sampling steps MeanFlowNFT already
matches or exceeds the $40$-step forward-process RL method
DiffusionNFT~\citep{diffusionnft} on several reward metrics (ImageReward
$1.45$ \emph{vs.}\ $1.41$, CLIPScore $0.297$ \emph{vs.}\ $0.289$, \emph{etc.}) at $10\times$ fewer
function evaluations. Qualitatively (\cref{fig:img_008,fig:img_014,fig:img_010}), MeanFlowNFT
produces more faithful samples with fewer reward-hacking artifacts, such as
over-saturated colors and implausible object scales, than DiffusionNFT and RTDMD.

\begin{figure}[!ht]
    \centering
    \includegraphics[width=\linewidth]{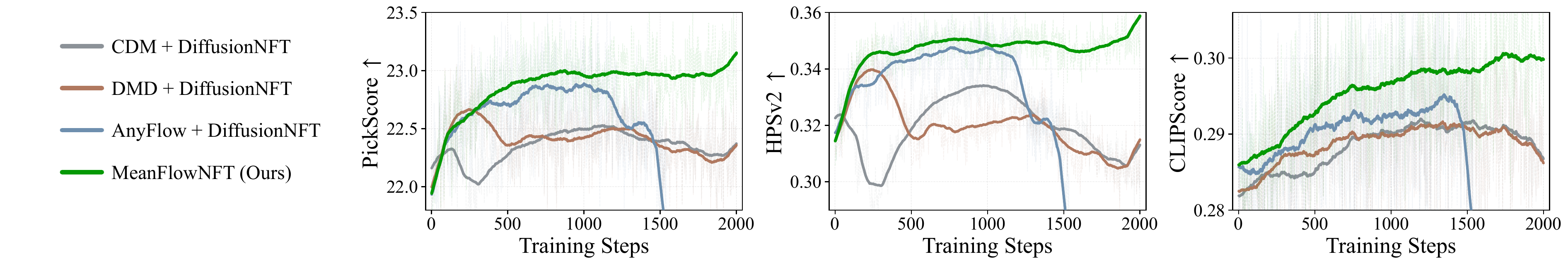}
    \vspace{-0.2in}
    \caption{Training reward curves of MeanFlowNFT compared with the baselines on \sdmodel.}
    \label{fig:baseline_curves}
    \vspace{-0.1in}
\end{figure}
Additionally, we apply DiffusionNFT directly to few-step generators. AnyFlow trains
an average-velocity network, while DMD and CDM are trained by distribution
matching~\citep{dmd} rather than flow matching. In each case the network does not
predict the instantaneous velocity as a posterior mean, so the linearity behind
DiffusionNFT's improvement guarantee is absent and no policy improvement can be
derived (\cref{sec:problem}). As a result, AnyFlow$+$DiffusionNFT\footnote{AnyFlow$+$DiffusionNFT plugs the average velocities $\vu^{\mathrm{old}}$ and $\vu_\theta$ into \cref{eq:diffnft_loss} in place of the instantaneous velocities $\vv^{\mathrm{old}}$ and $\vv_\theta$, and optimizes it exactly as DiffusionNFT does.} and
DMD/CDM$+$DiffusionNFT fall far short of MeanFlowNFT and train very unstably,
collapsing early (\cref{fig:baseline_curves}). For instance, CDM$+$DiffusionNFT
diverges within $400$ steps.

\begin{table}[!ht]
\vspace{-0.1in}
\centering
\setlength{\tabcolsep}{4.5pt}
\renewcommand{\arraystretch}{1.2}
\caption{Main results on \wanmodel 1.3B for video generation. We report
\nmfmt{VBench} scores and four metrics on $256$ held-out prompts
(\nmfmt{HPSv3-G}/\nmfmt{HPSv3-P}: HPSv3 general/percentile;
\nmfmt{MQ}/\nmfmt{TA}: VideoAlign motion-quality/text-alignment).
Among few-step models, \best{bold} and \second{underline} denote the best and second-best results.}
\vspace{0.02in}
\label{tab:wan_video}
{\scriptsize%
\begin{tabular}{l c c c c c c c}
\toprule
\multirow{2}{*}{\textbf{Method}} & \multicolumn{3}{c}{\nmfmt{VBench}} & \multicolumn{4}{c}{$256$ held-out prompts} \\
\cmidrule(lr){2-4}\cmidrule(lr){5-8}
 & \nmfmt{Total}$\uparrow$ & \nmfmt{Quality}$\uparrow$ & \nmfmt{Semantic}$\uparrow$
 & \nmfmt{HPSv3-G}$\uparrow$ & \nmfmt{HPSv3-P}$\uparrow$ & \nmfmt{MQ}$\uparrow$ & \nmfmt{TA}$\uparrow$ \\
\midrule
\rowcolor{colGray!16}\multicolumn{8}{c}{\textit{Multi-step models}\defstep{50}} \\
\wanmodel 1.3B (w/ CFG) & 82.94 & 84.69 & 75.97 & 3.9099 & 8.2868 & 0.8684 & 1.2255 \\
\quad + \longcat RL & 82.57 & 84.44 & 75.10 & 4.7099 & 9.2730 & 0.5493 & 1.6321 \\
\midrule
\rowcolor{colGray!16}\multicolumn{8}{c}{\textit{Few-step models}\defstep{4}} \\
rCM     & 82.43 & 84.58 & 73.82 & 3.9660 & 8.7198 & 0.2740 & 1.6290 \\
DMD     & 82.64 & 84.65 & 74.57 & 3.8598 & 8.7955 & 0.1810 & 1.6845 \\
SC-DMD  & 83.36 & 84.76 & \best{77.77} & -- & -- & -- & -- \\
AnyFlow & \second{83.71} & \second{85.36} & 77.11 & \second{6.0536} & \second{10.450} & \second{0.7504} & \best{1.7356} \\
\midrule
\rowcolor{colSky!22}\textbf{MeanFlowNFT (Ours)} & \best{84.33} & \best{85.99} & \second{77.68} & \best{6.5959} & \best{10.793} & \best{0.9535} & \second{1.7235}  \\
\bottomrule
\end{tabular}%
}
\end{table}

\noindent\textbf{Video generation.}
\cref{tab:wan_video} further evaluates MeanFlowNFT on video generation with
\wanmodel~1.3B. Since there is currently no open-source few-step video-generation
RL baseline, we compare with few-step distillation methods
(rCM~\citep{rcm}, DMD~\citep{dmd}, SC-DMD~\citep{scdmd}\footnote{We use the numbers reported in the SC-DMD paper.}, AnyFlow~\citep{gu2026anyflowanystepvideodiffusion})
and the $50$-step flow-matching RL baseline
\longcat RL~\citep{meituanlongcatteam2025longcatvideotechnicalreport}\footnote{\url{https://huggingface.co/lightx2v/Wan2.1-T2V-1.3B-longcat-step1500}}. With only
$4$ sampling steps, MeanFlowNFT obtains the best few-step results on $5$ of the
$7$ reported metrics and improves over AnyFlow on VBench total, HPSv3, and
motion quality. It also outperforms \longcat RL on all reported metrics with far
fewer function evaluations, confirming its applicability to video MeanFlow models.

\begin{figure}[!ht]
\vspace{-0.1in}
\centering
\includegraphics[width=\linewidth]{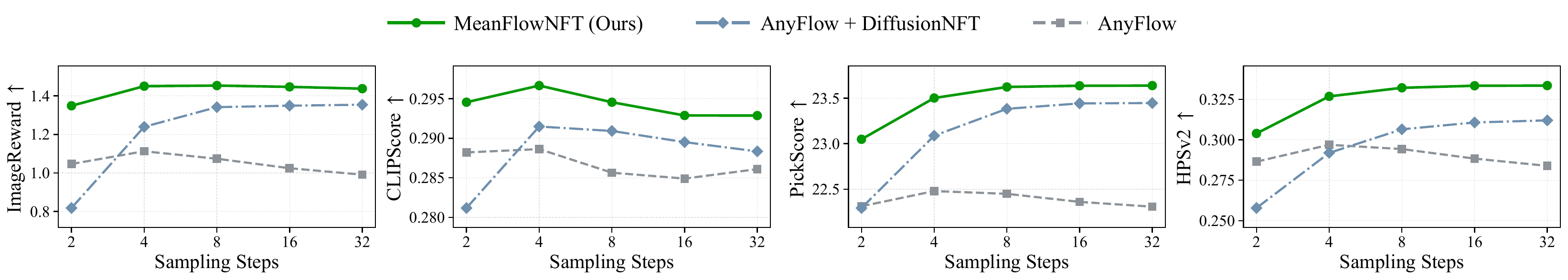}
\vspace{-0.2in}
\caption{Quantitative results of test-time scaling for MeanFlow-based methods on \sdmodel. More results for \wanmodel 1.3B and \sdmodel can be found in \cref{app:scaling}.}
\label{fig:scaling_main}
\vspace{-0.1in}
\end{figure}

\noindent\textbf{Test-time scaling.}
Since $\vu_\theta$ approximates the average velocity over any interval
$[s,t]$, MeanFlowNFT supports any-step inference via the sampler of
\cref{eq:meanflow_sampling}, sweeping the number of steps $N\in\{2,4,8,16,32\}$.
MeanFlowNFT exhibits the same test-time scaling trend as
AnyFlow~\citep{gu2026anyflowanystepvideodiffusion}: most metrics improve as
$N$ increases (\cref{fig:scaling_main}). More notably, MeanFlowNFT is
noticeably more step-consistent than AnyFlow, with samples that vary far less
across $N$ in both layout and content (\cref{fig:steps_image,fig:steps_video}).
As a more accurate average-velocity estimation yields more $N$-invariant sampling, this consistency
indicates that our RL improves AnyFlow.

\newcommand{\stepimg}[2]{{\color{black!35}\setlength{\fboxsep}{0pt}\fbox{\includegraphics[width=\dimexpr#1-2\fboxrule\relax]{#2}}}}
\begin{figure}[!ht]
\vspace{-0.1in}
\centering
\setlength{\tabcolsep}{1pt}\renewcommand{\arraystretch}{0.65}%
\newcommand{\sdw}{0.15\textwidth}%
\newcommand{\pcap}[1]{\multicolumn{6}{c}{\fontsize{6}{6.4}\selectfont\itshape ``#1''}}%
\begin{tabular}{@{}ccc@{\hspace{4pt}}ccc@{}}
\pcap{A pyramid made of falafel with a partial solar eclipse in the background.}\\[-1pt]
{\scriptsize 4 steps} & {\scriptsize 16 steps} & {\scriptsize 32 steps} & {\scriptsize 4 steps} & {\scriptsize 16 steps} & {\scriptsize 32 steps}\\[1.5pt]
\stepimg{\sdw}{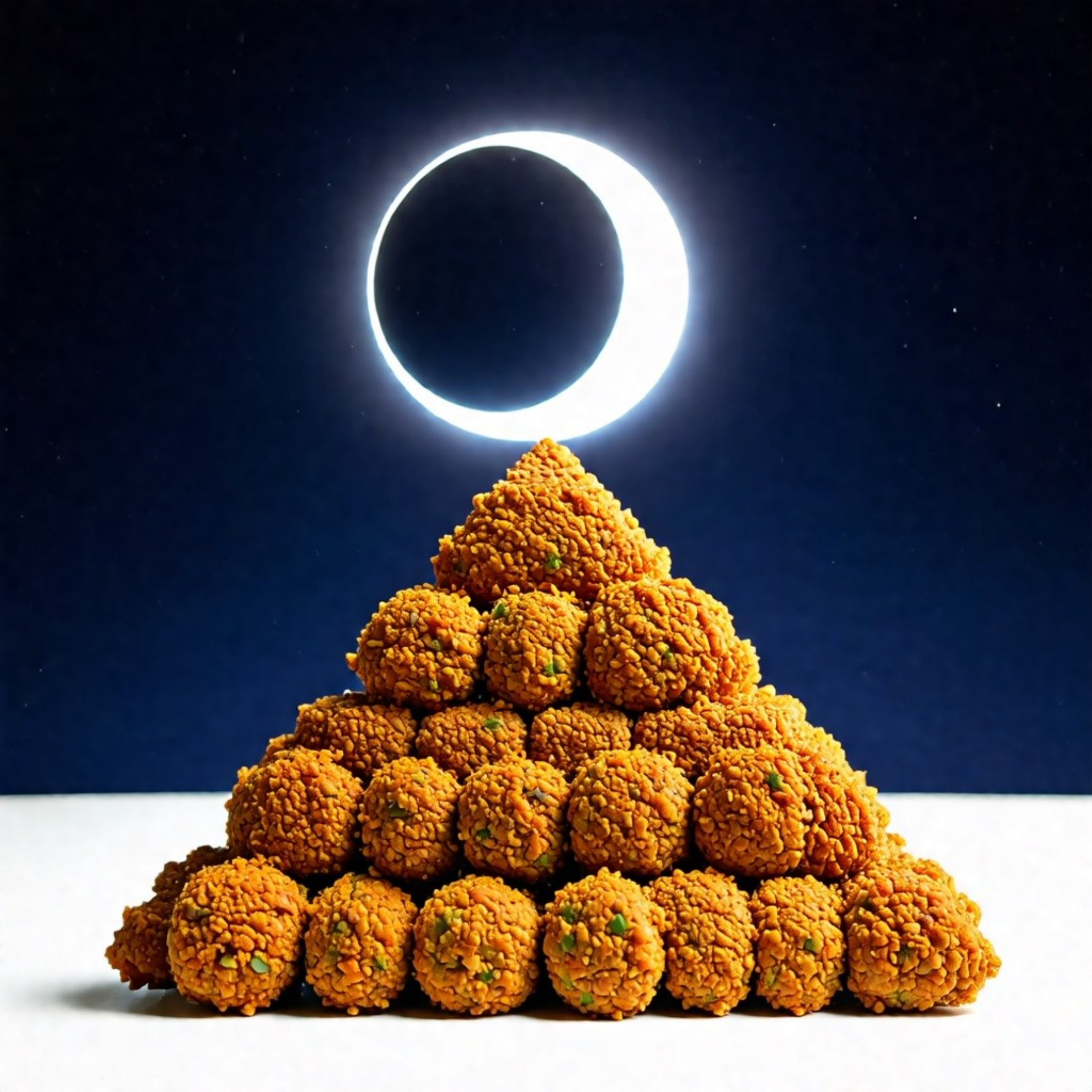} & \stepimg{\sdw}{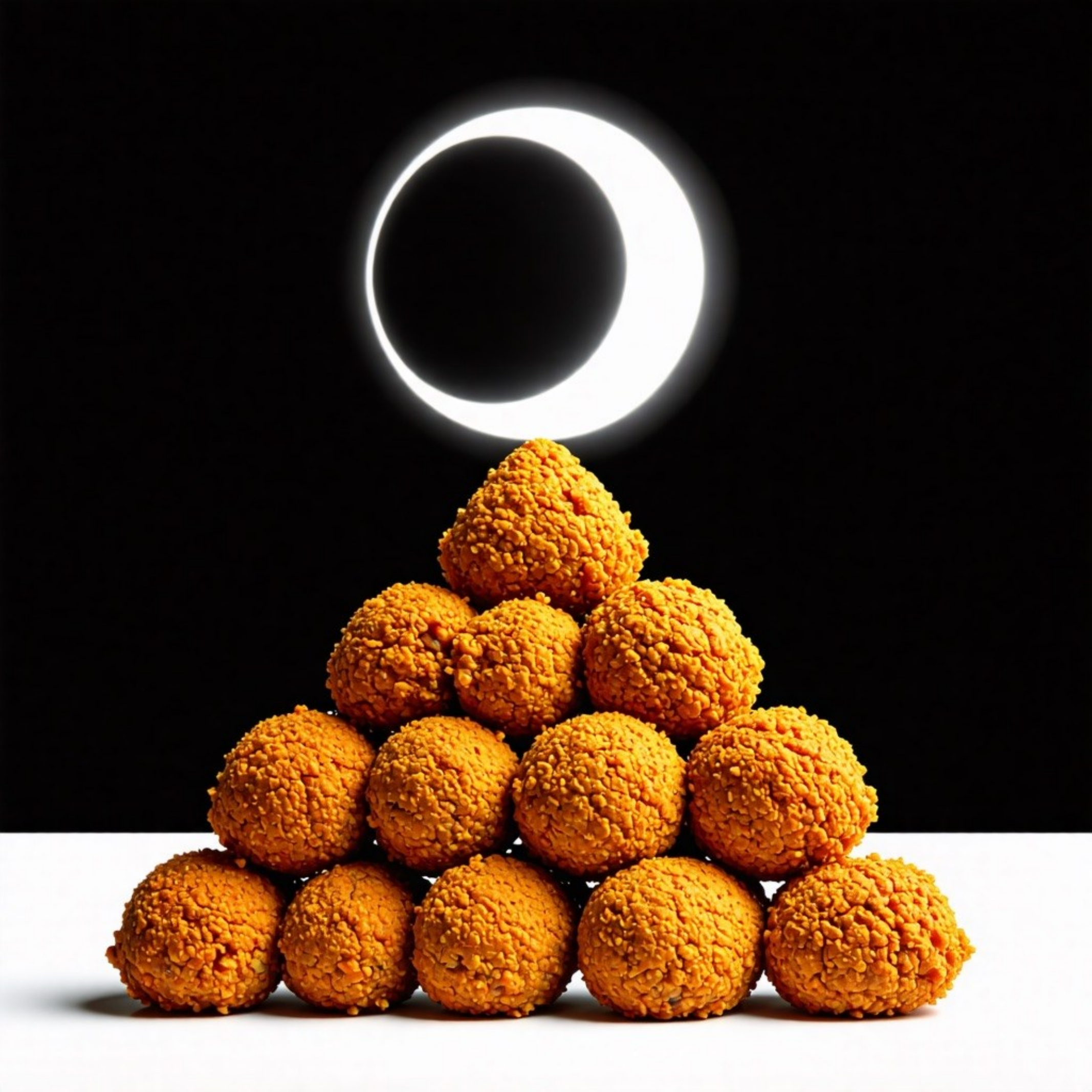} & \stepimg{\sdw}{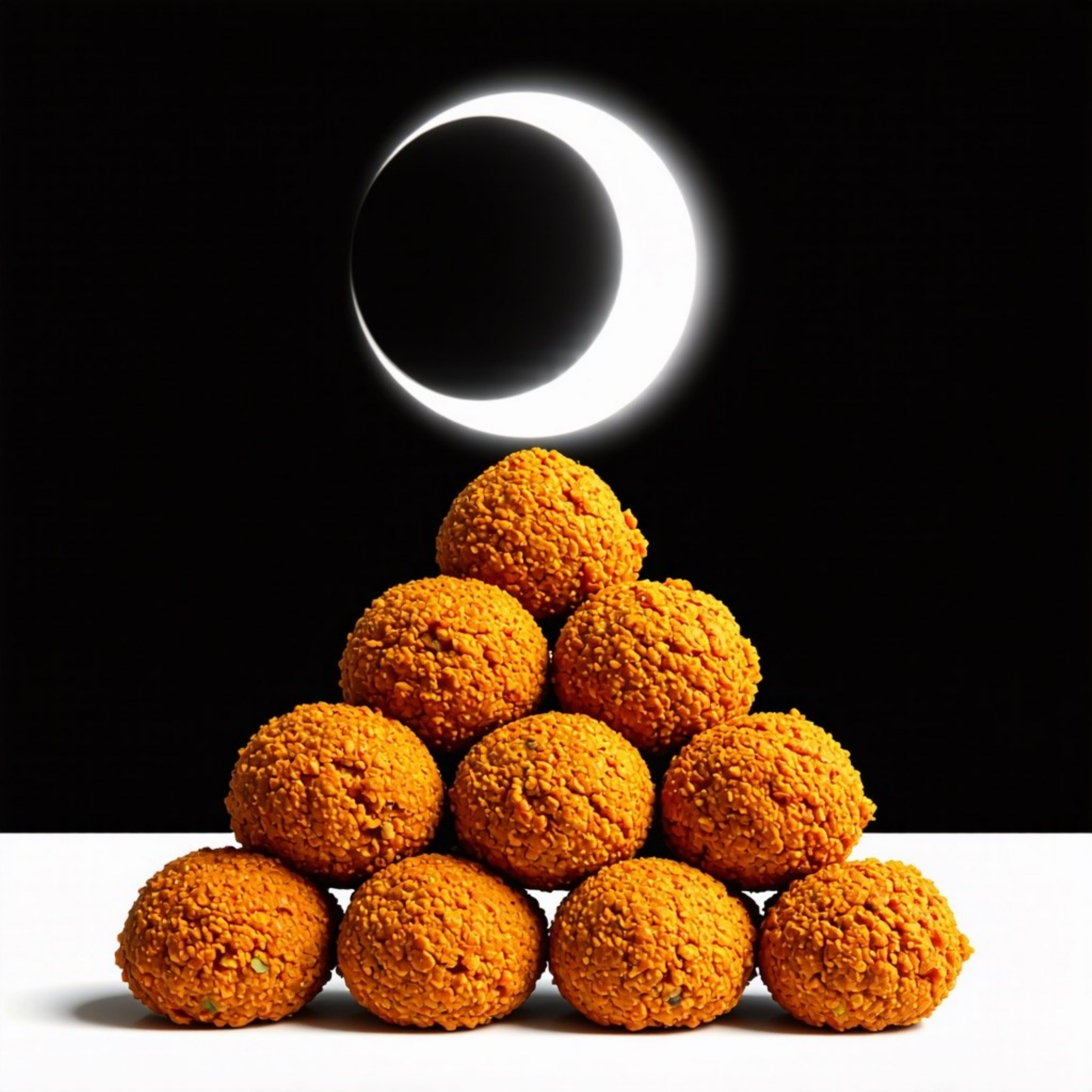} & \stepimg{\sdw}{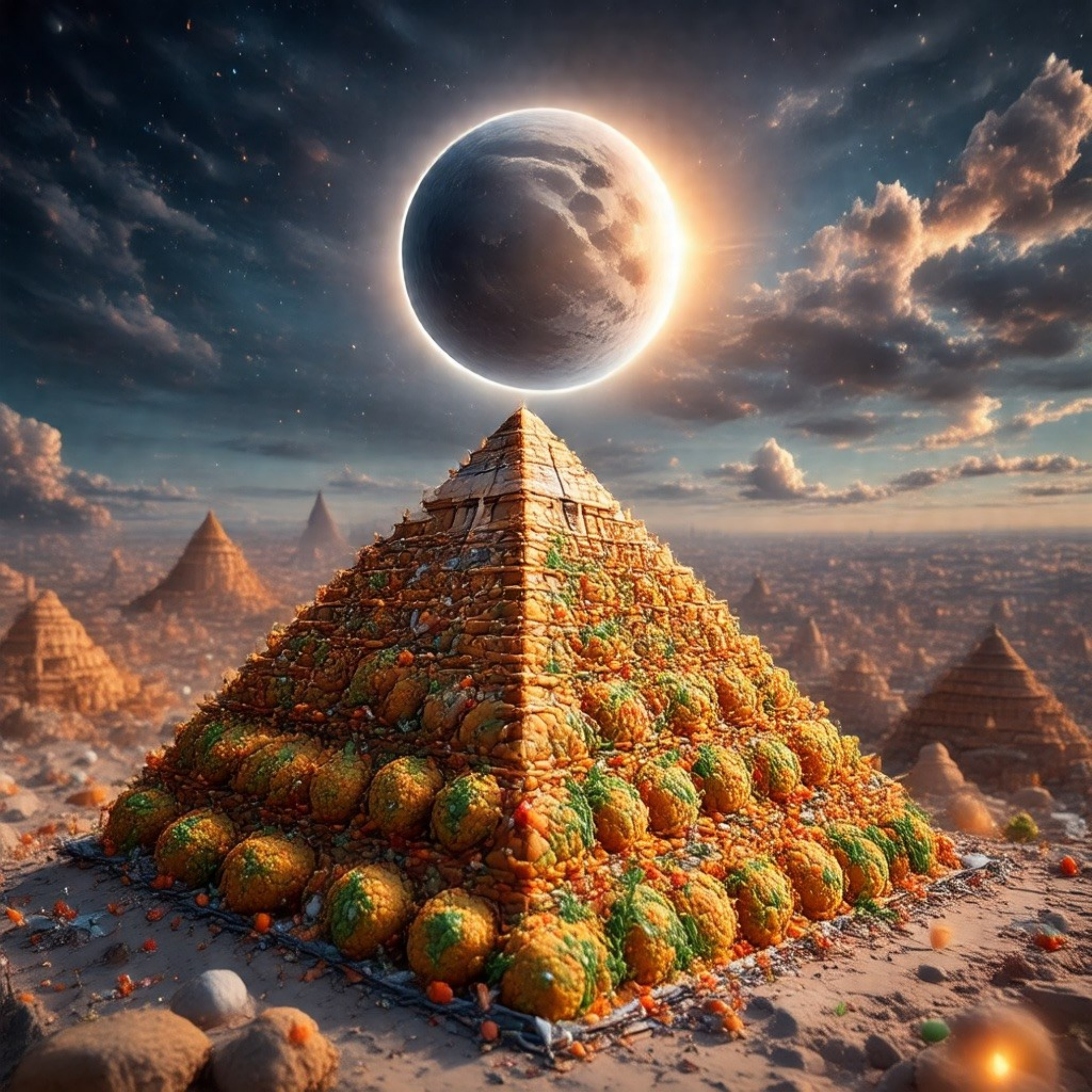} & \stepimg{\sdw}{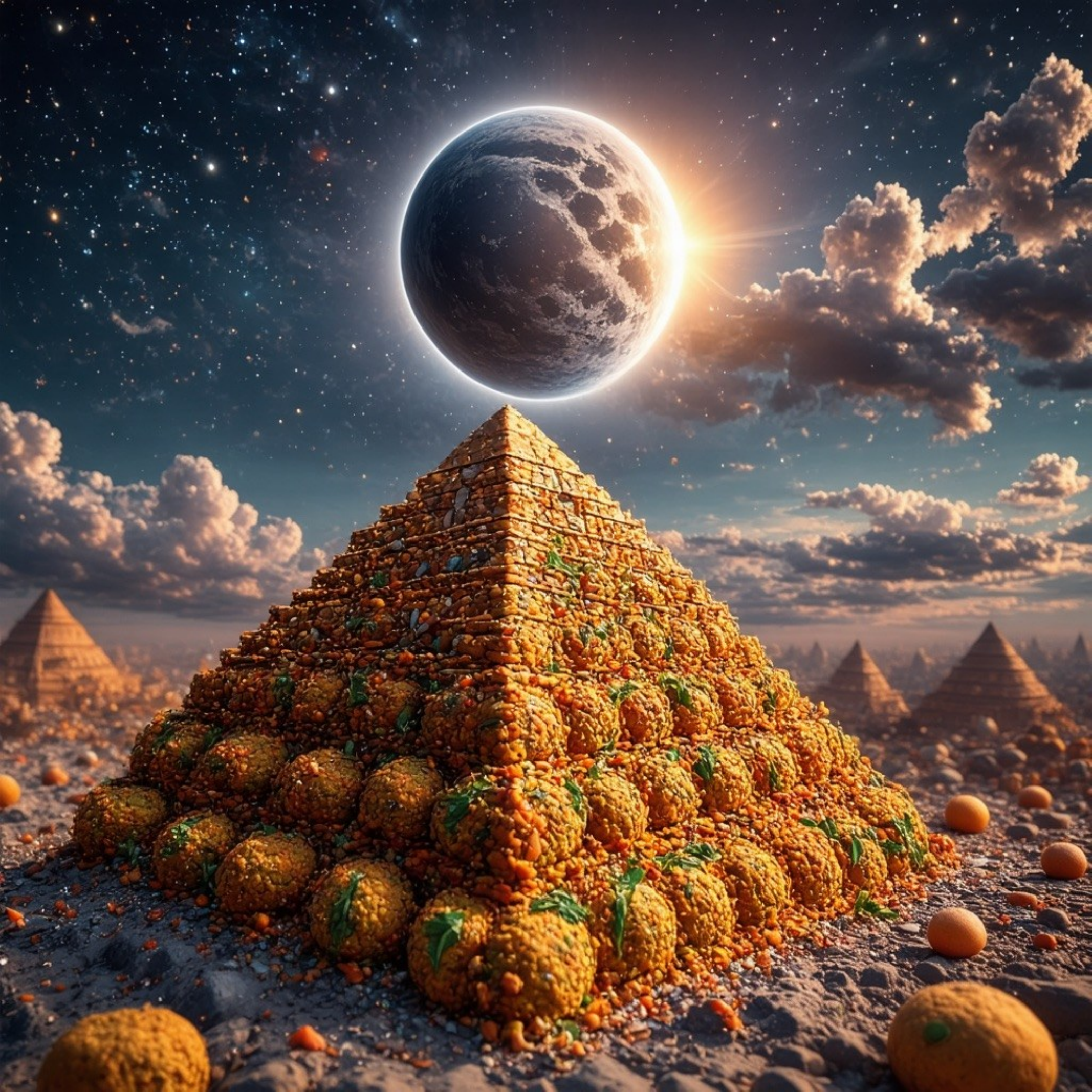} & \stepimg{\sdw}{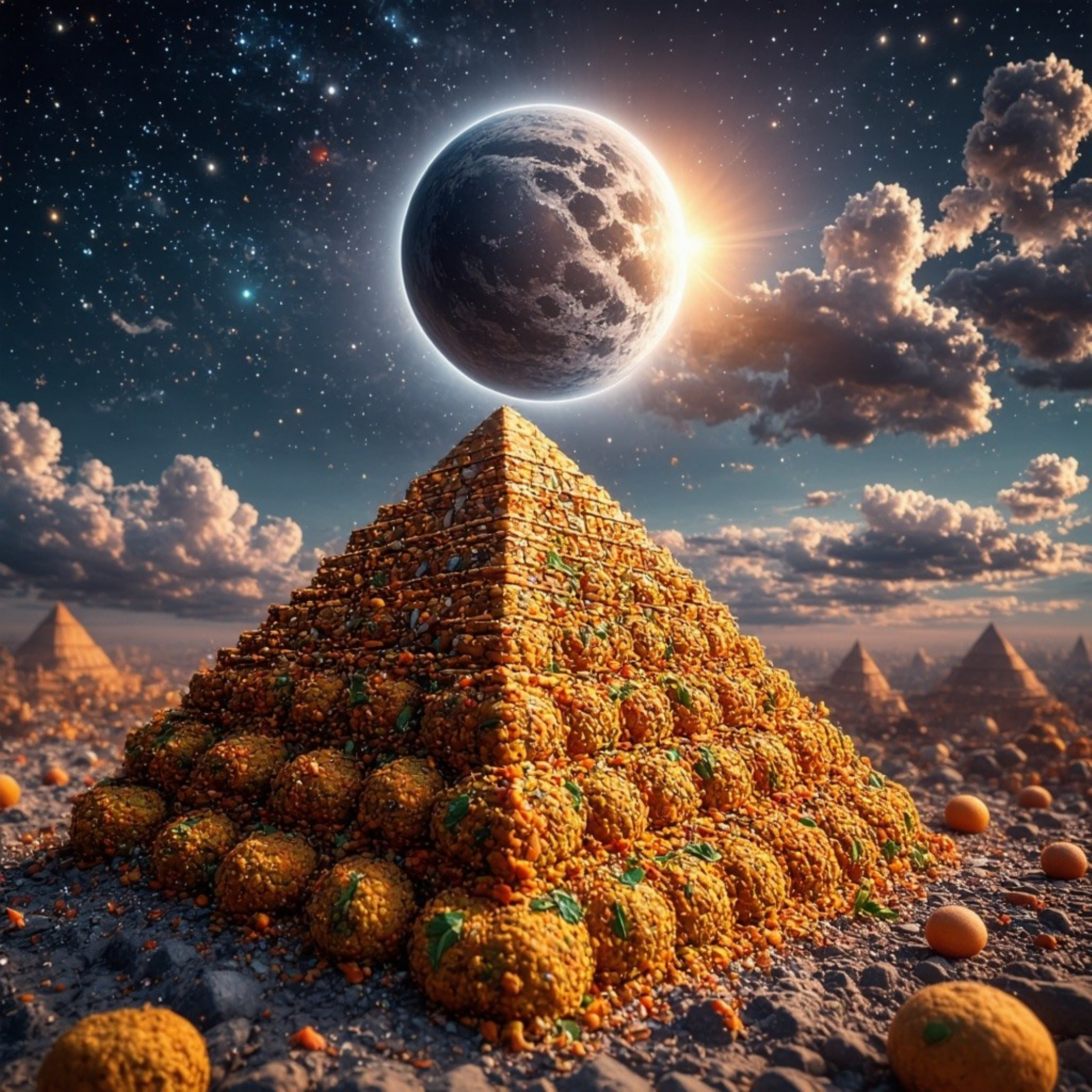}\\[3pt]
\pcap{Lego Arnold Schwarzenegger.}\\[-1pt]
{\scriptsize 4 steps} & {\scriptsize 16 steps} & {\scriptsize 32 steps} & {\scriptsize 4 steps} & {\scriptsize 16 steps} & {\scriptsize 32 steps}\\[1.5pt]
\stepimg{\sdw}{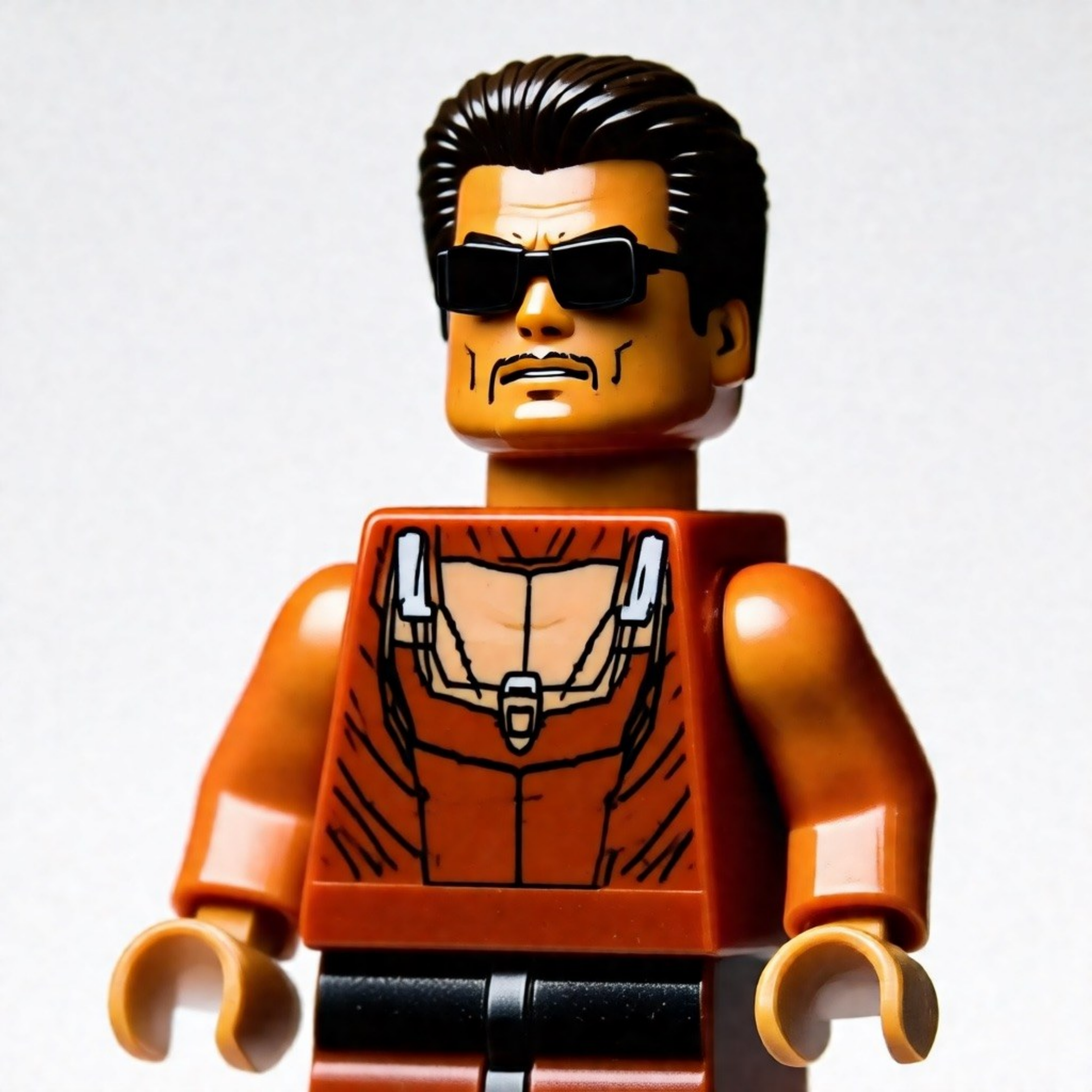} & \stepimg{\sdw}{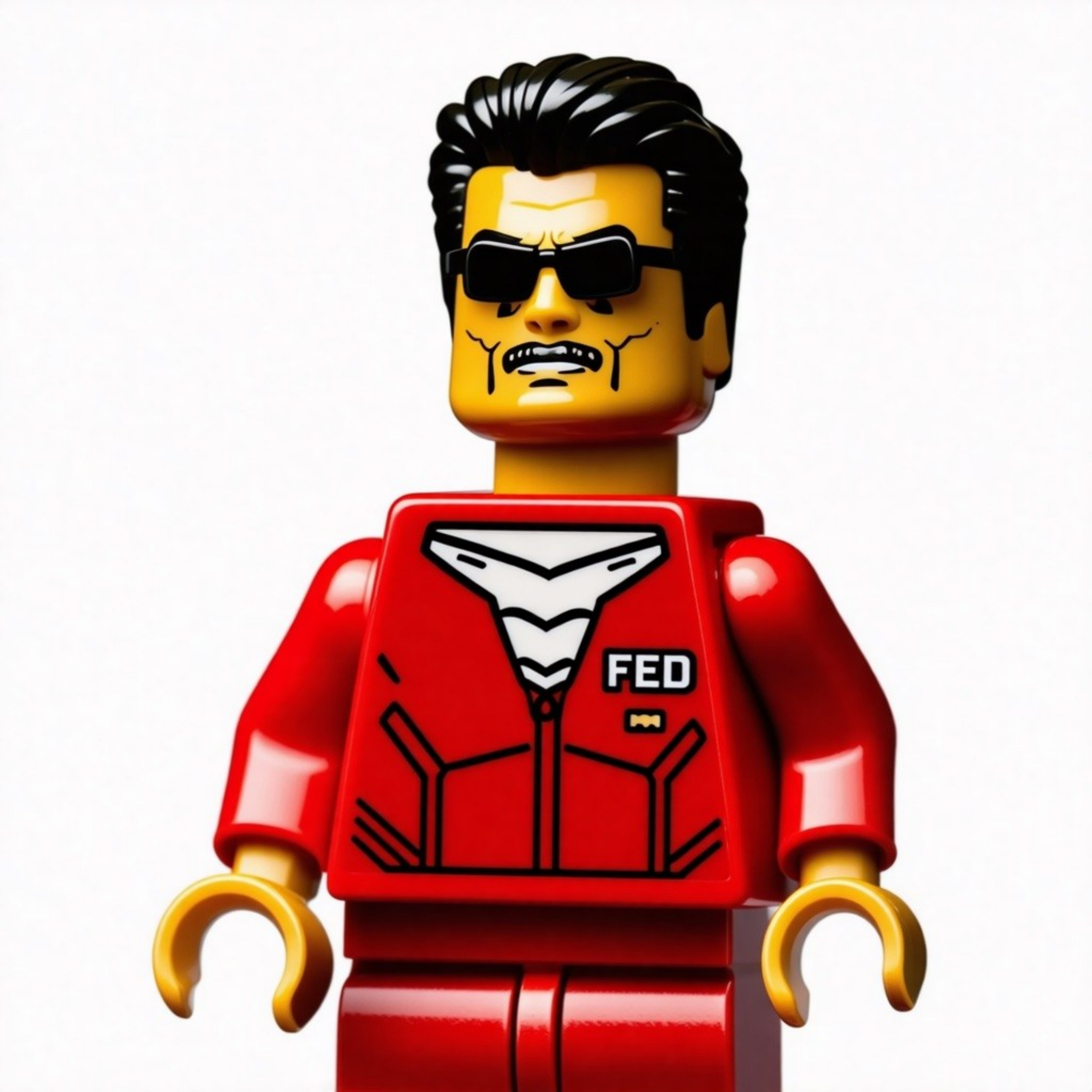} & \stepimg{\sdw}{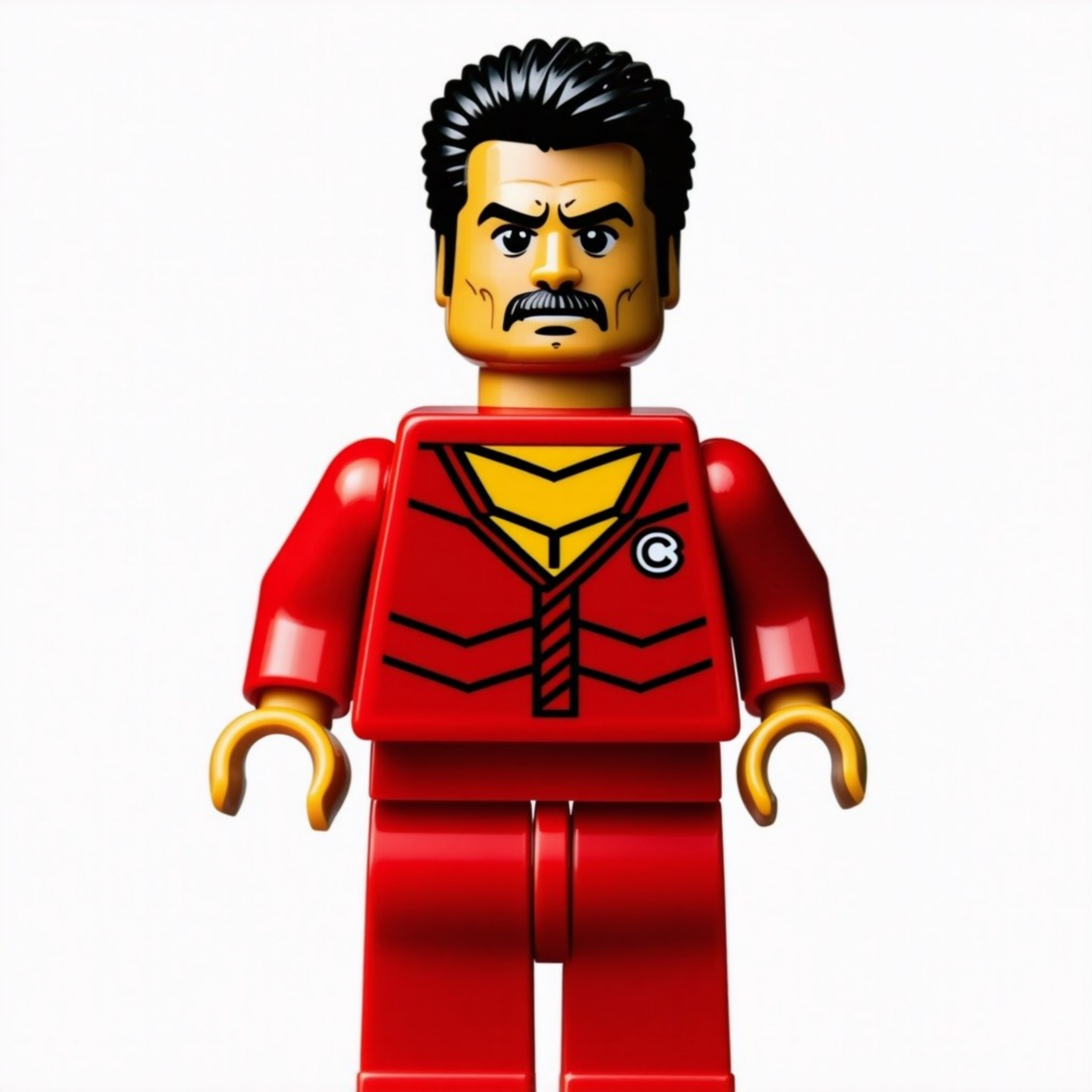} & \stepimg{\sdw}{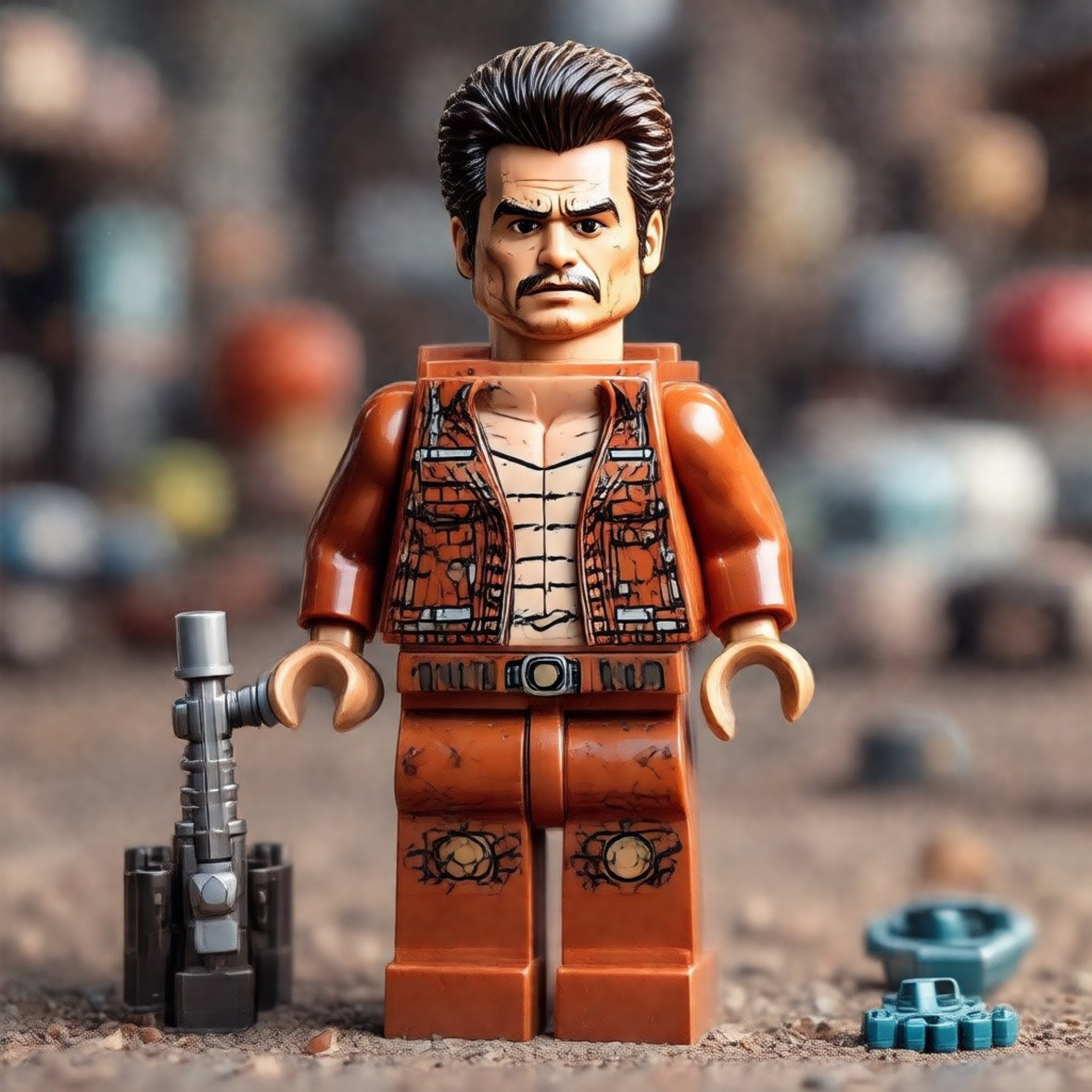} & \stepimg{\sdw}{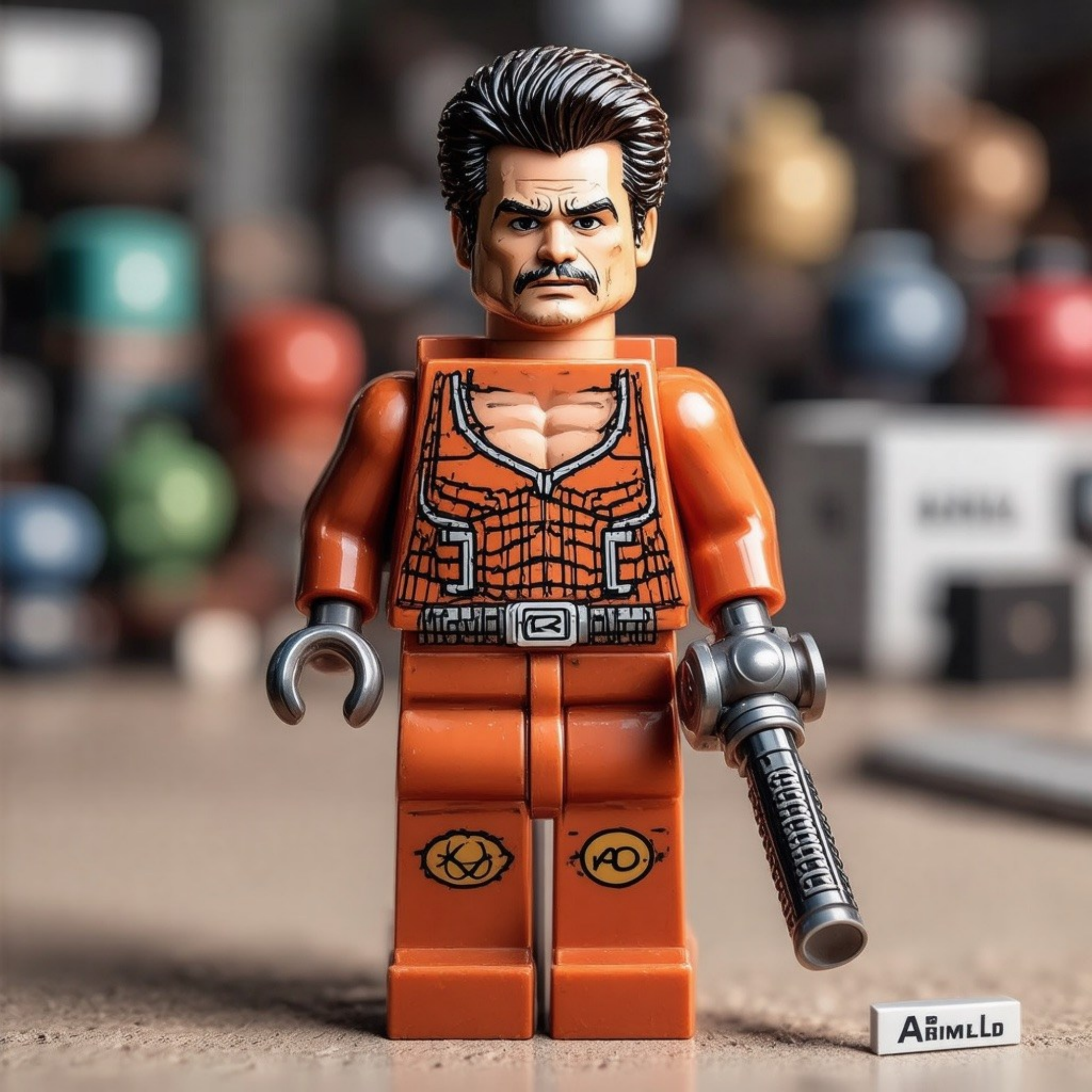} & \stepimg{\sdw}{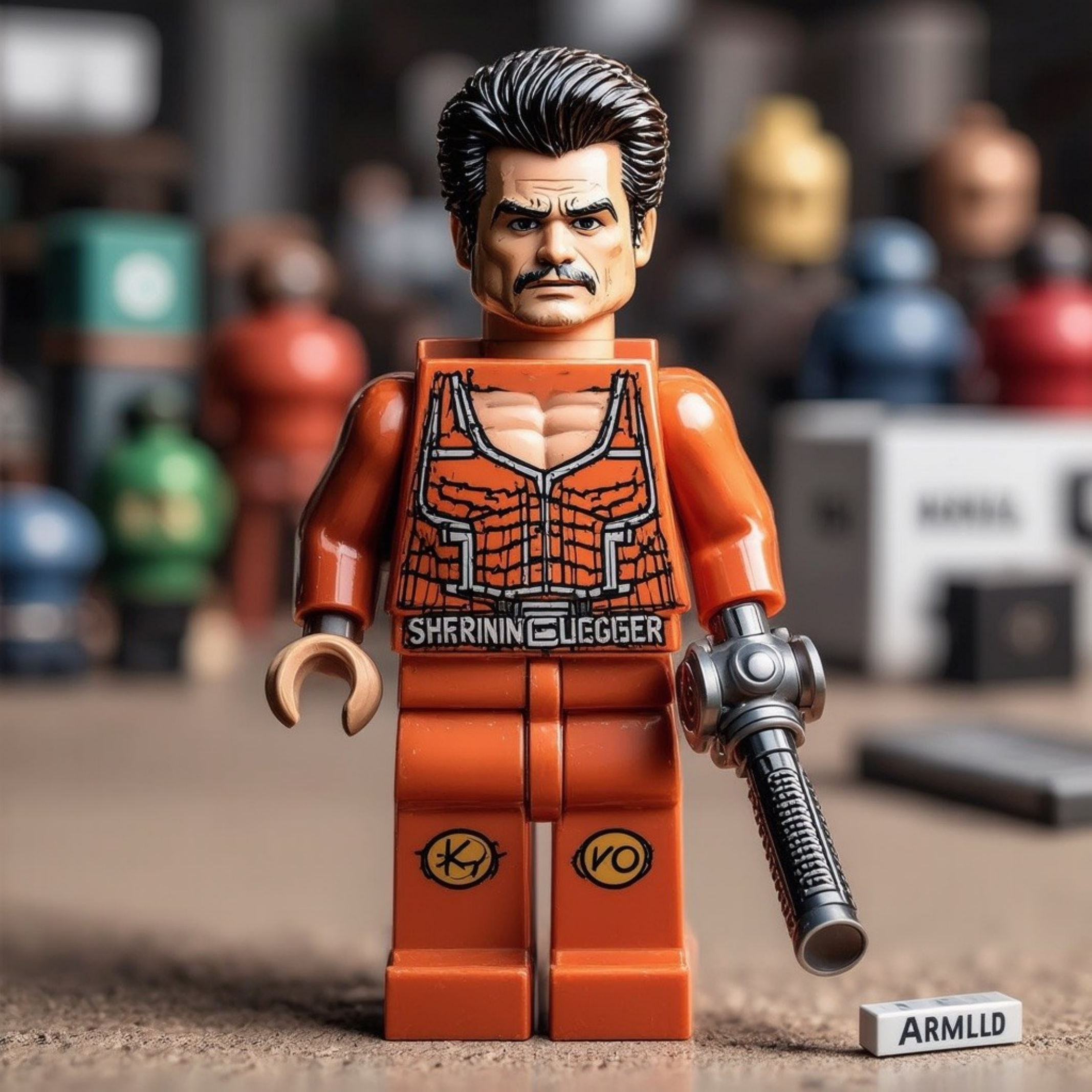}\\[1pt]
\multicolumn{3}{c}{\scriptsize (a) AnyFlow} & \multicolumn{3}{c}{\scriptsize (b) \textbf{MeanFlowNFT}}\\
\end{tabular}
\vspace{-0.1in}
\caption{Qualitative results of test-time scaling on \sdmodel.}
\vspace{-0.1in}
\label{fig:steps_image}
\end{figure}
\begin{figure}[!ht]
\centering
\captionsetup{skip=2pt}%
{\fontsize{6}{6.6}\selectfont\itshape ``A Mongol warrior in bright red-and-gold armor gallops on horseback, drawing a bow and firing an arrow across the open plain.''\par}\vspace{1pt}
\setlength{\tabcolsep}{0pt}\renewcommand{\arraystretch}{1.0}%
\newcommand{\wanw}{0.45\textwidth}%
\begin{tabular}{@{}>{\centering\arraybackslash}m{0.028\textwidth}@{\hspace{2pt}}>{\centering\arraybackslash}m{\wanw}@{\hspace{4pt}}>{\centering\arraybackslash}m{\wanw}@{}}
\rotatebox[origin=c]{90}{\scriptsize 4 steps} & \stepimg{\wanw}{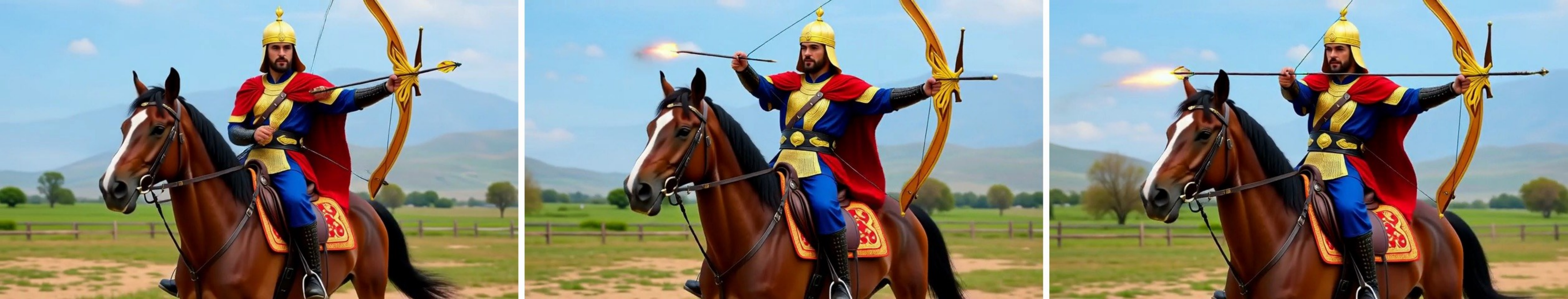} & \stepimg{\wanw}{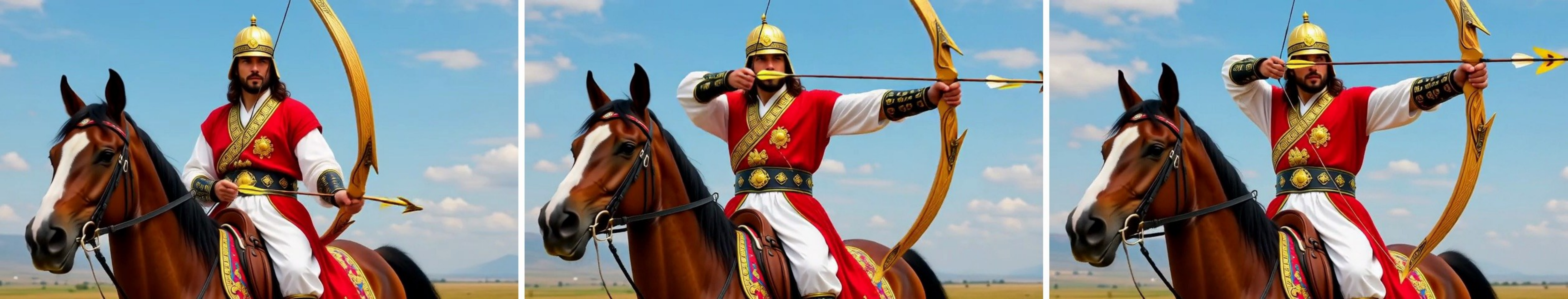}\\[1pt]
\rotatebox[origin=c]{90}{\scriptsize 16 steps} & \stepimg{\wanw}{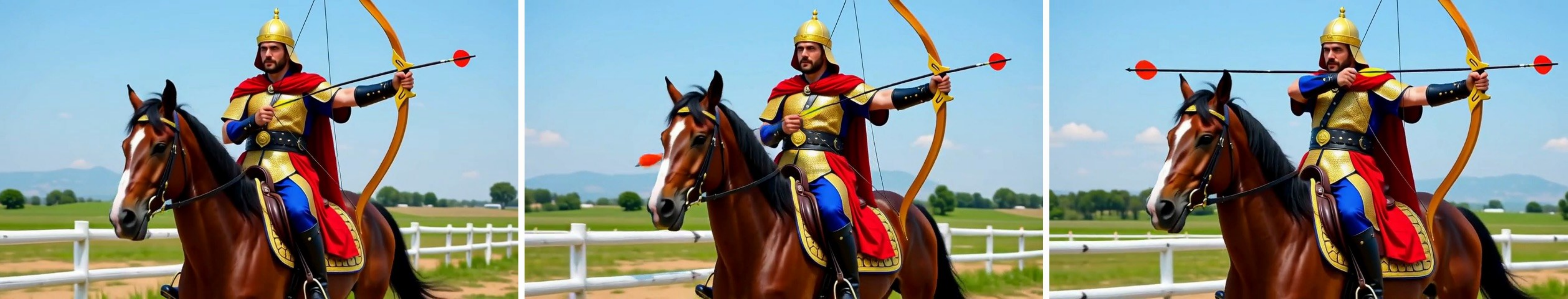} & \stepimg{\wanw}{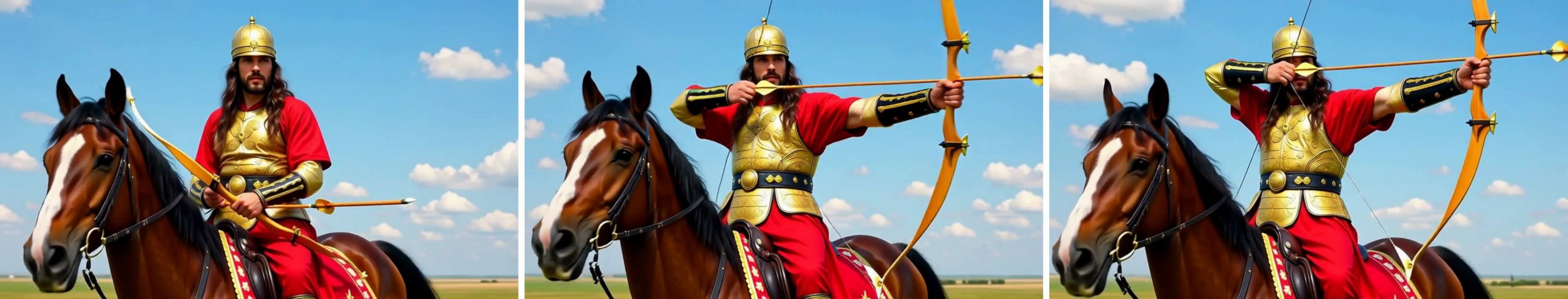}\\[1pt]
\rotatebox[origin=c]{90}{\scriptsize 32 steps} & \stepimg{\wanw}{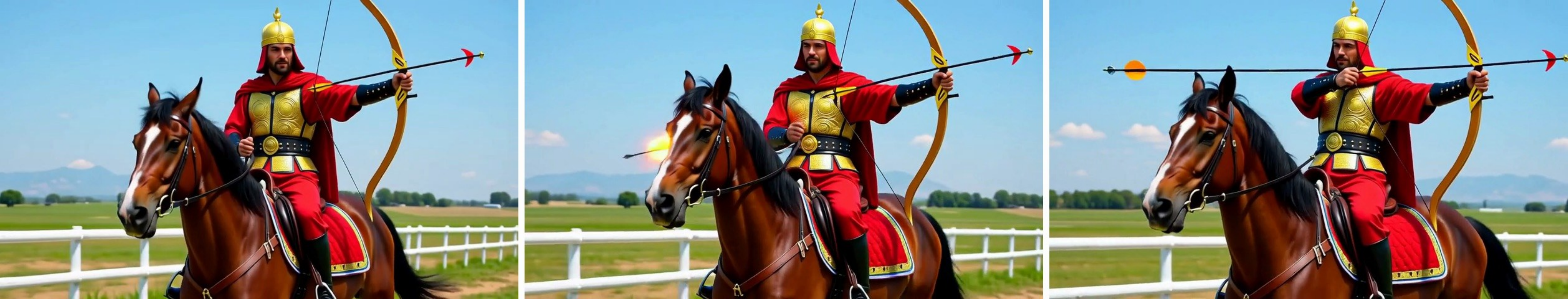} & \stepimg{\wanw}{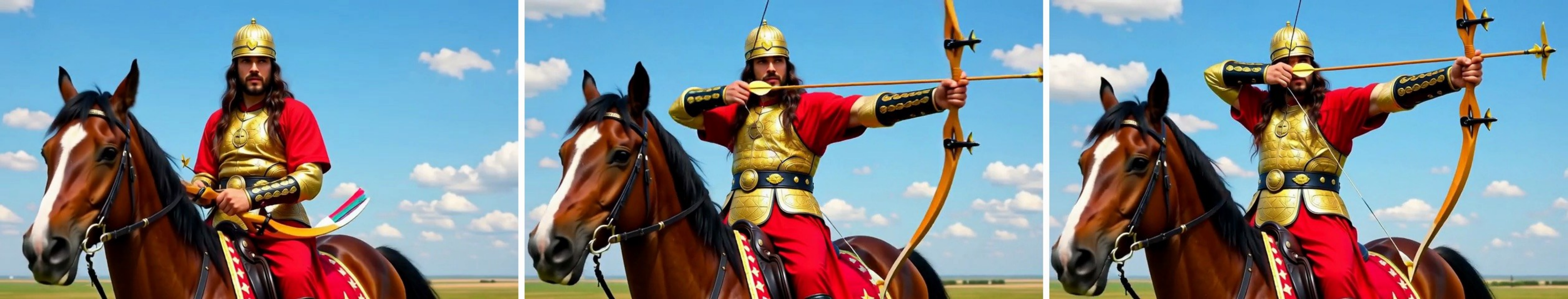}\\[1pt]
 & {\scriptsize (a) AnyFlow} & {\scriptsize (b) \textbf{MeanFlowNFT}}\\
\end{tabular}
\caption{Qualitative results of test-time scaling on \wanmodel~1.3B.}
\label{fig:steps_video}
\vspace{-0.1in}
\end{figure}

\subsection{Analysis}
\label{sec:analysis}
We further analyze several design choices in MeanFlowNFT on \sdmodel.

\noindent\textbf{Effect of shared $\tfrac{\mathrm{d}}{\mathrm{d}t}\vu^{\mathrm{old}}$.}
We examine whether the trainable and reference predictors should share a single
derivative term $\tfrac{\mathrm{d}}{\mathrm{d}t}\vu^{\mathrm{old}}$ or form one each, and find sharing
essential for stability. With a shared derivative, reward rises smoothly and
monotonically across \pickscore, \hpstwo, and \clipscore, whereas forming the two
derivatives independently lets reward rise briefly before collapsing
(\cref{fig:shared_reward}). The induced deviation $\|\vV_\theta-\vV^{\mathrm{old}}\|_2^2$
exposes the mechanism, which grows by orders of magnitude without sharing
(\cref{fig:old_deviate}). The gap between the two predictors is then governed by the
uncontrolled derivative-difference term rather than the EMA-bounded
$\vu_\theta-\vu^{\mathrm{old}}$ signal (\cref{sec:algorithm}). This reference drift
is what destabilizes training, and sharing eliminates it exactly.
\begin{figure}[!ht]
    \vspace{-0.15in}
    \centering
    \begin{subfigure}[b]{0.72\textwidth}
    \centering
    \includegraphics[width=\linewidth]{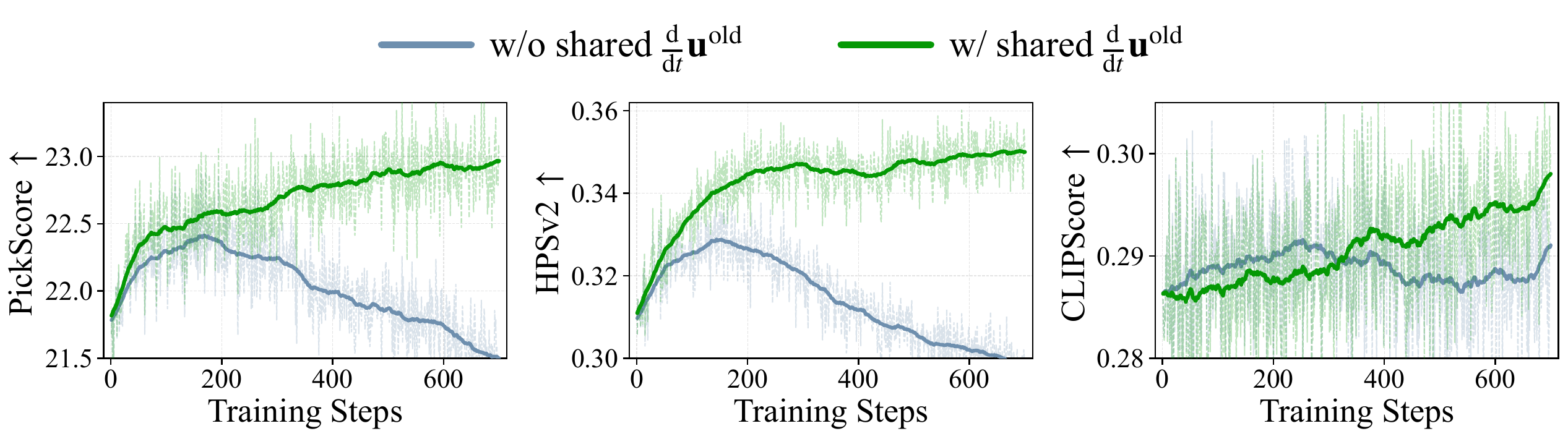}
    \caption{Training reward with and without shared $\tfrac{\mathrm{d}}{\mathrm{d}t}\vu^\mathrm{old}$.}
    \label{fig:shared_reward}
    \end{subfigure}
    \hfill
    \begin{subfigure}[b]{0.26\textwidth}
    \centering
    \includegraphics[width=\linewidth]{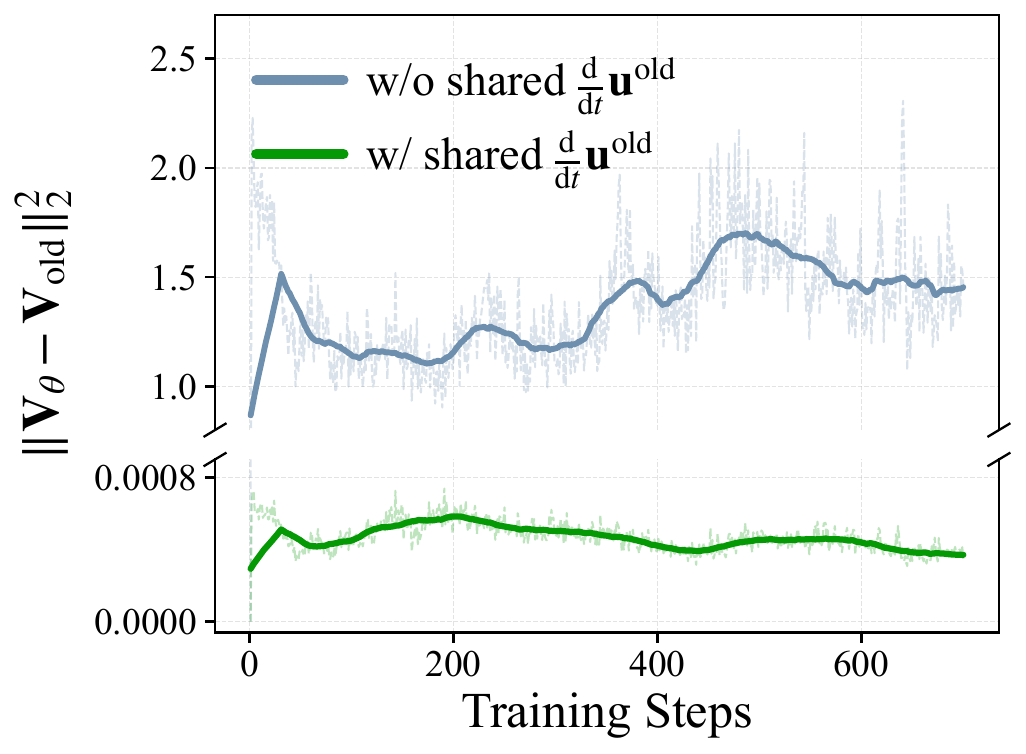}
    \caption{Velocity deviation.}
    \label{fig:old_deviate}
    \end{subfigure}
    \vspace{-0.1in}
    \caption{Effect of sharing the derivative term $\tfrac{\mathrm{d}}{\mathrm{d}t}\vu^\mathrm{old}$ between $\vV^{\mathrm{old}}$ and $\vV_\theta$ on \sdmodel.}
    \label{fig:shared}
    \vspace{-0.15in}
    \end{figure}

\noindent\textbf{Direction for $\vx_{t\pm\Delta t}$.}
We compare the two directions that can drive $\vx_{t\pm\Delta t}$ in the finite-difference derivative (\cref{sec:algorithm}): $\widehat{\vv}_\theta=\vu_\theta(\vx_t,t,t)$ (\cref{eq:V_theta}),
and the model-free conditional velocity $\vv_t=\dot\alpha_t\vx_0+\dot\sigma_t\epsilonv$ (\cref{alg:mfnft}).
Using $\vv_t$ improves reward steadily across all three metrics, whereas anchoring
the direction to the trainable $\widehat{\vv}_\theta$ peaks early and then degrades
(\cref{fig:training_curves}). As $\vv_t$ is model-free and saves an additional network
evaluation, we adopt it by default.
\begin{figure}[!ht]
    \vspace{-0.1in}
    \centering
    \includegraphics[width=0.95\linewidth]{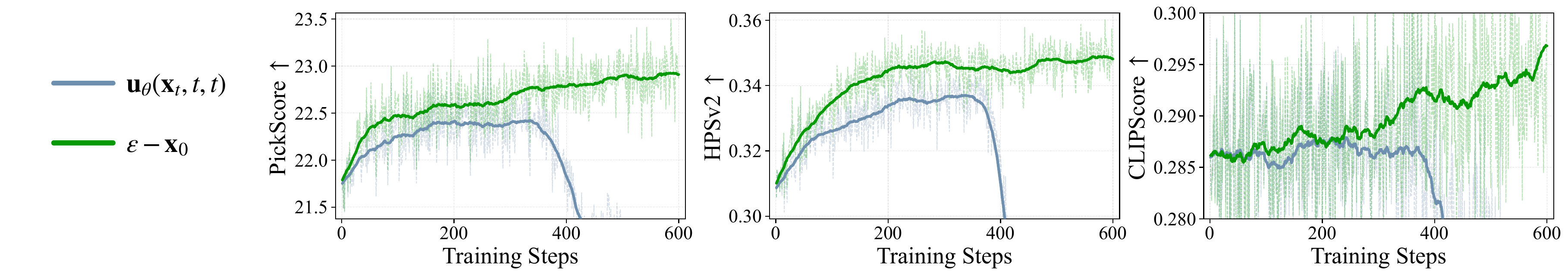}
    \vspace{-0.12in}
    \caption{Training reward curves for MeanFlowNFT with different choices of direction for $\vx_{t\pm\Delta t}$.}
    \label{fig:training_curves}
    \vspace{-0.15in}
    \end{figure}

\noindent\textbf{Training only on $s{=}t$ pairs.}
MeanFlow training mixes two kinds of pairs. Zero-length pairs ($s{=}t$) supervise
the instantaneous velocity $\vu_\theta(\vx_t,t,t)$, as in standard flow matching.
Finite-interval pairs ($s{<}t$) supervise the average velocity behind few-step
sampling. Training only on $s{=}t$ pairs keeps just the instantaneous velocity and
drops this average velocity. The $4$-step training rollouts then degrade, so the
reward briefly rises and collapses. In contrast, our default keeps both pair types throughout training and improves
stably (\cref{fig:rho_curves}).
\begin{figure}[!ht]
    \vspace{-0.1in}
    \centering
    \includegraphics[width=0.95\linewidth]{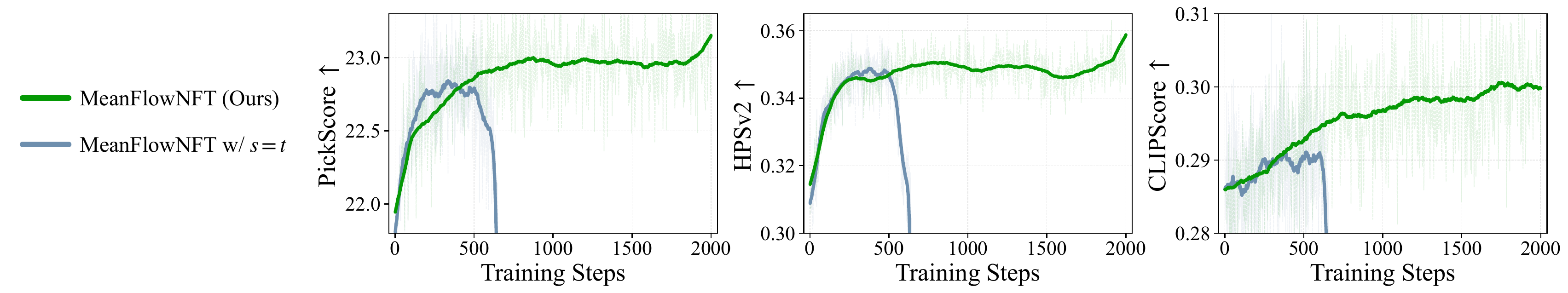}
    \vspace{-0.12in}
    \caption{Training reward curves for MeanFlowNFT and a variant trained only on $s{=}t$ pairs, \emph{i.e.}, only the instantaneous velocity $\vu_\theta(\vx_t,t,t)$.}
    \label{fig:rho_curves}
    \vspace{-0.15in}
\end{figure}

\section{Conclusion}
\label{sec:conclusion}
In this work, we presented MeanFlowNFT, the first forward-process RL method tailored for MeanFlow models. 
MeanFlow predicts an average velocity, but the DiffusionNFT framework we build on works with the instantaneous velocity. To bridge this gap, we use the MeanFlow
identity to optimize the reward on the instantaneous velocity while keeping the average-velocity network
for fast sampling. We prove that, in an idealized setting, this reaches DiffusionNFT's improved policy
and carries the gain to the average-velocity model. Experiments show that MeanFlowNFT consistently
improves MeanFlow baselines and can even surpass multi-step RL with only a few sampling steps.

\section{Limitations and Future Work}
\label{sec:limitation}
In terms of limitations, we only explore DiffusionNFT-style forward-process RL on MeanFlow. We do not
study other forward-process objectives, such as RAM~\citep{ram} and AWM~\citep{xue2025advantage}. Since
these methods also act on the instantaneous velocity and rely on a frozen reference, we expect our
induced instantaneous-velocity construction and practical implementation to carry over to them.
Additionally, we consider only MeanFlow, one instance of the broader family of flow-map models. These
models learn a direct map between two time points, so one or a few steps replace iterative sampling.
MeanFlow realizes this by predicting the average velocity over an interval, while other instances
include shortcut models~\citep{shortcut-model}, consistency trajectory models~\citep{kim2024consistency}, \emph{etc}.
Because the induced-predictor construction in \cref{eq:V_theta} is not specific to MeanFlow, we believe
our approach can also extend to them. In short, both directions
lie outside the scope of this paper, and we leave them to future work.

\bibliography{ref}

\begin{thebibliography}{68}
\providecommand{\natexlab}[1]{#1}
\providecommand{\url}[1]{\texttt{#1}}
\expandafter\ifx\csname urlstyle\endcsname\relax
  \providecommand{\doi}[1]{doi: #1}\else
  \providecommand{\doi}{doi: \begingroup \urlstyle{rm}\Url}\fi

\bibitem[Bergmeister et~al.(2026)Bergmeister, Jegelka, Nüsken, Domingo-Enrich, and Pidstrigach]{ram}
Andreas Bergmeister, Stefanie Jegelka, Nikolas Nüsken, Carles Domingo-Enrich, and Jakiw Pidstrigach.
\newblock Reinforce adjoint matching: Scaling rl post-training of diffusion and flow-matching models, 2026.
\newblock URL \url{https://arxiv.org/abs/2605.10759}.

\bibitem[Black et~al.(2024)Black, Janner, Du, Kostrikov, and Levine]{ddpo}
Kevin Black, Michael Janner, Yilun Du, Ilya Kostrikov, and Sergey Levine.
\newblock Training diffusion models with reinforcement learning.
\newblock In \emph{The Twelfth International Conference on Learning Representations}, 2024.
\newblock URL \url{https://openreview.net/forum?id=YCWjhGrJFD}.

\bibitem[Chen et~al.(2026)Chen, Huang, Sun, Liu, Zhu, Qu, Cheng, and Chen]{flash-dmd}
Guanjie Chen, Shirui Huang, Yifu Sun, Kai Liu, Jianchen Zhu, Xiaoye Qu, Yu~Cheng, and Peng Chen.
\newblock Flash-dmd: Towards high-fidelity few-step image generation with efficient distillation and joint reinforcement learning.
\newblock In \emph{Proceedings of the IEEE/CVF Conference on Computer Vision and Pattern Recognition}, pages 6010--6020, 2026.

\bibitem[Dong et~al.(2026)Dong, Guo, Bai, Yuan, Luo, and Zou]{dong2026guidingdistributionmatchingdistillation}
Linwei Dong, Ruoyu Guo, Ge~Bai, Zehuan Yuan, Yawei Luo, and Changqing Zou.
\newblock Guiding distribution matching distillation with gradient-based reinforcement learning, 2026.
\newblock URL \url{https://arxiv.org/abs/2604.19009}.

\bibitem[Esser et~al.(2024)Esser, Kulal, Blattmann, Entezari, M{\"u}ller, Saini, Levi, Lorenz, Sauer, Boesel, et~al.]{esser2024scaling}
Patrick Esser, Sumith Kulal, Andreas Blattmann, Rahim Entezari, Jonas M{\"u}ller, Harry Saini, Yam Levi, Dominik Lorenz, Axel Sauer, Frederic Boesel, et~al.
\newblock Scaling rectified flow transformers for high-resolution image synthesis.
\newblock In \emph{Forty-first international conference on machine learning}, 2024.

\bibitem[Fan et~al.(2026)Fan, Sun, Wen, Lu, and Song]{fan2026rtextdmreconceptualizingdistributionmatching}
Linqian Fan, Peiqin Sun, Tiancheng Wen, Shun Lu, and Chengru Song.
\newblock $r_\text{dm}$: Re-conceptualizing distribution matching as a reward for diffusion distillation, 2026.
\newblock URL \url{https://arxiv.org/abs/2603.28460}.

\bibitem[Fan et~al.(2023)Fan, Watkins, Du, Liu, Ryu, Boutilier, Abbeel, Ghavamzadeh, Lee, and Lee]{fan2023dpok}
Ying Fan, Olivia Watkins, Yuqing Du, Hao Liu, Moonkyung Ryu, Craig Boutilier, Pieter Abbeel, Mohammad Ghavamzadeh, Kangwook Lee, and Kimin Lee.
\newblock Dpok: Reinforcement learning for fine-tuning text-to-image diffusion models.
\newblock \emph{Advances in Neural Information Processing Systems}, 36:\penalty0 79858--79885, 2023.

\bibitem[Frans et~al.(2025)Frans, Hafner, Levine, and Abbeel]{shortcut-model}
Kevin Frans, Danijar Hafner, Sergey Levine, and Pieter Abbeel.
\newblock One step diffusion via shortcut models.
\newblock In \emph{The Thirteenth International Conference on Learning Representations}, 2025.
\newblock URL \url{https://openreview.net/forum?id=OlzB6LnXcS}.

\bibitem[Ge et~al.(2026)Ge, Zhang, Huang, He, Wang, Ma, Song, Liu, and Zhang]{scdmd}
Xingtong Ge, Yi~Zhang, Yushi Huang, Dailan He, Xiahong Wang, Bingqi Ma, Guanglu Song, Yu~Liu, and Jun Zhang.
\newblock Salt: Self-consistent distribution matching with cache-aware training for fast video generation.
\newblock In \emph{European Conference on Computer Vision}. Springer, 2026.

\bibitem[Geng et~al.(2026{\natexlab{a}})Geng, Deng, Bai, Kolter, and He]{meanflow}
Zhengyang Geng, Mingyang Deng, Xingjian Bai, Zico Kolter, and Kaiming He.
\newblock Mean flows for one-step generative modeling.
\newblock \emph{Advances in Neural Information Processing Systems}, 38:\penalty0 75460--75482, 2026{\natexlab{a}}.

\bibitem[Geng et~al.(2026{\natexlab{b}})Geng, Lu, Wu, Shechtman, Kolter, and He]{improved_meanflow}
Zhengyang Geng, Yiyang Lu, Zongze Wu, Eli Shechtman, J~Zico Kolter, and Kaiming He.
\newblock Improved mean flows: On the challenges of fastforward generative models.
\newblock In \emph{Proceedings of the IEEE/CVF Conference on Computer Vision and Pattern Recognition}, pages 30467--30476, 2026{\natexlab{b}}.

\bibitem[Gu et~al.(2026)Gu, Fang, Jiang, Mao, Han, Cai, and Shou]{gu2026anyflowanystepvideodiffusion}
Yuchao Gu, Guian Fang, Yuxin Jiang, Weijia Mao, Song Han, Han Cai, and Mike~Zheng Shou.
\newblock Anyflow: Any-step video diffusion model with on-policy flow map distillation, 2026.
\newblock URL \url{https://arxiv.org/abs/2605.13724}.

\bibitem[Guo et~al.(2025)Guo, Yang, Zhang, Song, Wang, Zhu, Xu, Zhang, Ma, Bi, et~al.]{guo2025deepseek}
Daya Guo, Dejian Yang, Haowei Zhang, Junxiao Song, Peiyi Wang, Qihao Zhu, Runxin Xu, Ruoyu Zhang, Shirong Ma, Xiao Bi, et~al.
\newblock Deepseek-r1 incentivizes reasoning in llms through reinforcement learning.
\newblock \emph{Nature}, 645\penalty0 (8081):\penalty0 633--638, 2025.

\bibitem[He et~al.(2026)He, Fu, Zhao, Li, Yang, Yin, Rao, and Zhang]{tempflow-grpo}
Xiaoxuan He, Siming Fu, Yuke Zhao, Wanli Li, Jian Yang, Dacheng Yin, Fengyun Rao, and Bo~Zhang.
\newblock {TEMPFLOW}-{GRPO}: {WHEN} {TIMING} {MATTERS} {FOR} {GRPO} {IN} {FLOW} {MODELS}.
\newblock In \emph{The Fourteenth International Conference on Learning Representations}, 2026.
\newblock URL \url{https://openreview.net/forum?id=7mCo3R3Wyn}.

\bibitem[Hessel et~al.(2021)Hessel, Holtzman, Forbes, Le~Bras, and Choi]{clipscore}
Jack Hessel, Ari Holtzman, Maxwell Forbes, Ronan Le~Bras, and Yejin Choi.
\newblock Clipscore: A reference-free evaluation metric for image captioning.
\newblock In \emph{Proceedings of the 2021 conference on empirical methods in natural language processing}, pages 7514--7528, 2021.

\bibitem[Ho and Salimans(2021)]{ho2021classifierfree}
Jonathan Ho and Tim Salimans.
\newblock Classifier-free diffusion guidance.
\newblock In \emph{NeurIPS 2021 Workshop on Deep Generative Models and Downstream Applications}, 2021.
\newblock URL \url{https://openreview.net/forum?id=qw8AKxfYbI}.

\bibitem[Ho et~al.(2020)Ho, Jain, and Abbeel]{ddpm}
Jonathan Ho, Ajay Jain, and Pieter Abbeel.
\newblock Denoising diffusion probabilistic models.
\newblock \emph{Advances in neural information processing systems}, 33:\penalty0 6840--6851, 2020.

\bibitem[Hu et~al.(2022)Hu, Shen, Wallis, Allen-Zhu, Li, Wang, Wang, Chen, et~al.]{lora}
Edward~J Hu, Yelong Shen, Phillip Wallis, Zeyuan Allen-Zhu, Yuanzhi Li, Shean Wang, Lu~Wang, Weizhu Chen, et~al.
\newblock Lora: Low-rank adaptation of large language models.
\newblock \emph{ICLR}, 1\penalty0 (2):\penalty0 3, 2022.

\bibitem[Huang et~al.(2026)Huang, Zhou, Wang, Zhang, Zhang, and Pang]{huang2026reinforcingfewstepgeneratorsrewardtilted}
Yushi Huang, Xiangxin Zhou, Ruoyu Wang, Chi Zhang, Jun Zhang, and Tianyu Pang.
\newblock Reinforcing few-step generators via reward-tilted distribution matching, 2026.
\newblock URL \url{https://arxiv.org/abs/2605.26108}.

\bibitem[Huang et~al.(2024)Huang, He, Yu, Zhang, Si, Jiang, Zhang, Wu, Jin, Chanpaisit, et~al.]{huang2024vbench}
Ziqi Huang, Yinan He, Jiashuo Yu, Fan Zhang, Chenyang Si, Yuming Jiang, Yuanhan Zhang, Tianxing Wu, Qingyang Jin, Nattapol Chanpaisit, et~al.
\newblock Vbench: Comprehensive benchmark suite for video generative models.
\newblock In \emph{Proceedings of the IEEE/CVF Conference on Computer Vision and Pattern Recognition}, pages 21807--21818, 2024.

\bibitem[Jiang et~al.(2026)Jiang, Liu, Wang, Wu, Li, Li, Jin, Liu, Li, Zhang, et~al.]{dmdr}
Dengyang Jiang, Dongyang Liu, Zanyi Wang, Qilong Wu, Liuzhuozheng Li, Hengzhuang Li, Xin Jin, David Liu, Zhen Li, Bo~Zhang, et~al.
\newblock Distribution matching distillation meets reinforcement learning.
\newblock In \emph{European Conference on Computer Vision}. Springer, 2026.

\bibitem[Kamath et~al.(2025)Kamath, Chang, Krishna, Zettlemoyer, Hu, and Ghazvininejad]{kamath2025geneval2addressingbenchmark}
Amita Kamath, Kai-Wei Chang, Ranjay Krishna, Luke Zettlemoyer, Yushi Hu, and Marjan Ghazvininejad.
\newblock Geneval 2: Addressing benchmark drift in text-to-image evaluation, 2025.
\newblock URL \url{https://arxiv.org/abs/2512.16853}.

\bibitem[Kim et~al.(2024)Kim, Lai, Liao, Murata, Takida, Uesaka, He, Mitsufuji, and Ermon]{kim2024consistency}
Dongjun Kim, Chieh-Hsin Lai, Wei-Hsiang Liao, Naoki Murata, Yuhta Takida, Toshimitsu Uesaka, Yutong He, Yuki Mitsufuji, and Stefano Ermon.
\newblock Consistency trajectory models: Learning probability flow {ODE} trajectory of diffusion.
\newblock In \emph{The Twelfth International Conference on Learning Representations}, 2024.
\newblock URL \url{https://openreview.net/forum?id=ymjI8feDTD}.

\bibitem[Kirstain et~al.(2023)Kirstain, Polyak, Singer, Matiana, Penna, and Levy]{pickscore}
Yuval Kirstain, Adam Polyak, Uriel Singer, Shahbuland Matiana, Joe Penna, and Omer Levy.
\newblock Pick-a-pic: An open dataset of user preferences for text-to-image generation.
\newblock \emph{Advances in neural information processing systems}, 36:\penalty0 36652--36663, 2023.

\bibitem[Labs(2024)]{flux}
Black~Forest Labs.
\newblock Flux.
\newblock \url{https://github.com/black-forest-labs/flux}, 2024.

\bibitem[Lee et~al.(2023)Lee, Liu, Ryu, Watkins, Du, Boutilier, Abbeel, Ghavamzadeh, and Gu]{lee2023aligning}
Kimin Lee, Hao Liu, Moonkyung Ryu, Olivia Watkins, Yuqing Du, Craig Boutilier, Pieter Abbeel, Mohammad Ghavamzadeh, and Shixiang~Shane Gu.
\newblock Aligning text-to-image models using human feedback.
\newblock \emph{arXiv preprint arXiv:2302.12192}, 2023.

\bibitem[Levine(2018)]{levine2018reinforcement}
Sergey Levine.
\newblock Reinforcement learning and control as probabilistic inference: Tutorial and review.
\newblock \emph{arXiv preprint arXiv:1805.00909}, 2018.

\bibitem[Li et~al.(2026{\natexlab{a}})Li, Cui, Huang, Ma, Fan, Yang, and Zhong]{mixgrpo}
Junzhe Li, Yutao Cui, Tao Huang, Yinping Ma, Chun Fan, Miles Yang, and Zhao Zhong.
\newblock Mixgrpo: Unlocking flow-based grpo efficiency with mixed ode-sde.
\newblock In \emph{European Conference on Computer Vision}. Springer, 2026{\natexlab{a}}.

\bibitem[Li et~al.(2026{\natexlab{b}})Li, Zhang, and Zhu]{flowmapgrpo}
Zhiqi Li, Wen Zhang, and Bo~Zhu.
\newblock Flow-map grpo: Reinforcement learning for few-step flow-map generators via anchored stochastic composition, 2026{\natexlab{b}}.
\newblock URL \url{https://arxiv.org/abs/2607.00535}.

\bibitem[Lipman et~al.(2023)Lipman, Chen, Ben-Hamu, Nickel, and Le]{lipman2022flow}
Yaron Lipman, Ricky T.~Q. Chen, Heli Ben-Hamu, Maximilian Nickel, and Matthew Le.
\newblock Flow matching for generative modeling.
\newblock In \emph{The Eleventh International Conference on Learning Representations}, 2023.
\newblock URL \url{https://openreview.net/forum?id=PqvMRDCJT9t}.

\bibitem[Liu et~al.(2026{\natexlab{a}})Liu, Liu, Liang, Li, Liu, Wang, Wan, ZHANG, and Ouyang]{flowgrpo}
Jie Liu, Gongye Liu, Jiajun Liang, Yangguang Li, Jiaheng Liu, Xintao Wang, Pengfei Wan, Di~ZHANG, and Wanli Ouyang.
\newblock Flow-{GRPO}: Training flow matching models via online {RL}.
\newblock In \emph{The Thirty-ninth Annual Conference on Neural Information Processing Systems}, 2026{\natexlab{a}}.
\newblock URL \url{https://openreview.net/forum?id=oCBKGw5HNf}.

\bibitem[Liu et~al.(2026{\natexlab{b}})Liu, Liu, Liang, Yuan, Liu, Zheng, Wu, Wang, Xia, Wang, et~al.]{liu2025improvingvideogenerationhuman}
Jie Liu, Gongye Liu, Jiajun Liang, Ziyang Yuan, Xiaokun Liu, Mingwu Zheng, Xiele Wu, Qiulin Wang, Menghan Xia, Xintao Wang, et~al.
\newblock Improving video generation with human feedback.
\newblock \emph{Advances in Neural Information Processing Systems}, 38:\penalty0 82155--82192, 2026{\natexlab{b}}.

\bibitem[Liu et~al.(2026{\natexlab{c}})Liu, Yan, Chen, Hu, Yue, Pan, Lan, Zhu, Cheng, Zheng, and Wang]{cdm}
Tao Liu, Hao Yan, Mengting Chen, Taihang Hu, Zhengrong Yue, Zihao Pan, Jinsong Lan, Xiaoyong Zhu, Ming-Ming Cheng, Bo~Zheng, and Yaxing Wang.
\newblock Continuous-time distribution matching for few-step diffusion distillation.
\newblock 2026{\natexlab{c}}.
\newblock URL \url{https://arxiv.org/abs/2605.06376}.

\bibitem[Liu et~al.(2023)Liu, Gong, and qiang liu]{liu2022flow}
Xingchao Liu, Chengyue Gong, and qiang liu.
\newblock Flow straight and fast: Learning to generate and transfer data with rectified flow.
\newblock In \emph{The Eleventh International Conference on Learning Representations}, 2023.
\newblock URL \url{https://openreview.net/forum?id=XVjTT1nw5z}.

\bibitem[Loshchilov and Hutter(2019)]{loshchilov2018decoupled}
Ilya Loshchilov and Frank Hutter.
\newblock Decoupled weight decay regularization.
\newblock In \emph{International Conference on Learning Representations}, 2019.
\newblock URL \url{https://openreview.net/forum?id=Bkg6RiCqY7}.

\bibitem[Luo et~al.(2023{\natexlab{a}})Luo, Tan, Huang, Li, and Zhao]{lcm}
Simian Luo, Yiqin Tan, Longbo Huang, Jian Li, and Hang Zhao.
\newblock Latent consistency models: Synthesizing high-resolution images with few-step inference.
\newblock \emph{arXiv preprint arXiv:2310.04378}, 2023{\natexlab{a}}.

\bibitem[Luo et~al.(2023{\natexlab{b}})Luo, Hu, Zhang, Sun, Li, and Zhang]{diffinstruct}
Weijian Luo, Tianyang Hu, Shifeng Zhang, Jiacheng Sun, Zhenguo Li, and Zhihua Zhang.
\newblock Diff-instruct: A universal approach for transferring knowledge from pre-trained diffusion models.
\newblock \emph{Advances in Neural Information Processing Systems}, 36:\penalty0 76525--76546, 2023{\natexlab{b}}.

\bibitem[Luo et~al.(2026)Luo, Hu, Luo, and Tang]{luo2026tdmr1reinforcingfewstepdiffusion}
Yihong Luo, Tianyang Hu, Weijian Luo, and Jing Tang.
\newblock {TDM}-r1: Reinforcing few-step diffusion models with non-differentiable reward.
\newblock In \emph{Forty-third International Conference on Machine Learning}, 2026.
\newblock URL \url{https://openreview.net/forum?id=4w9DpowGcs}.

\bibitem[Ma et~al.(2024)Ma, Goldstein, Albergo, Boffi, Vanden-Eijnden, and Xie]{sit}
Nanye Ma, Mark Goldstein, Michael~S Albergo, Nicholas~M Boffi, Eric Vanden-Eijnden, and Saining Xie.
\newblock Sit: Exploring flow and diffusion-based generative models with scalable interpolant transformers.
\newblock In \emph{European Conference on Computer Vision}, pages 23--40. Springer, 2024.

\bibitem[Ma et~al.(2025)Ma, Wu, Sun, and Li]{hpsv3}
Yuhang Ma, Xiaoshi Wu, Keqiang Sun, and Hongsheng Li.
\newblock Hpsv3: Towards wide-spectrum human preference score.
\newblock In \emph{Proceedings of the IEEE/CVF International Conference on Computer Vision}, pages 15086--15095, 2025.

\bibitem[Miao et~al.(2025)Miao, Yang, Lin, Wang, Liu, Wang, and Qiu]{pso}
Zichen Miao, Zhengyuan Yang, Kevin Lin, Ze~Wang, Zicheng Liu, Lijuan Wang, and Qiang Qiu.
\newblock Tuning timestep-distilled diffusion model using pairwise sample optimization.
\newblock In \emph{The Thirteenth International Conference on Learning Representations}, 2025.
\newblock URL \url{https://openreview.net/forum?id=fXnE4gB64o}.

\bibitem[Nie et~al.(2026)Nie, Berner, Ma, Liu, Xie, and Vahdat]{Nie_2026_CVPR}
Weili Nie, Julius Berner, Nanye Ma, Chao Liu, Saining Xie, and Arash Vahdat.
\newblock Transition matching distillation for fast video generation.
\newblock In \emph{Proceedings of the IEEE/CVF Conference on Computer Vision and Pattern Recognition (CVPR)}, pages 4645--4655, June 2026.

\bibitem[Peebles and Xie(2023)]{dit}
William Peebles and Saining Xie.
\newblock Scalable diffusion models with transformers.
\newblock In \emph{Proceedings of the IEEE/CVF international conference on computer vision}, pages 4195--4205, 2023.

\bibitem[Podell et~al.(2024)Podell, English, Lacey, Blattmann, Dockhorn, M{\"u}ller, Penna, and Rombach]{sdxl}
Dustin Podell, Zion English, Kyle Lacey, Andreas Blattmann, Tim Dockhorn, Jonas M{\"u}ller, Joe Penna, and Robin Rombach.
\newblock {SDXL}: Improving latent diffusion models for high-resolution image synthesis.
\newblock In \emph{The Twelfth International Conference on Learning Representations}, 2024.
\newblock URL \url{https://openreview.net/forum?id=di52zR8xgf}.

\bibitem[Qin et~al.(2025)Qin, Zhuo, Xin, Du, Li, Fu, Lu, Li, Liu, Zhu, Beddow, Millon, Perez, Wang, Qiao, Zhang, Liu, Li, Xu, and Gao]{lumina2}
Qi~Qin, Le~Zhuo, Yi~Xin, Ruoyi Du, Zhen Li, Bin Fu, Yiting Lu, Xinyue Li, Dongyang Liu, Xiangyang Zhu, Will Beddow, Erwann Millon, Victor Perez, Wenhai Wang, Yu~Qiao, Bo~Zhang, Xiaohong Liu, Hongsheng Li, Chang Xu, and Peng Gao.
\newblock Lumina-image 2.0: A unified and efficient image generative framework.
\newblock In \emph{Proceedings of the IEEE/CVF International Conference on Computer Vision (ICCV)}, pages 20031--20042, October 2025.

\bibitem[Rombach et~al.(2022)Rombach, Blattmann, Lorenz, Esser, and Ommer]{ldm}
Robin Rombach, Andreas Blattmann, Dominik Lorenz, Patrick Esser, and Bj{\"o}rn Ommer.
\newblock High-resolution image synthesis with latent diffusion models.
\newblock In \emph{Proceedings of the IEEE/CVF conference on computer vision and pattern recognition}, pages 10684--10695, 2022.

\bibitem[Saharia et~al.(2022)Saharia, Chan, Saxena, Li, Whang, Denton, Ghasemipour, Gontijo-Lopes, Ayan, Salimans, Ho, Fleet, and Norouzi]{saharia2022photorealistictexttoimagediffusionmodels}
Chitwan Saharia, William Chan, Saurabh Saxena, Lala Li, Jay Whang, Emily Denton, Seyed Kamyar~Seyed Ghasemipour, Raphael Gontijo-Lopes, Burcu~Karagol Ayan, Tim Salimans, Jonathan Ho, David~J. Fleet, and Mohammad Norouzi.
\newblock Photorealistic text-to-image diffusion models with deep language understanding.
\newblock In Alice~H. Oh, Alekh Agarwal, Danielle Belgrave, and Kyunghyun Cho, editors, \emph{Advances in Neural Information Processing Systems}, 2022.
\newblock URL \url{https://openreview.net/forum?id=08Yk-n5l2Al}.

\bibitem[Salimans and Ho(2022)]{salimans2022progressive}
Tim Salimans and Jonathan Ho.
\newblock Progressive distillation for fast sampling of diffusion models.
\newblock In \emph{The Tenth International Conference on Learning Representations, {ICLR} 2022, Virtual Event, April 25-29, 2022}. OpenReview.net, 2022.
\newblock URL \url{https://openreview.net/forum?id=TIdIXIpzhoI}.

\bibitem[Sauer et~al.(2024{\natexlab{a}})Sauer, Boesel, Dockhorn, Blattmann, Esser, and Rombach]{ladd}
Axel Sauer, Frederic Boesel, Tim Dockhorn, Andreas Blattmann, Patrick Esser, and Robin Rombach.
\newblock Fast high-resolution image synthesis with latent adversarial diffusion distillation.
\newblock In \emph{SIGGRAPH Asia 2024 Conference Papers}, pages 1--11, 2024{\natexlab{a}}.

\bibitem[Sauer et~al.(2024{\natexlab{b}})Sauer, Lorenz, Blattmann, and Rombach]{add}
Axel Sauer, Dominik Lorenz, Andreas Blattmann, and Robin Rombach.
\newblock Adversarial diffusion distillation.
\newblock In \emph{European Conference on Computer Vision}, pages 87--103. Springer, 2024{\natexlab{b}}.

\bibitem[Schuhmann(2022)]{schuhmann2022aesthetics}
Christoph Schuhmann.
\newblock Laion-aesthetics.
\newblock \url{https://laion.ai/blog/laion-aesthetics/}, 2022.

\bibitem[Seedream et~al.(2025)Seedream, Chen, Gao, Gong, Guo, Guo, Guo, Hou, Huang, Huang, et~al.]{seedream4}
Team Seedream, Yunpeng Chen, Yu~Gao, Lixue Gong, Meng Guo, Qiushan Guo, Zhiyao Guo, Xiaoxia Hou, Weilin Huang, Yixuan Huang, et~al.
\newblock Seedream 4.0: Toward next-generation multimodal image generation.
\newblock \emph{arXiv preprint arXiv:2509.20427}, 2025.

\bibitem[Song et~al.(2021{\natexlab{a}})Song, Meng, and Ermon]{ddim}
Jiaming Song, Chenlin Meng, and Stefano Ermon.
\newblock Denoising diffusion implicit models.
\newblock In \emph{International Conference on Learning Representations}, 2021{\natexlab{a}}.
\newblock URL \url{https://openreview.net/forum?id=St1giarCHLP}.

\bibitem[Song et~al.(2021{\natexlab{b}})Song, Sohl-Dickstein, Kingma, Kumar, Ermon, and Poole]{score-base-diff}
Yang Song, Jascha Sohl-Dickstein, Diederik~P Kingma, Abhishek Kumar, Stefano Ermon, and Ben Poole.
\newblock Score-based generative modeling through stochastic differential equations.
\newblock In \emph{International Conference on Learning Representations}, 2021{\natexlab{b}}.
\newblock URL \url{https://openreview.net/forum?id=PxTIG12RRHS}.

\bibitem[Song et~al.(2023)Song, Dhariwal, Chen, and Sutskever]{song2023consistency}
Yang Song, Prafulla Dhariwal, Mark Chen, and Ilya Sutskever.
\newblock Consistency models.
\newblock In \emph{International Conference on Machine Learning}, pages 32211--32252. PMLR, 2023.

\bibitem[Team et~al.(2025)Team, Cai, Huang, Kang, Li, Liang, Ma, Ren, Wei, Xie, and Zhang]{meituanlongcatteam2025longcatvideotechnicalreport}
Meituan~LongCat Team, Xunliang Cai, Qilong Huang, Zhuoliang Kang, Hongyu Li, Shijun Liang, Liya Ma, Siyu Ren, Xiaoming Wei, Rixu Xie, and Tong Zhang.
\newblock Longcat-video technical report, 2025.
\newblock URL \url{https://arxiv.org/abs/2510.22200}.

\bibitem[Wallace et~al.(2024)Wallace, Dang, Rafailov, Zhou, Lou, Purushwalkam, Ermon, Xiong, Joty, and Naik]{diffusiondpo}
Bram Wallace, Meihua Dang, Rafael Rafailov, Linqi Zhou, Aaron Lou, Senthil Purushwalkam, Stefano Ermon, Caiming Xiong, Shafiq Joty, and Nikhil Naik.
\newblock Diffusion model alignment using direct preference optimization.
\newblock In \emph{Proceedings of the IEEE/CVF Conference on Computer Vision and Pattern Recognition}, pages 8228--8238, 2024.

\bibitem[Wan et~al.(2025)Wan, Wang, Ai, Wen, Mao, Xie, Chen, Yu, Zhao, Yang, Zeng, Wang, Zhang, Zhou, Wang, Chen, Zhu, Zhao, Yan, Huang, Feng, Zhang, Li, Wu, Chu, Feng, Zhang, Sun, Fang, Wang, Gui, Weng, Shen, Lin, Wang, Wang, Zhou, Wang, Shen, Yu, Shi, Huang, Xu, Kou, Lv, Li, Liu, Wang, Zhang, Huang, Li, Wu, Liu, Pan, Zheng, Hong, Shi, Feng, Jiang, Han, Wu, and Liu]{wan2025wanopenadvancedlargescale}
Team Wan, Ang Wang, Baole Ai, Bin Wen, Chaojie Mao, Chen-Wei Xie, Di~Chen, Feiwu Yu, Haiming Zhao, Jianxiao Yang, Jianyuan Zeng, Jiayu Wang, Jingfeng Zhang, Jingren Zhou, Jinkai Wang, Jixuan Chen, Kai Zhu, Kang Zhao, Keyu Yan, Lianghua Huang, Mengyang Feng, Ningyi Zhang, Pandeng Li, Pingyu Wu, Ruihang Chu, Ruili Feng, Shiwei Zhang, Siyang Sun, Tao Fang, Tianxing Wang, Tianyi Gui, Tingyu Weng, Tong Shen, Wei Lin, Wei Wang, Wei Wang, Wenmeng Zhou, Wente Wang, Wenting Shen, Wenyuan Yu, Xianzhong Shi, Xiaoming Huang, Xin Xu, Yan Kou, Yangyu Lv, Yifei Li, Yijing Liu, Yiming Wang, Yingya Zhang, Yitong Huang, Yong Li, You Wu, Yu~Liu, Yulin Pan, Yun Zheng, Yuntao Hong, Yupeng Shi, Yutong Feng, Zeyinzi Jiang, Zhen Han, Zhi-Fan Wu, and Ziyu Liu.
\newblock Wan: Open and advanced large-scale video generative models, 2025.
\newblock URL \url{https://arxiv.org/abs/2503.20314}.

\bibitem[Wang et~al.(2026)Wang, Niu, Zhou, Huang, Liu, and Zhang]{wang2026exploring}
Ruoyu Wang, Boye Niu, Xiangxin Zhou, Yushi Huang, Tongliang Liu, and Chi Zhang.
\newblock Exploring the design space of reward backpropagation for flow matching.
\newblock \emph{arXiv preprint arXiv:2606.11075}, 2026.

\bibitem[Wu et~al.(2023)Wu, Hao, Sun, Chen, Zhu, Zhao, and Li]{hpsv2}
Xiaoshi Wu, Yiming Hao, Keqiang Sun, Yixiong Chen, Feng Zhu, Rui Zhao, and Hongsheng Li.
\newblock Human preference score v2: A solid benchmark for evaluating human preferences of text-to-image synthesis.
\newblock \emph{arXiv preprint arXiv:2306.09341}, 2023.

\bibitem[Xu et~al.(2023)Xu, Liu, Wu, Tong, Li, Ding, Tang, and Dong]{xu2023imagereward}
Jiazheng Xu, Xiao Liu, Yuchen Wu, Yuxuan Tong, Qinkai Li, Ming Ding, Jie Tang, and Yuxiao Dong.
\newblock Imagereward: Learning and evaluating human preferences for text-to-image generation.
\newblock \emph{Advances in Neural Information Processing Systems}, 36:\penalty0 15903--15935, 2023.

\bibitem[Xue et~al.(2026)Xue, GE, Zhang, Li, and Ma]{xue2025advantage}
Shuchen Xue, Chongjian GE, Shilong Zhang, Yichen Li, and Zhi-Ming Ma.
\newblock Advantage weighted matching: Aligning {RL} with pretraining in diffusion models.
\newblock In \emph{Forty-third International Conference on Machine Learning}, 2026.
\newblock URL \url{https://openreview.net/forum?id=nLY2pOYBrJ}.

\bibitem[Xue et~al.(2025)Xue, Wu, Gao, Kong, Zhu, Chen, Liu, Liu, Guo, Huang, et~al.]{xue2025dancegrpo}
Zeyue Xue, Jie Wu, Yu~Gao, Fangyuan Kong, Lingting Zhu, Mengzhao Chen, Zhiheng Liu, Wei Liu, Qiushan Guo, Weilin Huang, et~al.
\newblock Dancegrpo: Unleashing grpo on visual generation.
\newblock \emph{arXiv preprint arXiv:2505.07818}, 2025.

\bibitem[Yin et~al.(2024{\natexlab{a}})Yin, Gharbi, Park, Zhang, Shechtman, Durand, and Freeman]{dmd2}
Tianwei Yin, Micha{\"e}l Gharbi, Taesung Park, Richard Zhang, Eli Shechtman, Fredo Durand, and Bill Freeman.
\newblock Improved distribution matching distillation for fast image synthesis.
\newblock \emph{Advances in neural information processing systems}, 37:\penalty0 47455--47487, 2024{\natexlab{a}}.

\bibitem[Yin et~al.(2024{\natexlab{b}})Yin, Gharbi, Zhang, Shechtman, Durand, Freeman, and Park]{dmd}
Tianwei Yin, Micha{\"e}l Gharbi, Richard Zhang, Eli Shechtman, Fredo Durand, William~T Freeman, and Taesung Park.
\newblock One-step diffusion with distribution matching distillation.
\newblock In \emph{Proceedings of the IEEE/CVF conference on computer vision and pattern recognition}, pages 6613--6623, 2024{\natexlab{b}}.

\bibitem[Zhao et~al.(2023)Zhao, Gu, Varma, Luo, Huang, Xu, Wright, Shojanazeri, Ott, Shleifer, Desmaison, Balioglu, Damania, Nguyen, Chauhan, Hao, Mathews, and Li]{fsdp}
Yanli Zhao, Andrew Gu, Rohan Varma, Liang Luo, Chien-Chin Huang, Min Xu, Less Wright, Hamid Shojanazeri, Myle Ott, Sam Shleifer, Alban Desmaison, Can Balioglu, Pritam Damania, Bernard Nguyen, Geeta Chauhan, Yuchen Hao, Ajit Mathews, and Shen Li.
\newblock Pytorch fsdp: Experiences on scaling fully sharded data parallel.
\newblock \emph{Proc. VLDB Endow.}, 16\penalty0 (12):\penalty0 3848–3860, August 2023.
\newblock ISSN 2150-8097.
\newblock \doi{10.14778/3611540.3611569}.
\newblock URL \url{https://doi.org/10.14778/3611540.3611569}.

\bibitem[Zheng et~al.(2026{\natexlab{a}})Zheng, Chen, Ye, Wang, Zhang, Jiang, Su, Ermon, Zhu, and Liu]{diffusionnft}
Kaiwen Zheng, Huayu Chen, Haotian Ye, Haoxiang Wang, Qinsheng Zhang, Kai Jiang, Hang Su, Stefano Ermon, Jun Zhu, and Ming-Yu Liu.
\newblock Diffusion{NFT}: Online diffusion reinforcement with forward process.
\newblock In \emph{The Fourteenth International Conference on Learning Representations}, 2026{\natexlab{a}}.
\newblock URL \url{https://openreview.net/forum?id=VJZ477R89F}.

\bibitem[Zheng et~al.(2026{\natexlab{b}})Zheng, Wang, Ma, Chen, Zhang, Balaji, Chen, Liu, Zhu, and Zhang]{rcm}
Kaiwen Zheng, Yuji Wang, Qianli Ma, Huayu Chen, Jintao Zhang, Yogesh Balaji, Jianfei Chen, Ming-Yu Liu, Jun Zhu, and Qinsheng Zhang.
\newblock Large scale diffusion distillation via score-regularized continuous-time consistency.
\newblock In \emph{The Fourteenth International Conference on Learning Representations}, 2026{\natexlab{b}}.
\newblock URL \url{https://openreview.net/forum?id=2uNlM353RI}.

\end{thebibliography}

\clearpage
\appendix
\appendixtitle
\crefalias{section}{appendix}
\startcontents[app]
\printcontents[app]{l}{1}{}

\section{Related Work}
\label{app:related_work}

\noindent\textbf{Diffusion and flow models.} Diffusion~\citep{ddpm,score-base-diff,ddim} and flow models~\citep{liu2022flow,lipman2022flow} have become the dominant paradigm for
high-quality image and video generation. Diffusion models synthesize data by learning to reverse a
fixed noising process, and latent diffusion~\citep{ldm} runs this
process in a compressed latent space for efficiency. Transformer backbones improve scalability~\citep{dit,sit}, and scaling rectified-flow transformers
enables high-resolution synthesis~\citep{esser2024scaling}. Flow
matching~\citep{lipman2022flow} and rectified flow~\citep{liu2022flow} recast generation as regressing
a velocity field that transports noise to data. Building on these formulations, large
text-to-image~\citep{sdxl,flux,lumina2,seedream4} and text-to-video~\citep{wan2025wanopenadvancedlargescale,meituanlongcatteam2025longcatvideotechnicalreport}
systems reach strong visual quality.

\noindent\textbf{Few-step generation.} Because iterative sampling is slow, many methods distill pretrained
diffusion or flow models into few-step generators, including progressive
distillation~\citep{salimans2022progressive}, consistency models~\citep{song2023consistency,lcm},
distribution matching distillation~\citep{dmd,dmd2,diffinstruct}, and adversarial
distillation~\citep{add,ladd}. A closely related line learns \emph{flow maps}, long-range transport
operators that move samples directly between two time points instead of integrating an instantaneous
velocity~\citep{shortcut-model,kim2024consistency}. In particular, MeanFlow~\citep{meanflow,improved_meanflow} parameterizes the average velocity over a
time interval, so one evaluation advances a sample across it and enables one or few-step generation.
AnyFlow~\citep{gu2026anyflowanystepvideodiffusion} and transition matching~\citep{Nie_2026_CVPR} distill
such flow maps. Our work builds on this
average-velocity parameterization.

\noindent\textbf{Reinforcement learning for diffusion models.} Reinforcement learning is now a standard tool
for aligning diffusion and flow models with human preferences and task rewards. Existing approaches rely on
policy gradients~\citep{ddpo,fan2023dpok}, reward-weighted training~\citep{lee2023aligning},
preference optimization~\citep{diffusiondpo}, or reward
backpropagation~\citep{xu2023imagereward,wang2026exploring}. A dominant recent recipe discretizes the reverse sampling
process into a Markov decision process and applies GRPO-style policy
gradients~\citep{guo2025deepseek,flowgrpo,xue2025dancegrpo,mixgrpo,tempflow-grpo}, 
which need a stochastic policy and per-step likelihood estimation. A complementary line runs reinforcement learning on the
forward process, recasting reward optimization as a regression objective on analytically noised samples
and thereby avoiding reverse rollouts and likelihoods. It includes advantage weighted
matching~\citep{xue2025advantage}, DiffusionNFT~\citep{diffusionnft}, and reinforce adjoint
matching~\citep{ram}. MeanFlowNFT follows this forward-process line.

\noindent\textbf{Reinforcement learning for few-step generation.} A growing line of work brings reinforcement
learning to few-step generators. Most methods reinforce distribution-matching distillation, either by combining a reward objective with
the matching loss~\citep{dmdr,flash-dmd,fan2026rtextdmreconceptualizingdistributionmatching,dong2026guidingdistributionmatchingdistillation}
or by reward-tilting the target distribution~\citep{huang2026reinforcingfewstepgeneratorsrewardtilted}.
Others post-train distilled few-step generators with surrogate reward learning~\citep{luo2026tdmr1reinforcingfewstepdiffusion} or pairwise sample objectives~\citep{pso}. 
Overall, reinforcement learning for flow-map generators remains
underexplored. We study it through forward-process RL, bringing the DiffusionNFT~\citep{diffusionnft}
paradigm to MeanFlow. The concurrent Flow-Map GRPO~\citep{flowmapgrpo}
introduces path-preserving stochastic flow-map transitions during RL training
and applies GRPO with stochastic rollouts and per-step likelihood ratios.
MeanFlowNFT instead uses analytically noised forward-process samples and a
likelihood-free regression objective.

\section{Why Direct Plug-Ins Lack DiffusionNFT's Guarantee}
\label{sec:problem}

This section explains why the direct DiffusionNFT plug-ins for AnyFlow, DMD,
and CDM in \cref{sec:main_results} do not inherit the policy improvement
guarantee of DiffusionNFT. The issue is the meaning of the network output
rather than the number of reverse sampling steps.

\noindent\textbf{Output required by DiffusionNFT.}
The idealized DiffusionNFT analysis assumes that the optimized output is the
marginal instantaneous velocity
\begin{equation*}
\vv^q(\vx_t,t)
=
\E_{\pi^q}[\vv_t\mid\vx_t,\vc,t],
\qquad
\vv_t=\dot{\alpha}_t\vx_0+\dot{\sigma}_t\epsilonv,
\quad q\in\{\mathrm{old},+,-\}.
\end{equation*}
The map from $(\vx_0,\epsilonv)$ to $\vv_t$ is identical for all three
policies. Reward reweighting changes the posterior distribution of
$(\vx_0,\epsilonv)$ but leaves this map unchanged. At fixed
$(\vx_t,\vc,t)$, the definition
$\pi^{+}\propto r\pi^{\mathrm{old}}$ gives
$\E_{\pi^{\mathrm{old}}}[r\vv_t]=\alpha\vv^{+}$. The definition
$\pi^{-}\propto(1-r)\pi^{\mathrm{old}}$ similarly gives
$\E_{\pi^{\mathrm{old}}}[(1-r)\vv_t]=(1-\alpha)\vv^{-}$. All expectations
here are conditioned on the fixed tuple, and we write $\alpha$ for
$\alpha(\vx_t,\vc)$ in \cref{eq:alpha}. Adding these identities and using
$r+(1-r)=1$ gives
$\vv^{\mathrm{old}}=\alpha\vv^{+}+(1-\alpha)\vv^{-}$. This decomposition
yields the guidance $\Delta$ in \cref{eq:delta}
\citep[Thm.~3.1]{diffusionnft}. DiffusionNFT then minimizes
\cref{eq:diffnft_loss} over an instantaneous velocity predictor. Its pointwise
optimum is
$\vv_{\theta^*}=\vv^{\mathrm{old}}+\tfrac{2}{\beta}\Delta$
\citep[Thm.~3.2]{diffusionnft}. At points where $\alpha>0$, substituting
$\beta=2\alpha$ and
$\Delta=\alpha(\vv^{+}-\vv^{\mathrm{old}})$ gives
$\vv_{\theta^*}=\vv^{+}$. The guarantee therefore uses both the posterior mean
identity above and the fact that the optimized output is an instantaneous
velocity. Reusing the same loss for another output quantity preserves the
algebraic form but not this policy interpretation.

\noindent\textbf{Direct substitution of AnyFlow outputs.}
AnyFlow~\citep{gu2026anyflowanystepvideodiffusion} is a MeanFlow network that
predicts the interval average velocity $\vu$ in \cref{eq:avg_vel}, rather than
the marginal instantaneous velocity $\vv$ in \cref{eq:marg_vel}. The AnyFlow$+$DiffusionNFT baseline replaces $\vv_\theta$ and
$\vv^{\mathrm{old}}$ in \cref{eq:diffnft_loss} with $\vu_\theta$ and
$\vu^{\mathrm{old}}$. It keeps the regression target $\vv_t$. Fix
$(\vx_t,\vc,s,t)$ and regard $\vu_\theta(\vx_t,s,t)$ as a free output value.
The resulting conditional objective is
\begin{equation*}
\begin{aligned}
\ell\bigl(\vu_\theta(\vx_t,s,t)\bigr)
={}&\E_{\pi^{\mathrm{old}}}\!\Big[
r\bigl\|
\vu^{\mathrm{old}}+\beta(\vu_\theta-\vu^{\mathrm{old}})-\vv_t
\bigr\|_2^2\\
&+(1-r)\bigl\|
\vu^{\mathrm{old}}-\beta(\vu_\theta-\vu^{\mathrm{old}})-\vv_t
\bigr\|_2^2
\mid\vx_t,\vc,s,t
\Big],
\end{aligned}
\end{equation*}
where every average velocity in the display is evaluated at $(\vx_t,s,t)$.
The interval sampler makes $s$ independent of $(\vx_0,\epsilonv,r)$ given
$(\vc,t)$. Conditioning on $s$ therefore leaves the required posterior moments
unchanged.

Differentiating the conditional objective with respect to its output gives
\begin{equation*}
\begin{aligned}
\frac{1}{2\beta}\nabla_{\vu_\theta}\ell
={}&\beta\bigl(
\vu_\theta(\vx_t,s,t)-\vu^{\mathrm{old}}(\vx_t,s,t)
\bigr)-\E_{\pi^{\mathrm{old}}}\!\left[
(2r-1)\bigl(\vv_t-\vu^{\mathrm{old}}(\vx_t,s,t)\bigr)
\mid\vx_t,\vc,s,t
\right].
\end{aligned}
\end{equation*}
For $\beta>0$, the objective is strictly convex in this output value. Its
unique minimizer sets the displayed gradient to zero. The conditional residual
can be written as
\begin{equation*}
\begin{aligned}
&\E_{\pi^{\mathrm{old}}}\!\left[
(2r-1)\bigl(\vv_t-\vu^{\mathrm{old}}(\vx_t,s,t)\bigr)
\mid\vx_t,\vc,s,t
\right]\\
&=
2\E_{\pi^{\mathrm{old}}}[r\vv_t\mid\vx_t,\vc,t]
-\E_{\pi^{\mathrm{old}}}[\vv_t\mid\vx_t,\vc,t]
-\bigl(2\alpha-1\bigr)\vu^{\mathrm{old}}(\vx_t,s,t)\\
&=
2\alpha\vv^{+}(\vx_t,t)
-\vv^{\mathrm{old}}(\vx_t,t)
-\bigl(2\alpha-1\bigr)\vu^{\mathrm{old}}(\vx_t,s,t)\\
&=
2\Delta(\vx_t,\vc,t)
+\bigl(2\alpha-1\bigr)
\bigl(
\vv^{\mathrm{old}}(\vx_t,t)-\vu^{\mathrm{old}}(\vx_t,s,t)
\bigr).
\end{aligned}
\end{equation*}
Denoting the unique minimizer by $\vu^\dagger(\vx_t,s,t)$ and substituting the
residual into the zero gradient condition gives
\begin{equation}
\label{eq:direct-u-optimum}
\begin{aligned}
\vu^\dagger(\vx_t,s,t)
&=\vu^{\mathrm{old}}(\vx_t,s,t)
+\frac{2}{\beta}\Delta(\vx_t,\vc,t)+\frac{2\alpha-1}{\beta}
\bigl(
\vv^{\mathrm{old}}(\vx_t,t)-\vu^{\mathrm{old}}(\vx_t,s,t)
\bigr)\\
&\overset{\beta=2\alpha}{=}
\vv^{+}(\vx_t,t)
+\frac{
\vu^{\mathrm{old}}(\vx_t,s,t)-\vv^{\mathrm{old}}(\vx_t,t)
}{2\alpha},
\qquad \alpha>0.
\end{aligned}
\end{equation}
The last equality uses
$\Delta=\alpha(\vv^{+}-\vv^{\mathrm{old}})$ from \cref{eq:delta} and collects
the terms involving $\vu^{\mathrm{old}}$ and $\vv^{\mathrm{old}}$. Let
$\vu^{+}(\vx_t,s,t)$ denote the exact average velocity induced by $\vv^{+}$
over $[s,t]$. This is the desired output, whereas
\cref{eq:direct-u-optimum} gives
\begin{equation*}
\begin{aligned}
\vu^\dagger(\vx_t,s,t)-\vu^{+}(\vx_t,s,t)
={}&\vv^{+}(\vx_t,t)-\vu^{+}(\vx_t,s,t)+\frac{
\vu^{\mathrm{old}}(\vx_t,s,t)-\vv^{\mathrm{old}}(\vx_t,t)
}{2\alpha}.
\end{aligned}
\end{equation*}
Both terms vanish when $s$ approaches $t$. For a finite interval, the
DiffusionNFT identities do not force them to cancel. The direct objective
therefore does not guarantee
$\vu^\dagger(\vx_t,s,t)=\vu^{+}(\vx_t,s,t)$.

\noindent\textbf{Direct substitution of DMD and CDM outputs.}
DMD~\citep{dmd} and CDM~\citep{cdm} are trained with distribution matching.
Their distribution matching terms compare the generated and target marginal
distributions after averaging over latent noise. Matching these distributions
does not identify a unique generator map at each fixed $(\vx_t,\vc,t)$. Both
methods include additional regression or alignment terms, but these terms do
not regress the network output against the forward conditional target $\vv_t$.
Consequently, even when the output is parameterized as a velocity, their
objectives do not establish that the network output equals the posterior mean
$\E[\vv_t\mid\vx_t,\vc,t]$ required by DiffusionNFT.
The posterior decomposition therefore does not apply to the outputs inserted
into \cref{eq:diffnft_loss}. As a result, DiffusionNFT's regression analysis
cannot identify the optimum of the substituted objective with $\vv^{+}$.
Its Theorems~3.1 and~3.2 do not establish that DMD$+$DiffusionNFT or
CDM$+$DiffusionNFT realizes $\pi^{+}$.

\section{Proofs for MeanFlowNFT}
\label{sec:appendix-proofs}

This appendix restates and proves each result of \cref{sec:mf-theory}.

\medskip\noindent\textbf{\Cref{prop:mfnft_opt}} (Idealized pointwise optimum). \textit{Conditioned on $(\vx_t,\vc,s,t)$, the idealized pointwise minimizer is
\begin{equation*}
\begin{aligned}
\vV_{\theta^*}(\vx_t,s,t)
&= \vV^{\mathrm{old}}(\vx_t,s,t) + \frac{2}{\beta}\,\widehat{\Delta}(\vx_t,\vc,s,t),\\
\widehat{\Delta}(\vx_t,\vc,s,t)
&\triangleq \frac{1}{2}\,\E_{\pi^{\mathrm{old}}}\!\left[(2r-1)\bigl(\vv_t-\vV^{\mathrm{old}}(\vx_t,s,t)\bigr)\,\middle|\,\vx_t,\vc,s,t\right].
\end{aligned}
\end{equation*}}

\begin{proof}
Fix $(\vx_t,\vc,s,t)$. Under \cref{eq:V_theta}, the induced predictor
$\vV_\theta(\vx_t,s,t)$ is a deterministic function of the conditioned
variables for fixed $\theta$. Its value is therefore shared by all posterior
draws $(\vx_0,\epsilonv)$ consistent with the same $\vx_t$, and the idealized
pointwise optimum can be obtained by minimizing
$\E_{\pi^{\mathrm{old}}}[\ell\mid\vx_t,\vc,s,t]$ over the value $\vV_\theta(\vx_t,s,t)$. By the definitions of $\vV_\theta^{\pm}$, the
per-sample integrand of \cref{eq:mfnft_loss} is
\begin{equation*}
\begin{aligned}
\ell\bigl(\vV_\theta(\vx_t,s,t)\bigr)
={}&r\bigl\|(1-\beta)\vV^{\mathrm{old}}(\vx_t,s,t)+\beta\vV_\theta(\vx_t,s,t)-\vv_t\bigr\|_2^2\\
&+(1-r)\bigl\|(1+\beta)\vV^{\mathrm{old}}(\vx_t,s,t)-\beta\vV_\theta(\vx_t,s,t)-\vv_t\bigr\|_2^2 .
\end{aligned}
\end{equation*}
Dividing the gradient by $2\beta$ and collecting terms gives
\begin{equation*}
\tfrac{1}{2\beta}\,\nabla_{\vV_\theta}\ell
=\beta\bigl(\vV_\theta(\vx_t,s,t)-\vV^{\mathrm{old}}(\vx_t,s,t)\bigr)
+(2r-1)\bigl(\vV^{\mathrm{old}}(\vx_t,s,t)-\vv_t\bigr).
\end{equation*}
Since this is a strictly convex quadratic in $\vV_\theta(\vx_t,s,t)$ for $\beta>0$, the
stationarity condition
$\E_{\pi^{\mathrm{old}}}[\nabla_{\vV_\theta}\ell\mid\vx_t,\vc,s,t]=0$ yields
\begin{equation*}
\begin{aligned}
&\beta\bigl(\vV_{\theta^*}(\vx_t,s,t)-\vV^{\mathrm{old}}(\vx_t,s,t)\bigr)\\
&\quad=\E_{\pi^{\mathrm{old}}}\!\left[(2r-1)\bigl(\vv_t-\vV^{\mathrm{old}}(\vx_t,s,t)\bigr)\mid\vx_t,\vc,s,t\right]\\
&\quad=2\widehat{\Delta}(\vx_t,\vc,s,t),
\end{aligned}
\end{equation*}
which is exactly \cref{eq:opt_form}.
\end{proof}

\medskip\noindent\textbf{\Cref{cor:V_improvement}} ($\vV_{\theta^*}$ recovers the improved marginal velocity).
\textit{Setting the guidance strength to $\beta=2\alpha(\vx_t,\vc)$, \cref{eq:opt_form} collapses to
\begin{equation*}
\vV_{\theta^*}(\vx_t,s,t)=\vv^{+}(\vx_t,t)
\end{equation*}
for all $s\le t$, the marginal instantaneous velocity of the improved policy $\pi^{+}$.}

\begin{proof}
Since $\vu^{\mathrm{old}}$ is the MeanFlow average velocity of $\pi^{\mathrm{old}}$, its MeanFlow-induced instantaneous velocity is
\begin{equation*}
\widehat{\vv}^{\mathrm{old}}(\vx_t,t)
=\vu^{\mathrm{old}}(\vx_t,t,t)
=\vv^{\mathrm{old}}(\vx_t,t).
\end{equation*}
Applying the MeanFlow identity to this exact reference therefore gives
$\vV^{\mathrm{old}}(\vx_t,s,t)=\vv^{\mathrm{old}}(\vx_t,t)$ for all
$s\le t$. Thus $\vV^{\mathrm{old}}$ is a deterministic function of
$(\vx_t,\vc,t)$ and may be pulled out of the conditional expectation. Then
\begin{align*}
\widehat{\Delta}
&=\frac12\E_{\pi^{\mathrm{old}}}\!\left[(2r-1)\bigl(\vv_t-\vv^{\mathrm{old}}(\vx_t,t)\bigr)
\mid\vx_t,\vc,s,t\right] \\
&=\E_{\pi^{\mathrm{old}}}\!\left[r\,\vv_t\mid\vx_t,\vc\right]
-\tfrac12\E_{\pi^{\mathrm{old}}}\!\left[\vv_t\mid\vx_t,\vc\right]
-\tfrac12\E_{\pi^{\mathrm{old}}}\!\left[2r-1\mid\vx_t,\vc\right]\vv^{\mathrm{old}}(\vx_t,t) .
\end{align*}
Writing $\alpha=\E_{\pi^{\mathrm{old}}}[r\mid\vx_t,\vc]$ and using the posterior-mean identities
$\E_{\pi^{\mathrm{old}}}[\vv_t\mid\vx_t,\vc]=\vv^{\mathrm{old}}(\vx_t,t)$ (\cref{eq:marg_vel}) and $\E_{\pi^{\mathrm{old}}}[r\,\vv_t\mid\vx_t,\vc]=\alpha\vv^{+}(\vx_t,t)$, the last display becomes
\begin{equation*}
\widehat{\Delta}
=\alpha\vv^{+}(\vx_t,t)-\tfrac12\vv^{\mathrm{old}}(\vx_t,t)
-\tfrac12(2\alpha-1)\vv^{\mathrm{old}}(\vx_t,t)
=\alpha\bigl(\vv^{+}(\vx_t,t)-\vv^{\mathrm{old}}(\vx_t,t)\bigr)
=\Delta,
\end{equation*}
the DiffusionNFT reinforcement guidance in \cref{eq:delta}. Substituting
this into \cref{eq:opt_form} and using $\beta=2\alpha$ gives, for all
$s\le t$,
\begin{equation*}
\vV_{\theta^*}(\vx_t,s,t)
=\vv^{\mathrm{old}}(\vx_t,t)+\frac{2}{2\alpha}\alpha(\vv^{+}(\vx_t,t)-\vv^{\mathrm{old}}(\vx_t,t))
=\vv^{+}(\vx_t,t),
\end{equation*}
as claimed.
\end{proof}

\medskip\noindent\textbf{\Cref{lem:mf_consistency}} (MeanFlow consistency).
\textit{Let $\vv$ be an instantaneous velocity field. If $\vu$ satisfies
\begin{equation*}
\vu(\vx_t,s,t)
+(t-s)\left[\partial_t\vu(\vx_t,s,t)+(\partial_{\vx}\vu)(\vx_t,s,t)\vv(\vx_t,t)\right]
=\vv(\vx_t,t)
\end{equation*}
for all $s\le t$, then $\vu$ is the exact average velocity of the ODE $\dot\vx_\tau=\vv(\vx_\tau,\tau)$ over $[s,t]$.}

\begin{proof}
Fix $s<t$ and a trajectory $\{\vx_\tau\}_{\tau\in[s,t]}$ of the velocity
field $\vv$ with endpoint $\vx_t$ at time $t$, i.e.,
$\dot\vx_\tau=\vv(\vx_\tau,\tau)$. Define
$G(\tau)=(\tau-s)\vu(\vx_\tau,s,\tau)$. Differentiating along the trajectory,
\begin{align*}
G'(\tau)
&=\vu(\vx_\tau,s,\tau)
+(\tau-s)\left[\partial_t\vu(\vx_\tau,s,\tau)
+(\partial_{\vx}\vu)(\vx_\tau,s,\tau)\dot\vx_\tau\right] \\
&=\vu(\vx_\tau,s,\tau)
+(\tau-s)\left[\partial_t\vu(\vx_\tau,s,\tau)
+(\partial_{\vx}\vu)(\vx_\tau,s,\tau)\vv(\vx_\tau,\tau)\right] \\
&=\vv(\vx_\tau,\tau),
\end{align*}
where the last equality uses \cref{eq:mf_consistency} at time $\tau$. Since
$G(s)=0$, integrating from $s$ to $t$ gives
\begin{equation*}
(t-s)\vu(\vx_t,s,t)
=\int_s^t \vv(\vx_\tau,\tau)\,\mathrm{d}\tau .
\end{equation*}
This is precisely the average velocity of the ODE induced by $\vv$ over $[s,t]$.
\end{proof}

\medskip\noindent\textbf{\Cref{thm:u_improvement}} (The deployed MeanFlow policy improves).
\textit{In the setting of \cref{cor:V_improvement}, if the induced optimum is attained
for all intervals $s\le t$, the optimal average velocity $\vu_{\theta^*}$ is the
exact average velocity of the ODE induced by $\vv^{+}$. Therefore the MeanFlow policy induced
by $\vu_{\theta^*}$ coincides with $\pi^{+}$, and consequently
\begin{equation*}
J(\pi_{\theta^*}) = J(\pi^{+}) > J(\pi^{\mathrm{old}}).
\end{equation*}}

\begin{proof}
By the premise, the optimum in \cref{cor:V_improvement} is attained for all
intervals $s\le t$, so
\begin{equation*}
\vV_{\theta^*}(\vx_t,s,t)=\vv^{+}(\vx_t,t).
\end{equation*}
Taking $s=t$ in \cref{eq:V_theta}, the correction term vanishes and hence
\begin{equation*}
\widehat{\vv}_{\theta^*}(\vx_t,t)
=\vu_{\theta^*}(\vx_t,t,t)
=\vV_{\theta^*}(\vx_t,t,t)
=\vv^{+}(\vx_t,t).
\end{equation*}
Substituting this identity back into \cref{eq:V_theta} and using
$\vV_{\theta^*}(\vx_t,s,t)=\vv^{+}(\vx_t,t)$ yields
\begin{equation*}
\vu_{\theta^*}(\vx_t,s,t)
+(t-s)\left[
\partial_t\vu_{\theta^*}(\vx_t,s,t)
+(\partial_{\vx}\vu_{\theta^*})(\vx_t,s,t)\vv^{+}(\vx_t,t)
\right]
=\vv^{+}(\vx_t,t),
\end{equation*}
which is \cref{eq:mf_consistency} with $\vv=\vv^{+}$. By
\cref{lem:mf_consistency}, $\vu_{\theta^*}$ is the exact average velocity of the
ODE induced by $\vv^{+}$. Since $\vv^{+}$ is the marginal instantaneous velocity of the
positive policy $\pi^{+}$, this ODE has marginals corresponding to $\pi^{+}$, so the
exact MeanFlow update driven by $\vu_{\theta^*}$ samples from
$\pi^{+}$. Consequently
$J(\pi_{\theta^*})=J(\pi^{+})>J(\pi^{\mathrm{old}})$ for any non-degenerate reward.
\end{proof}

\section{More Implementation Details}
\label{app:impl}

This appendix provides more implementation details for MeanFlowNFT.

\noindent\textbf{Reward.}
For image training, we follow the multi-reward setup of
DiffusionNFT~\citep{diffusionnft}, using equally weighted \clipscore, \pickscore, and \hpstwo as
reward signals on the \pickscore prompt set. For video training, we follow the multi-reward setup of
\longcat~\citep{meituanlongcatteam2025longcatvideotechnicalreport}. The
\nmfmt{HPSv3-general} reward evaluates visual quality by scoring each frame with
the generic prompt \textit{``A high-quality image''} and averaging over all frames. The
\nmfmt{HPSv3-percentile} reward instead uses the video caption as the text prompt
and averages the top $30\%$ frame scores, reducing the effect of occasional low
scores caused by temporal content changes. These two HPSv3 rewards are combined
with the \videoalign motion-quality and text-alignment rewards described in
\cref{sec:impl}.

\noindent\textbf{Training.}
For \sdmodel, we first construct the MeanFlow policy with the two-stage
AnyFlow recipe~\citep{gu2026anyflowanystepvideodiffusion}. The first stage is
flow-map pretraining for $6000$ steps on precomputed latent--prompt pairs. It
uses AnyFlow's three-mode endpoint sampling: $50\%$ of samples take $s=t$ for
standard flow matching, $25\%$ take $s=0$ for endpoint consistency, and the
remaining samples draw $s\sim\mathcal{U}(0,t)$. Following AnyFlow~\citep{gu2026anyflowanystepvideodiffusion}, we use the reverse-CFG fusion with scale $4.5$, so the
resulting policy can be sampled CFG-free. The second stage performs on-policy
AnyFlow distillation for $12000$ steps, initialized from the stage-one LoRA
checkpoint, with sampling steps drawn from $\{2,4,8,16,40\}$.

MeanFlowNFT is then applied on top of the trained AnyFlow policy. We run RL
finetuning for $2000$ steps on \sdmodel and $1600$ steps on \wanmodel, where the
video run starts from the publicly released AnyFlow-Wan checkpoint. Both runs
use CFG-free $4$-step rollouts, equally weighted reward dimensions, $\beta=0.1$,
KL weight $10^{-4}$, AdamW with learning rate $3{\times}10^{-6}$, and fresh
training-time $(s,t)$ pairs sampled from the same three-mode AnyFlow schedule.
For video, following \longcat, we compute a group-normalized relative advantage
for each reward dimension independently, and then average the normalized
advantages for the RL update.

\noindent\textbf{Evaluation.}
Besides the evaluation protocol described in \cref{sec:impl}, we evaluate videos
with the official \vbench suite following AnyFlow~\citep{gu2026anyflowanystepvideodiffusion}.
Specifically, we compute all $16$ VBench dimensions and report their aggregated
\nmfmt{Total}, \nmfmt{Quality}, and \nmfmt{Semantic} scores in
\cref{tab:wan_video}, with the full per-dimension breakdown in
\cref{tab:vbench_1,tab:vbench_2}. All evaluations, including the image benchmarks,
are run on $8$ NVIDIA H20 GPUs.

\section{Additional Test-time Scaling Results}
\label{app:scaling}

\cref{fig:scaling_rest} reports the remaining evaluation metrics for the
test-time scaling study of \cref{fig:scaling_main}, on \sdmodel, and
\cref{fig:scaling_wan} reports the corresponding results on \wanmodel~1.3B.
MeanFlowNFT retains AnyFlow's any-step scaling behavior while consistently
delivering stronger generation quality.

\begin{figure}[!ht]
\centering
\includegraphics[width=\linewidth]{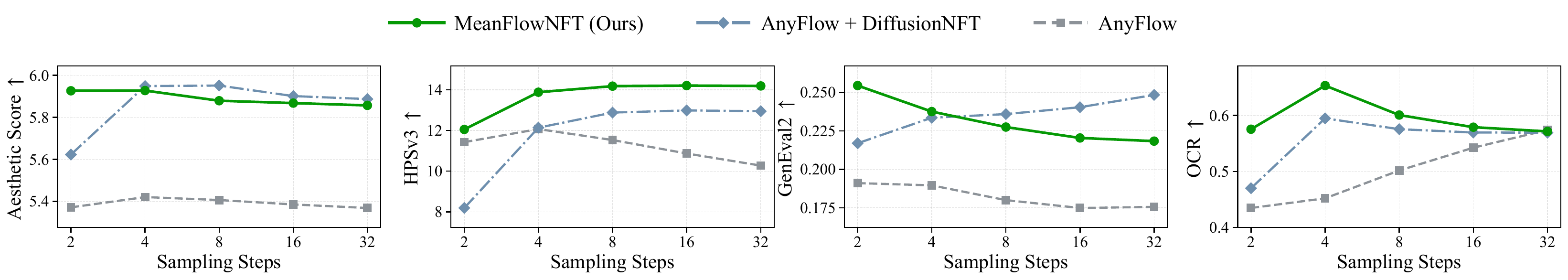}
\caption{Additional quantitative results of MeanFlowNFT test-time scaling on \sdmodel.}
\label{fig:scaling_rest}
\end{figure}

\begin{figure}[!ht]
\centering
\includegraphics[width=\linewidth]{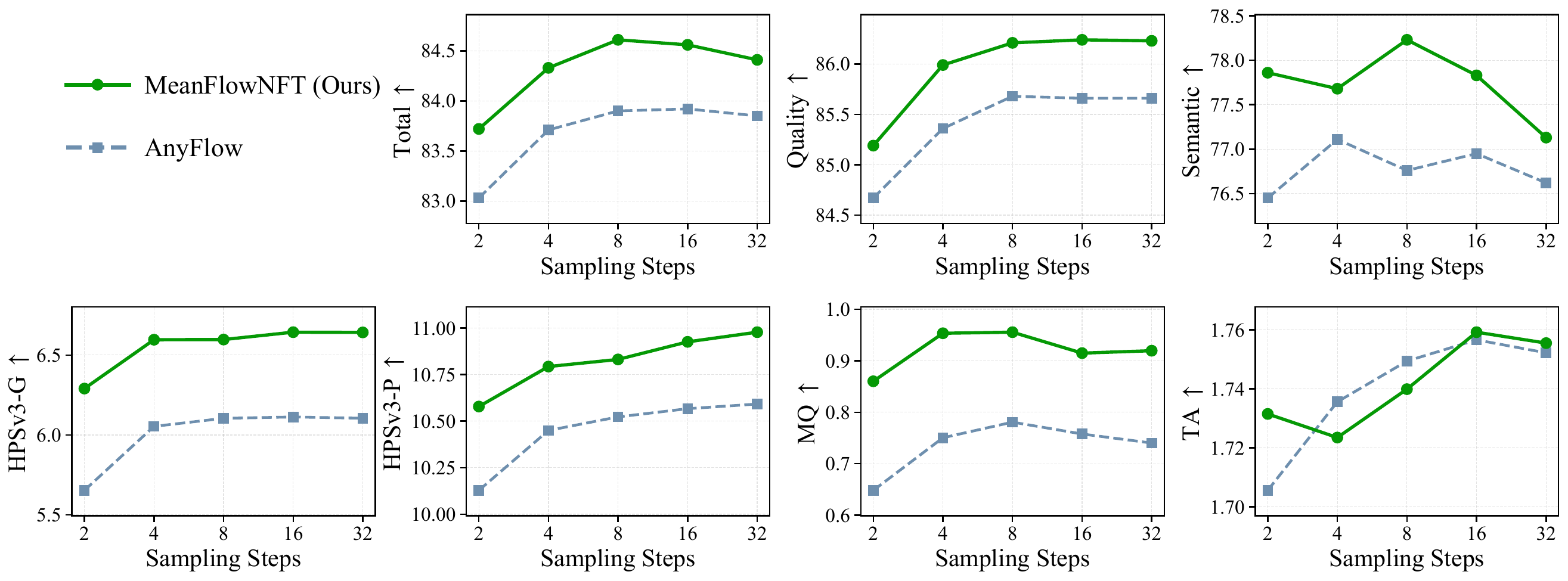}
\caption{Quantitative results of MeanFlowNFT test-time scaling on \wanmodel~1.3B.
\nmfmt{Total}, \nmfmt{Quality}, and \nmfmt{Semantic} are \nmfmt{VBench} scores, while the remaining metrics
(\nmfmt{HPSv3-G}/\nmfmt{HPSv3-P} and \nmfmt{MQ}/\nmfmt{TA}) are evaluated on the
$256$ held-out prompts.}
\label{fig:scaling_wan}
\end{figure}

\section{VBench Full Results}
\label{app:vbench_full}

In this section, we report the per-dimension \nmfmt{VBench} breakdown for the \wanmodel~1.3B
video-generation experiment of \cref{tab:wan_video}, split across
\cref{tab:vbench_1,tab:vbench_2} for readability.

\begin{table}[!ht]
\centering
\setlength{\tabcolsep}{4.5pt}
\renewcommand{\arraystretch}{1.2}
\caption{Full \nmfmt{VBench} per-dimension results on \wanmodel 1.3B (part 1 of 2).
Among few-step models, \best{bold} and \second{underline} denote the best and second-best results.}
\label{tab:vbench_1}
\resizebox{0.95\textwidth}{!}{%
\begin{tabular}{l c c c c c c c c}
\toprule
\makecell[l]{\textbf{Method}}
& \makecell{\nmfmt{Dynamic}\\\nmfmt{Degree}$\uparrow$} & \makecell{\nmfmt{Temporal}\\\nmfmt{Flickering}$\uparrow$} & \makecell{\nmfmt{Human}\\\nmfmt{Action}$\uparrow$}
& \makecell{\nmfmt{Overall}\\\nmfmt{Consistency}$\uparrow$} & \makecell{\nmfmt{Multiple}\\\nmfmt{Objects}$\uparrow$} & \makecell{\nmfmt{Color}$\uparrow$}
& \makecell{\nmfmt{Appearance}\\\nmfmt{Style}$\uparrow$} & \makecell{\nmfmt{Scene}$\uparrow$} \\
\midrule
\rowcolor{colGray!16}\multicolumn{9}{c}{\textit{Multi-step models}\defstep{50}} \\
\wanmodel 1.3B (w/ CFG) & 65.56 & 99.32 & 93.80 & 25.47 & 74.36 & 89.43 & 21.32 & 44.91 \\
\quad + \longcat RL & 52.78 & 99.12 & 92.00 & 25.43 & 75.90 & 87.23 & 21.19 & 39.55 \\
\midrule
\rowcolor{colGray!16}\multicolumn{9}{c}{\textit{Few-step models}\defstep{4}} \\
rCM     & \best{88.89} & 97.28 & 91.20 & 24.72 & 71.42 & 88.12 & 20.59 & 41.82 \\
DMD     & \second{88.61} & 97.37 & \second{94.00} & 24.90 & 76.16 & 86.00 & 20.19 & 39.52 \\
AnyFlow & 58.89 & \second{98.83} & 93.20 & \best{25.17} & \second{82.48} & \second{88.42} & \second{20.81} & \best{43.71} \\
\midrule
\rowcolor{colSky!22}\textbf{MeanFlowNFT (Ours)} & 59.45 & \best{99.44} & \best{94.40} & \second{25.09} & \best{84.79} & \best{89.21} & \best{21.08} & \second{43.41} \\
\bottomrule
\end{tabular}%
}
\end{table}

\begin{table}[!ht]
\centering
\setlength{\tabcolsep}{4.5pt}
\renewcommand{\arraystretch}{1.2}
\caption{Full \nmfmt{VBench} per-dimension results on \wanmodel 1.3B (part 2 of 2).
Among few-step models, \best{bold} and \second{underline} denote the best and second-best results.}
\label{tab:vbench_2}
\resizebox{\textwidth}{!}{%
\begin{tabular}{l c c c c c c c c}
\toprule
\makecell[l]{\textbf{Method}}
& \makecell{\nmfmt{Object}\\\nmfmt{Class}$\uparrow$} & \makecell{\nmfmt{Spatial}\\\nmfmt{Relationship}$\uparrow$} & \makecell{\nmfmt{Aesthetic}\\\nmfmt{Quality}$\uparrow$}
& \makecell{\nmfmt{Motion}\\\nmfmt{Smoothness}$\uparrow$} & \makecell{\nmfmt{Temporal}\\\nmfmt{Style}$\uparrow$} & \makecell{\nmfmt{Imaging}\\\nmfmt{Quality}$\uparrow$}
& \makecell{\nmfmt{Subject}\\\nmfmt{Consistency}$\uparrow$} & \makecell{\nmfmt{Background}\\\nmfmt{Consistency}$\uparrow$} \\
\midrule
\rowcolor{colGray!16}\multicolumn{9}{c}{\textit{Multi-step models}\defstep{50}} \\
\wanmodel 1.3B (w/ CFG) & 90.98 & 71.87 & 65.95 & 98.76 & 23.33 & 67.42 & 94.98 & 96.63 \\
\quad + \longcat RL & 89.68 & 75.56 & 67.07 & 98.87 & 23.08 & 70.08 & 96.38 & 96.27 \\
\midrule
\rowcolor{colGray!16}\multicolumn{9}{c}{\textit{Few-step models}\defstep{4}} \\
rCM     & 88.78 & 71.80 & 65.27 & 97.89 & 22.67 & \second{68.95} & 94.04 & 93.96 \\
DMD     & 87.71 & 76.88 & 65.86 & 97.88 & \second{23.04} & 68.03 & 94.41 & 94.16 \\
AnyFlow & \second{90.28} & \best{81.96} & \best{70.00} & \second{98.61} & 22.77 & \best{69.57} & \second{97.89} & \second{96.55} \\
\midrule
\rowcolor{colSky!22}\textbf{MeanFlowNFT (Ours)} & \best{90.49} & \second{81.36} & \second{69.53} & \best{99.22} & \best{23.06} & 68.82 & \best{98.44} & \best{97.07} \\
\bottomrule
\end{tabular}%
}
\vspace{-0.1in}
\end{table}

\section{More Qualitative Results}
\label{app:qualitative_results}

This section provides additional qualitative comparisons for both image and video
generation. The image examples compare MeanFlowNFT with multi-step/few-step RL methods,
few-step distillation baselines, and directly applying DiffusionNFT to few-step
generators. The video examples further compare against \wanmodel, \longcat~RL,
and few-step video distillation baselines. Across these examples, MeanFlowNFT
produces more faithful and visually coherent results, while preserving strong
quality across different sampling steps.

\providecommand{\stepimg}[2]{{\color{black!35}\setlength{\fboxsep}{0pt}\fbox{\includegraphics[width=\dimexpr#1-2\fboxrule\relax]{#2}}}}

\newcommand{\sdw}{0.15\textwidth}
\newcommand{\sdqual}[2]{%
\begin{figure}[!ht]
\centering
\captionsetup{skip=3pt}\captionsetup[subfigure]{skip=1pt,belowskip=-1pt,aboveskip=0pt}%
{\fontsize{6}{6.6}\selectfont\itshape ``#2''\par}\vspace{2pt}
\setlength{\tabcolsep}{1pt}\renewcommand{\arraystretch}{0.6}%
\begin{tabular}{@{}*{6}{>{\centering\arraybackslash}m{\sdw}}@{}}
{\scriptsize 40 steps} & {\scriptsize 40 steps} & {\scriptsize 40 steps} & {\scriptsize 4 steps} & {\scriptsize 4 steps} & {\scriptsize 4 steps}\\[1.5pt]
\stepimg{\sdw}{figures/qualitative/sd35/#1/sd35_base_cfg45.pdf} & \stepimg{\sdw}{figures/qualitative/sd35/#1/40_fg.pdf} & \stepimg{\sdw}{figures/qualitative/sd35/#1/40_nft.pdf} & \stepimg{\sdw}{figures/qualitative/sd35/#1/dmd.pdf} & \stepimg{\sdw}{figures/qualitative/sd35/#1/cdm.pdf} & \stepimg{\sdw}{figures/qualitative/sd35/#1/R.pdf}\\[-1pt]
\subcaption{\sdmodel} & \subcaption{Flow-GRPO} & \subcaption{DiffusionNFT} & \subcaption{DMD} & \subcaption{CDM} & \subcaption{$R_{\mathrm{dm}}$}\\[2pt]
{\scriptsize 4 steps} & {\scriptsize 4 steps} & {\scriptsize 4 steps} & {\scriptsize 4 steps} & {\scriptsize 16 steps} & {\scriptsize 32 steps}\\[1.5pt]
\stepimg{\sdw}{figures/qualitative/sd35/#1/rtdmd.pdf} & \stepimg{\sdw}{figures/qualitative/sd35/#1/dmd_nft.pdf} & \stepimg{\sdw}{figures/qualitative/sd35/#1/cdm_nft.pdf} & \stepimg{\sdw}{figures/qualitative/sd35/#1/4_af.pdf} & \stepimg{\sdw}{figures/qualitative/sd35/#1/16_af.pdf} & \stepimg{\sdw}{figures/qualitative/sd35/#1/32_af.pdf}\\[-1pt]
\subcaption{RTDMD} & \subcaption{DMD+NFT} & \subcaption{CDM+NFT} & \multicolumn{3}{c}{{\scriptsize (j)~AnyFlow}}\\[2pt]
{\scriptsize 4 steps} & {\scriptsize 16 steps} & {\scriptsize 32 steps} & {\scriptsize 4 steps} & {\scriptsize 16 steps} & {\scriptsize 32 steps}\\[1.5pt]
\stepimg{\sdw}{figures/qualitative/sd35/#1/4_af_nft.pdf} & \stepimg{\sdw}{figures/qualitative/sd35/#1/16_af_nft.pdf} & \stepimg{\sdw}{figures/qualitative/sd35/#1/32_af_nft.pdf} & \stepimg{\sdw}{figures/qualitative/sd35/#1/4_mf_nft.pdf} & \stepimg{\sdw}{figures/qualitative/sd35/#1/16_mf_nft.pdf} & \stepimg{\sdw}{figures/qualitative/sd35/#1/32_mf_nft.pdf}\\[-1pt]
\multicolumn{3}{c}{{\scriptsize (k)~AnyFlow+NFT}} & \multicolumn{3}{c}{{\scriptsize (l)~\textbf{MeanFlowNFT}}}\\
\end{tabular}
\caption{Text-to-image comparison on \sdmodel. Here ``$+$NFT'' abbreviates ``$+$DiffusionNFT''.}
\label{fig:img_#1}
\end{figure}
}

\sdqual{008}{New York Skyline with `Hello World' written with fireworks on the sky.}
\sdqual{014}{A man and woman sit on a park bench.}
\sdqual{010}{A car playing soccer, digital art.}

\newcommand{\vidw}{0.45\textwidth}
\newcommand{\vidqual}[2]{%
\begin{figure}[!ht]
\centering
\captionsetup{skip=3pt}\captionsetup[subfigure]{skip=1pt,belowskip=-1pt,aboveskip=0pt}%
{\fontsize{6}{6.6}\selectfont\itshape ``#2''\par}\vspace{1pt}
\setlength{\tabcolsep}{0pt}\renewcommand{\arraystretch}{0.6}%
\begin{tabular}{@{}>{\centering\arraybackslash}m{0.03\textwidth}@{\hspace{2pt}}>{\centering\arraybackslash}m{\vidw}@{\hspace{4pt}}>{\centering\arraybackslash}m{\vidw}@{}}
\rotatebox[origin=c]{90}{\scriptsize 50 steps} & \stepimg{\vidw}{figures/qualitative/wan/#1/wan.pdf} & \stepimg{\vidw}{figures/qualitative/wan/#1/lc.pdf}\\[-1pt]
 & \subcaption{\wanmodel~1.3B} & \subcaption{\longcat~RL}\\[4pt]
\rotatebox[origin=c]{90}{\scriptsize 4 steps} & \stepimg{\vidw}{figures/qualitative/wan/#1/dmd.pdf} & \stepimg{\vidw}{figures/qualitative/wan/#1/rcm.pdf}\\[-1pt]
 & \subcaption{DMD} & \subcaption{rCM}\\[4pt]
\rotatebox[origin=c]{90}{\scriptsize 4 steps} & \stepimg{\vidw}{figures/qualitative/wan/#1/4_af.pdf} & \stepimg{\vidw}{figures/qualitative/wan/#1/4_mf_nft.pdf}\\[1pt]
\rotatebox[origin=c]{90}{\scriptsize 16 steps} & \stepimg{\vidw}{figures/qualitative/wan/#1/16_af.pdf} & \stepimg{\vidw}{figures/qualitative/wan/#1/16_mf_nft.pdf}\\[1pt]
\rotatebox[origin=c]{90}{\scriptsize 32 steps} & \stepimg{\vidw}{figures/qualitative/wan/#1/32_af.pdf} & \stepimg{\vidw}{figures/qualitative/wan/#1/32_mf_nft.pdf}\\[-1pt]
 & \subcaption{AnyFlow} & \subcaption{\textbf{MeanFlowNFT}}\\
\end{tabular}
\caption{Text-to-video comparison on \wanmodel~1.3B.}
\label{fig:vid_#1}
\end{figure}
}

\vidqual{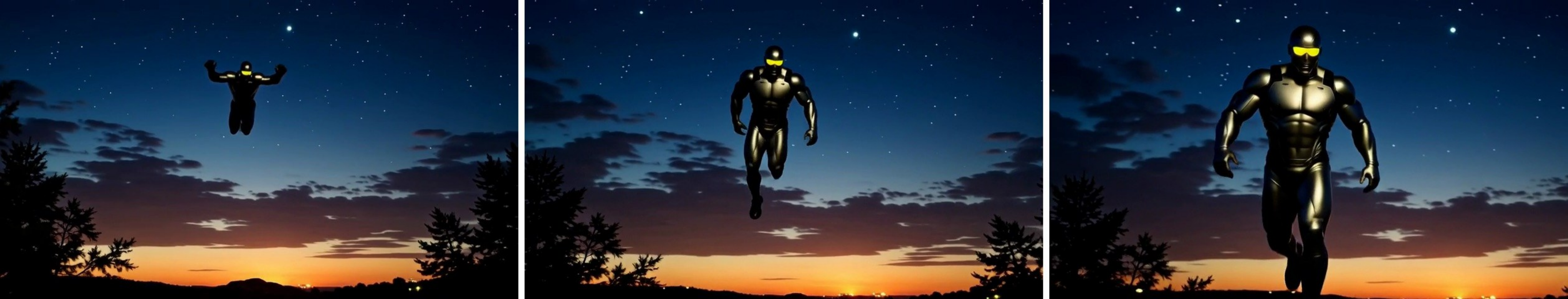}{An armored superhero in a metallic suit plunges headfirst from the night sky, shot cinematically against glowing dusk clouds.}
\vidqual{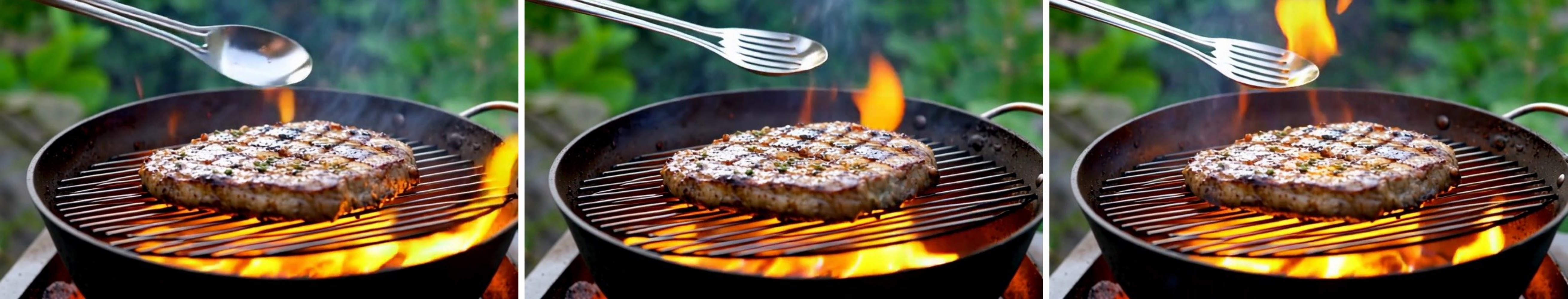}{A thick steak with a rich, seared Maillard crust sizzles on a flaming grill, sparks and fire rising from the grates.}
\vidqual{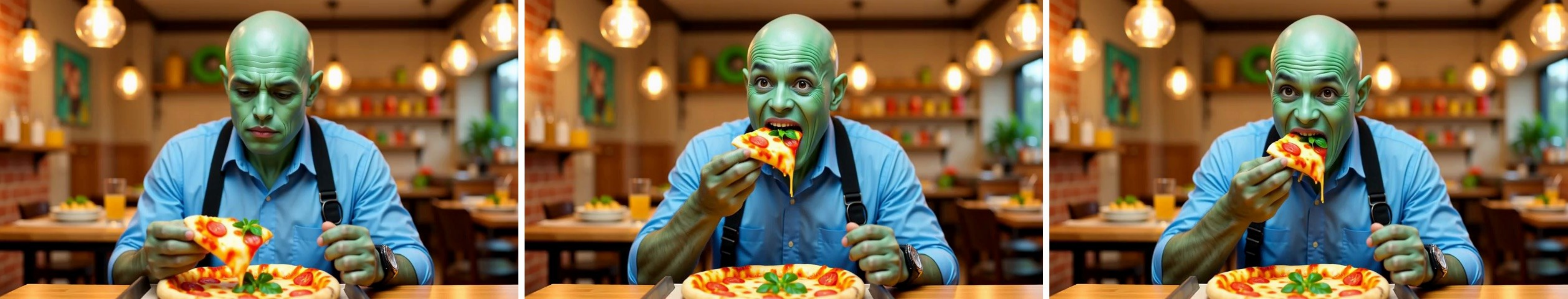}{A little green alien sits inside a cozy pizzeria, happily eating a big slice of pizza.}

\end{document}